\documentclass[twoside]{article}
\PassOptionsToPackage{numbers, compress}{natbib}

\usepackage[accepted]{aistats2021}

\usepackage{times}

\usepackage{nicefrac}       
\usepackage{microtype}      

\usepackage[utf8]{inputenc}
\usepackage[T1]{fontenc}
\usepackage{microtype}
\usepackage{graphicx}
\usepackage{subfigure}
\usepackage{booktabs}
\usepackage[colorlinks=true,citecolor=blue,linkcolor=blue]{hyperref}

\usepackage{url}
\usepackage{nicefrac}
\usepackage{listings}
\usepackage{mathtools}
\usepackage{amsfonts}
\usepackage{amssymb}
\usepackage{amsmath}
\usepackage{amsthm}
\usepackage{docmute}
\usepackage[acronym, nomain]{glossaries}
\usepackage{indentfirst}
\usepackage{bm}
\usepackage{here}
\usepackage{booktabs}       
\usepackage{cancel}
\usepackage{multirow}
\usepackage{wrapfig}
\usepackage{lipsum}

\usepackage{natbib}
\usepackage{algorithm,algorithmic}

\newcommand{\diff}[1]{{\textcolor{black}{#1}}}

\newtheorem{lem}{Lemma}
\newtheorem{thm}{Theorem}
\newtheorem{prop}{Proposition}
\newtheorem{cor}{Corollary}
\newtheorem{assump}{Assumption}
\newtheorem{definition}{Definition}

\DeclareMathOperator*{\argmax}{\mathop{\rm argmax}}
\DeclareMathOperator*{\argmin}{\mathop{\rm argmin}}

\DeclarePairedDelimiterX{\KL}[2]{\mathrm{KL}[}{]}{#1\;\delimsize\|\;#2}

\DeclarePairedDelimiterX\braket[2]{\langle}{\rangle}{#1 \delimsize\vert #2}

\newcommand{\const}{\mathrm{const.}}

\newcommand{\med}{\mathrm{med}}

\newcommand{\IID}{i.i.d.\ }
\usepackage{amsthm}

\def \Re {\mathbb{R}}
\def \Na {\mathbb{N}}
\newcommand{\restrict}[1]{\raisebox{-.5ex}{$|$}_{#1}}
\newcommand{\ReD}{\Re^\Dimension{}}
\def \Rgeqzero {\mathbb{R}_{\geq 0}}
\def \colinfty {{\mathrm{col},\infty}}
\newcommand{\colinftynrm}[1]{\|#1\|_\colinfty}
\usepackage{dsfont}
\def \Indicator {\mathds{1}}
\newcommand{\Ind}[1]{\Indicator\{#1\}}

\def \sumin {\sum_{i=1}^n}
\def \sumizn {\sum_{i=0}^n}
\def \sumjm {\sum_{j=1}^m}
\def \thetap {\theta'}
\def \dthp {\mathrm{d}\thetap}

\newtheorem{theorem}{Theorem}

\newtheorem{remark}{Remark}

\newcommand{\Dimension}{d}
\newcommand{\Ynum}{m}

\def \xn {X^n}
\def \xz {X_0}
\newcommand{\xncontami}[1]{X^n_{[#1]}}
\def \xzn {\xncontami{\xz}}

\def \Ym {Y^\Ynum{}}
\def \Yone {Y_1}
\def \Yj {Y_j}

\def \dYj {\mathrm{d}Y_j}
\def \ReDm {\Re^{\Ynum{} \times \Dimension{}}}

\def \Pth {P_\theta}
\def \Pthp {P_{\thetap}}
\def \Fth {F_\theta}
\def \Fthp {F_{\thetap}}
\def \Gth {G_\theta}
\def \dGth {\mathrm{d}\Gth^m}
\def \dGthY {\mathrm{d}\Gth^m(\Ym)}
\def \dGthYone {\mathrm{d}\Gth(\Yone)}
\def \dGthYj {\mathrm{d}\Gth(\Yj)}
\def \prodm {\prod_{j=1}^m}

\usepackage{mathtools}
\DeclarePairedDelimiterX{\infdivx}[2]{(}{)}{%
  #1\;\delimsize\|\;#2%
}
\def \DivMark {{\widehat D}}

\newcommand{\Div}{\DivMark\infdivx}
\newcommand{\Divznm}{\Div{\xzn}{\Ym}}
\newcommand{\Divnm}{\Div{X^{n}}{\Ym}}
\newcommand{\Dznm}{\DivMark^{n, m}_{[\xz]}}
\newcommand{\Dzjnm}{\DivMark^{n, m}_{[\xzj]}}
\newcommand{\Dnm}{\DivMark^{n, m}}

\def \cbar {\bar c}

\def \limz {\lim_{\|\xz\| \to \infty}}
\def \PseudoPosterior {{\hat \pi}}
\def \SCn {\mathrm{SC}_{n+1}^{\theta}}
\def \Hn {\Lambda_n}
\def \H {\Lambda}

\newcommand{\rhoki}{\rho_k(i)}
\newcommand{\nuki}{\nu_k(i)}
\def \rhokdi {\rho_k^d(i)}
\def \nukdi {\nu_k^d(i)}
\def \rhokz {\rho_k(0)}
\def \nukz {\nu_k(0)}
\def \rhokzinvdg {\rho_k^{-\Dimension{}\gamma}(0)}
\def \nukzinvdg {\nu_k^{-\Dimension{}\gamma}(0)}
\def \Aone {A_1}
\def \Atwo {A_2}
\newcommand{\tepsilon}{\tilde \epsilon}

\newcommand{\xzdummyj}{X_j'}
\newcommand{\xzj}{\xzdummyj}
\newcommand{\limj}{\lim_{j \to \infty}}
\newcommand{\alphafn}[2]{\eta_{#1}(#2)}
\newcommand{\Deltafn}[3]{\Delta_{#1, #2, #3}}

\newcommand{\alphafnthep}{\alphafn{\theta}{\epsilon}}
\newcommand{\Deltafnthep}{\Deltafn{\theta}{\tepsilon}{\epsilon}}

\newcommand{\Nullset}{\mathcal{Z}}

\newcommand{\mukj}{\mu_k(j)}

\newcommand{\nukdigammainv}{(\nu_k(i))^{-\gamma \Dimension{}}}
\newcommand{\mukdjgammainv}{(\mu_k(j))^{-\gamma \Dimension{}}}

\newcommand{\Uc}{U^\mathrm{c}}
\newcommand{\fl}{f_l}
\newcommand{\Ul}{U_l}

\begin{document}
\runningtitle{$\gamma$-ABC: Outlier-Robust Approximate Bayesian Computation Based on a Robust Divergence Estimator}
%

%

\twocolumn[

\aistatstitle{$\gamma$-ABC: Outlier-Robust Approximate Bayesian Computation\newline{} Based on a Robust Divergence Estimator}

\aistatsauthor{ Masahiro Fujisawa \And Takeshi Teshima \And  Issei Sato \And Masashi Sugiyama }

\aistatsaddress{ The University of Tokyo \\ RIKEN AIP \\ fujisawa@ms.k.u-tokyo.ac.jp \And The University of Tokyo \\ RIKEN AIP \\ teshima@ms.k.u-tokyo.ac.jp \And The University of Tokyo \\ RIKEN AIP \\ sato@k.u-tokyo.ac.jp \And RIKEN AIP \\ The University of Tokyo \\ sugi@k.u-tokyo.ac.jp} ]

\begin{abstract}
Approximate Bayesian computation (ABC) is a likelihood-free inference
method that has been employed in various applications. 
\diff{However, ABC can be sensitive to outliers if a data discrepancy measure is chosen inappropriately.}
In this paper, we propose to use a
nearest-neighbor-based $\gamma$-divergence estimator as a data
discrepancy measure. We show that our estimator possesses a suitable theoretical
robustness property called the \emph{redescending} property. In
addition, our estimator enjoys various desirable properties such as
high flexibility, asymptotic unbiasedness, almost sure convergence,
and linear-time computational complexity.
Through experiments, we demonstrate that our method achieves
significantly higher robustness than existing discrepancy measures.
\end{abstract}

\section{Introduction}
\label{intro}
Approximate Bayesian computation (ABC) has been proposed as a ``likelihood-free'' inference scheme to approximately perform Bayesian inference when a complex model is used and it is impossible or difficult to compute its likelihood (see \cite{Marin12} for a general overview).
Instead of investigating the explicit form of the likelihood function, \diff{ABC seeks parameters of a simulator-based model that can generate data that is close to the observed data under some discrepancy measure.}
ABC has been applied to many research fields, e.g., evolutionary biology~\cite{Tavare97}, dynamic systems~\cite{Simon10}, economics~\cite{Peters12}, epidemiology~\cite{Blum20}. aeronautics~\cite{Jason17}, and astronomy~\cite{Cameron12}.

Rejection ABC~\cite{Tavare97,Pritchard99,Jiang18}, the basic ABC algorithm, proceeds as follows: (i) we draw an independent sample of the parameter $\theta$ from some prior $\pi$,
(ii) we simulate data $Y^m = \{Y_{j}\}_{1:m}$ for each value of $\theta$,
(iii) the parameter $\theta$ is discarded if the discrepancy $D(X^{n},Y^{m})$ between the observed data $X^{n} = \{X_{i}\}_{1:n}$ and the simulated data $Y^m = \{Y_{j}\}_{1:m}$ exceeds a tolerance threshold $\epsilon$. 
The accepted $\theta$ is used in the subsequent inference as a sample from an approximation to the posterior distribution called the \emph{ABC posterior distribution}.
Many studies have been performed to enhance the computational efficiency of the rejection ABC scheme, e.g., applying Markov Chain Monte Carlo (MCMC)~\cite{Marjoram03,Wegmann09} or sequential Monte Carlo (SMC)~\cite{Sisson09,Drovandi11,Pierre12}.
\begin{table*}[t]
\centering
\caption{Relationship between previous work and our work (\emph{OR}: Outlier robustness, \emph{RP}: Redescending property~\cite{MaronnaRobust2019}, \emph{AU}: Asymptotically unbiasedness, \emph{ASC}: Almost sure convergence, \emph{QA}: ABC posterior analysis).
In MONK~\cite{Lerasle19}, the number of blocks that divide the data is denoted by $Q$.
The order of time costs for the $q$-Wasserstein distance is based on approximate optimization algorithms~\cite{Cuturi13,cuturi14}.
The order of time costs for the CAD is based on logistic regression, where $d$ is the dimension of the observed and synthesized data.
The symbol $(n \lor m)$ denotes $\max \{n, m \}$.
}
\label{Related work}
\scalebox{1}{
\begin{tabular}{l|cccccc}
\hline
\multicolumn{1}{c}{Discrepancy measure}                & \begin{tabular}[c]{@{}c@{}}OR\end{tabular} & \begin{tabular}[c]{@{}c@{}}RP\end{tabular} & \begin{tabular}[c]{@{}c@{}}AU\end{tabular} & \begin{tabular}[c]{@{}c@{}}ASC\end{tabular} & \begin{tabular}[c]{@{}c@{}}QA\end{tabular} & \begin{tabular}[c]{@{}c@{}}Time cost\end{tabular} \\ \hline \hline
MMD~\cite{Park16,smola07}                                      & -                                                            & -                                                               & -                                                                 & -                                                                 & -                                                                  & $\mathcal{O}((n+m)^{2})$                            \\
$q$-Wasserstein
distance~\cite{Bernton17}                 & -                                                            & -                                                               & -                                                                 & -                                                                 & -                                                                  & $\mathcal{O}((n+m)^{2})$                            \\
CAD~\cite{Gutmann18}                                & -                                                            & -                                                               & -                                                                 & -                                                                 & $\checkmark$                                                       & $\mathcal{O}((n+m)d)$                                  \\
MONK BCD-Fast
~\cite{Lerasle19}               & $\checkmark$                                                 & -                                                               & -                                                                 & -                                                                 & -                                                                  & $\mathcal{O}
\bigg(\frac{(n+m)^{3}}{Q^{2}}\bigg)$              \\
KL-divergence estimator~\cite{PrezCruz08,Jiang18}                  & -                                                            & -                                                               & $\checkmark$                                                      & $\checkmark$                                                      & $\checkmark$                                                       & $\mathcal{O}((n \lor m) \log (n \lor m))$           \\
\textbf{$\gamma$-divergence estimator (ours)} & $\checkmark$                                                 & $\checkmark$                                                    & $\checkmark$                                                      & $\checkmark$                                                      & $\checkmark$                                                       & $\mathcal{O}((n \lor m) \log (n \lor m))$          
\end{tabular}
}
\end{table*}

The core element of ABC is the data discrepancy measure $D(X^{n},Y^{m})$ and the \diff{accuracy of parameters} from the ABC posterior distribution crucially depends on its choice.
While many discrepancies for ABC have been proposed, such as the distance between summary statistics~\cite{blum13,Drovandi15,Wegmann09}, the maximum mean discrepancy (MMD)~\cite{Park16}, the Wasserstein distance~\cite{Bernton17}, a Kullback-Leibler (KL) divergence estimator~\cite{Jiang18}, and the classification accuracy discrepancy (CAD)~\cite{Gutmann18}, these are often not robust to severe contamination of data~\cite{Ruli20,Lerasle19,staerman20}.
Recently, two outlier-robust discrepancies have been proposed:
one is a robust discrepancy based on MMD~\cite{Lerasle19}, and the other is using a robust M-estimator, e.g., Huber's estimation function~\cite{Ruli20}.
However, these methods do not possess an ideal robust property for an extreme outlier, called the \textit{redescending property}~\cite{MaronnaRobust2019}.
In addition, the former method has the cubic time cost $\mathcal{O}((n + m)^{3}/Q^{2})$, where $Q$ is the number of blocks that divide the data.
When ABC is applied to astronomy~\cite{Kremer17,Cameron12}, for example, we have to deal with noisy large-scale datasets, and the cubic time cost can be intractable.
For the CAD, we may improve the robustness of the CAD by employing a robust classifier, such as robust LDA~\cite{croux01}, robust FDA~\cite{kim06}, or robust logistic regression~\cite{feng14}; however, its performance depends on the choice of the classifier, and its validity and robustness for heavy contamination data are not guaranteed in the ABC framework.
\diff{Outside the ABC framework, recently, many robust inference schemes have been proposed, e.g., using robust divergences for \emph{parametric} model inference~\cite{Basu98,Fujisawa08}, Bayesian inference~\cite{Knoblauch18,Jewson18,Nakagawa20}, variational inference~\cite{futami18}, \diff{and} constructing Bayesian inference through a pseudo-likelihood via MMD~\cite{Cherief20}.
Unfortunately, these methods cannot be used as a data discrepancy measure for ABC because these assume a \emph{tractable} likelihood or \emph{parametric} models.}
Therefore, there is no discrepancy measure that has both \emph{well-guaranteed} robustness for an extreme outlier and reasonable time costs.

In this paper, we propose a novel outlier-robust and computationally-efficient discrepancy measure based on the $\gamma$-divergence~\cite{Fujisawa08}.
Our discrepancy measure results in a robustness property of the ABC posterior called the \emph{redescending property}~\cite{MaronnaRobust2019}, i.e., it automatically ignores extreme outliers in the observed data.
Furthermore, we show that the $\gamma$-divergence estimation using a naive $k$-nearest neighbor density estimate has desirable asymptotic properties, which is not straightforward to prove unlike divergence estimators in $f$-divergence class, such as $\alpha$-divergence estimator~\cite{poczos11a}.
Table~\ref{Related work} summarizes the relations among our method and major existing discrepancy measures.
Our contributions are as follows.
\begin{itemize}
\setlength{\parskip}{-0.2cm}
\setlength{\itemsep}{0.15cm}
    \item We construct a non-parametric and robust divergence estimator based on the $\gamma$-divergence (Section \ref{subsec:gamma_div_estimator_derivation}).
    \item We show that our method theoretically enjoys the robustness and validity (Sections \ref{sec:robust_analysis} and \ref{sec:asymp_quasi_posterior}).
    \item \diff{We show that our method has the asymptotic unbiasedness and the almost sure convergence property, which are mathematically much harder to show than those for the divergence estimators belonging to the $f$-divergence family (Section \ref{subsec:asymp_analysis}).}
    \item Through experiments, we show that our estimator can significantly reduce the influence of the outliers even when the observed data have a large number of outliers (Section \ref{sec:experiments}).
\end{itemize}

The rest of this paper is organized as follows.
We summarize the ABC framework in Section \ref{sec:ABC}.
In Sections~\ref{sec:robust_div_est} and \ref{sec:asm_anals}, we introduce the $k$-nearest neighborhood ($k$-NN) based density estimation, explain how to derive our divergence estimator based on $k$-NN, and conduct the theoretical analyses.
Finally, we show the experimental results and the conclusion in Sections~\ref{sec:experiments} and \ref{sec:conclusion}.

\section{Preliminaries}
\label{sec:ABC}
In this section, we give an overview of ABC.
More detailed descriptions of the discrepancy measures which are often used in ABC can be found in Appendix~\ref{app:detail_discrepancy}.

\subsection{Approximate Bayesian Computation}
\diff{We define $\mathcal{X} \subset \mathbb{R}^{d}$ as the data space and $\Theta$ as the parameter space.}
The model $\{p_{\theta}: \theta \in \Theta \}$ is a family of probability distributions on $\mathcal{X}$ and has no explicit formula, but we assume that we can generate \IID random samples from $p_\theta$ given the value of $\theta$.
The purpose of ABC is to seek the model parameter $\theta$ by comparing the observed data $ X_{1},\ldots,X_{n}\overset{\textrm{i.i.d.}}{\sim} p_{\theta^{*}}$ and the synthetic data $ Y_{1},\ldots,Y_{m} \overset{\textrm{i.i.d.}}{\sim} p_{\theta}$, where $\theta^*\in\Theta$ is the true parameter.
The criterion used to compare these datasets $X^n=\{X_1,\ldots,X_n\}$ and $Y^m=\{Y_1,\ldots,Y_m\}$ is the data discrepancy measure $D(X^{n},Y^{m})$ \diff{defined over $\mathcal{X}^{n} \times \mathcal{X}^{m}$.}

\begin{algorithm}[t]                      
                \caption{Rejection ABC Algorithm~\cite{Tavare97,Pritchard99}}
                \label{alg:reject_abc}
                \footnotesize
                \begin{algorithmic}[1]
                        \REQUIRE {Observed data $\{X_{i}\}_{i=1}^{n}$, prior $\pi(\theta)$ on the parameter space $\Theta$, tolerance threshold $\epsilon$, data discrepancy measure $D$}
                        \STATE \textbf{Initialize:} $\epsilon$
                        \FOR {$t=1$ to $T$}
                        \STATE \textbf{repeat:} propose $\theta \sim \pi(\theta)$ and draw $Y_{1} \ldots,Y_{m} \overset{\textrm{i.i.d.}}{\sim} p_{\theta}$
                        \STATE \textbf{until:} $D(X^{n}, Y^{m}) < \epsilon$
                        \STATE Obtain $\theta^{(t)} = \theta$
                        \ENDFOR
                        \RETURN $\{ \theta^{(t)} \}_{t=1}^{T}$
                \end{algorithmic}
\end{algorithm}

\diff{A well known algorithm of ABC is the rejection ABC~\cite{Tavare97,Pritchard99,Jiang18}, which proceeds as follows: (i) we draw an independent sample of the parameter $\theta$ from some prior $\pi$,
(ii) we simulate data $Y^m = \{Y_{j}\}_{1:m}$ for each value of $\theta$,
(iii) the parameter $\theta$ is discarded if the discrepancy $D(X^{n},Y^{m})$ between the observed data $X^{n} = \{X_{i}\}_{1:n}$ and the simulated data $Y^m = \{Y_{j}\}_{1:m}$ exceeds a tolerance threshold $\epsilon$. 
}
The rejection ABC algorithm is shown in Algorithm~\ref{alg:reject_abc}.
It enables us to obtain \IID random samples $\{\theta^{(t)} \}_{t=1}^{T}$ from the ABC posterior distribution defined as follows.
\begin{definition}[ABC posterior distribution]
\diff{Let $\epsilon$ be fixed.}
Then, the ABC posterior distribution is defined by

\footnotesize
\begin{align}
\label{eq:quasi_posterior}
    \pi(\theta&|X^{n},D, \epsilon) 
    \propto \int \pi(\theta) \Ind{D(X^{n},Y^{m}) < \epsilon}p_{\theta}(Y^{m}) \mathrm{d}Y^{m},
\end{align}
\normalsize
where $\pi(\theta)$ is a prior over the parameter space $\Theta$, $\epsilon > 0$ is a tolerance threshold, and $p_{\theta}(Y^{m}) = \prod_{j=1}^{m}p_{\theta}(Y_{j})$.
\end{definition}

We only focus on the rejection ABC~\cite{Tavare97,Pritchard99,Jiang18} throughout this paper.
\diff{The main reason is two-fold: (i) to make a fair comparison and followed the experimental setting of the recent paper~\cite{Jiang18} proposing discrepancy measures for ABC, and (ii) to give theoretical guarantees to the ABC posterior explicitly, e.g., Theorem~\ref{thm:redescending-sensitivity-curve} and Corollary~\ref{cor:asymp_quasi_posterior_gamma} in this paper.
The rejection ABC is a reasonable choice for this purpose.}
\diff{While there are many sophisticated ABC algorithms, they are often extensions of the rejection ABC~\cite{Marjoram03,Wegmann09,Sisson09,Drovandi11,Pierre12}; therefore, we can easily combine our development of the rejection ABC with these algorithms.}

\subsection{Model of Data Contamination}
\label{subsec:def_contami}
In this paper, \diff{we assume that the \emph{observed} data are contaminated with outliers} and focus on Huber's contamination-by-outlier case~\cite{Huber64}, where observed data are sampled \IID from the following distribution:
\begin{align}
\label{def_huber_contami}
    (1-\eta)G(x) + \eta H(x),
\end{align}
where $G(x)$ is a distribution we are interested in, $H(x)$ is an arbitrary contamination distribution, and $\eta \in [0, 1]$ is the proportion of contamination.
If $\eta$ is relatively high, e.g., $\eta=0.2$, the observed data are highly contaminated by $H$.

Due to distribution contamination described above, severe bias occurs in parameter estimation.
Many robust estimation methods have been proposed to reduce the estimation bias caused by outliers~\cite{Huber64,Huber09,windham95,Basu98}.
However, these methods tend to exhibit undesirable behaviors both empirically and theoretically, for heavily contaminated data \cite{Fujisawa08}.
Furthermore, for \emph{non-parametric} inference schemes such as ABC, many of such robust inference frameworks cannot be used because they normally assume that the likelihood is tractable.
Although \citet{Lerasle19} and \citet{Ruli20} have proposed robust discrepancy measures that can be compatible with non-parametric inference, these methods also cannot deal with a heavy contamination and the former method has high time costs.
To conduct a robust non-parametric inference for heavily contaminated data, it is necessary to construct an alternative discrepancy measure with both robustness for an extreme outlier and reasonable time costs.

\section{\texorpdfstring{$\gamma$}{Lg}-ABC and Its Robustness}
\label{sec:robust_div_est}
In this section, we construct a \emph{non-parametric} ``likelihood-free'' inference scheme based on the $\gamma$-divergence that has been used in robust parameter estimation from heavily contaminated data.
In Section \ref{subsec:derivation_gamma}, we introduce the $\gamma$-divergence and explain why we choose a $k$-NN based density estimation to derive our estimator.
Next, we overview a $k$-NN based density estimation in Section \ref{subsec:knn_density_est} and derive our estimator in Section \ref{subsec:gamma_div_estimator_derivation}.
Finally, we guarantee the robustness of the ABC based on our estimator in Section \ref{sec:robust_analysis}.

\subsection{\texorpdfstring{$\gamma$}{Lg}-divergence and Its Estimation}
\label{subsec:derivation_gamma}
To make a robust parameter estimation in the heavily contamination situation described in Section \ref{subsec:def_contami}, \citet{Fujisawa08} proposed the \emph{$\gamma$-divergence}, which possesses strong robustness for heavily contaminated data.
\begin{definition}[$\gamma$-divergence~\cite{Fujisawa08}]
\label{def:gamma_div}
Let $p$ and $q$ be positive measurable functions from
a measurable set $\mathcal{M}_0 \subseteq \mathbb{R}^d$ to $\mathbb{R}$.
Let $\gamma > 0$.
Then, the $\gamma$-divergence is defined as
\small
\begin{align}
\label{gamma_div}
    D_{\gamma}&(p\|q) \notag \\
    &= \frac{1}{\gamma(1+\gamma)} \log \frac{\bigg( \int_{\mathcal{M}_{0}} p^{1+\gamma}(x) \mathrm{d}x \bigg) \bigg( \int_{\mathcal{M}_{0}} q^{1+\gamma}(x) \mathrm{d}x \bigg)^{\gamma} }{\bigg( \int_{\mathcal{M}_{0}} p(x)q^{\gamma}(x)\mathrm{d}x \bigg)^{1+\gamma}},
\end{align}
\normalsize
\end{definition}
To combine the $\gamma$-divergence with ABC, we need to estimate Eq.~\eqref{gamma_div} from observed and synthesized data.
\diff{A potential approach we can consider} to estimating Eq.~\eqref{gamma_div} is extending $f$-divergence estimation frameworks, e.g., kernel density estimation (KDE)~\cite{Hardle06} and direct density ratio estimation methods such as KLIEP~\cite{sugi08} or uLSIF~\cite{kanamori09}.
However, the former method suffers from high time cost due to a kernel function and the necessity to select appropriate kernels and its parameters.
Furthermore, the latter methods need to construct a model directly for a density ratio; therefore, it is hard to use this approach in the estimation of Eq.~\eqref{gamma_div} since $\gamma$-divergence, which is not included in the $f$-divergence class, is not expressed as a functional of the density ratio.

For these reasons, we consider using a $k$-NN based density estimation to estimate $\gamma$-divergence.
This approach has only one hyper-parameter, $k$.
Furthermore, this approach does not depend on any additional models.

\subsection{\texorpdfstring{$k$}{Lg}-Nearest Neighbor based Density Estimation}
\label{subsec:knn_density_est}

Let $X^{n}$ be an \IID sample of size $n$ from a probability distribution with density $p$, and $Y^{m}$ be an \IID sample of size $m$ from $q$.
Furthermore, we define $\rho_{k}(i)$ as the Euclidean distance between the $i$-th sample $X_i$ of $X^{n}$ and its $k$-th nearest neighbor ($k$-NN) among $X^{n}\setminus X_{i}$.
Similarly, we define $\nu_{k}(i)$ as the Euclidean distance between the $i$-th sample $X_{i}$ and its $k$-NN among $Y^{m}$.
Let $\mathcal{B}(x,R)$ be a closed ball with radius $R$ around $x \in \mathbb{R}^{d}$.
Finally, $\mathcal{V}(\mathcal{B}(x,R)) = \bar{c}R^{d}$ is defined as its volume, where
$\bar{c}$ is the volume of the $d$-dimensional unit ball.

\citet{loftsgaarden65} constructed the density estimators of $p$ and $q$ at the $i$-th sample $X_{i}$ via $k$-NN as follows:
\normalsize
\begin{align}
\label{def:knn_est_p}
\hat{p}_{k}(x_{i}) &= \frac{k}{(n-1)\mathcal{V}(\mathcal{B}(x_i,\rho_{k}(i)))} = \frac{k}{(n-1)\bar{c}\rho_{k}^{d}(i)}, \\
\label{def:knn_est_q}
\hat{q}_{k}(x_{i}) &= \frac{k}{m\mathcal{V}(\mathcal{B}(x_i,\nu_{k}(i))} = \frac{k}{m\bar{c}\nu_{k}^{d}(i)}.
\end{align}
\normalsize

These density estimators, Eqs.~\eqref{def:knn_est_p} and \eqref{def:knn_est_q}, are consistent estimators of the density only when the number of neighbors $k$ goes to infinity as the sample size $n$ goes to infinity.
We use these estimators for constructing our robust divergence estimator.
Hereafter, we fix the value of $k$ and show that our divergence estimator still has desirable asymptotic properties, including consistency.

\subsection{Robust Divergence Estimator on \texorpdfstring{$\gamma$}{Lg}-divergence}
\label{subsec:gamma_div_estimator_derivation}
Now we derive a \emph{non-parametric} $\gamma$-divergence estimator based on $k$-NN density estimation.
In ABC settings, outliers could be included in the observed data $X^{n}$.
To reduce the influence of outliers, we rewrite the term $\int_{\mathcal{M}_{0}} q^{1+\gamma}(x) \mathrm{d}x$ in Eq.~\eqref{gamma_div} to be $k$-NN estimatable from $\mathcal{M}' \subseteq \mathbb{R}^d$, where $\mathcal{M}'$ is the support of $q$,
i.e.,
\begin{align*}
    \int_{\mathcal{M}'} q^{1+\gamma}(y) \mathrm{d}y.
\end{align*}
We can use the same notion of Eq.~\eqref{def:knn_est_p} when we focus on the synthetic data $Y^{m}$; therefore, the density estimation for $q(y)$ based on $k$-NN can be written as
\begin{align}
\label{def:knn_est_q_y}
\hat{q}_{k}(y_{j}) &= \frac{k}{(m-1)\mathcal{V}(\mathcal{B}(y_j,\bar{\rho}_{k}(j))} = \frac{k}{(m-1)\bar{c}\bar{\rho}_{k}^{d}(j)},
\end{align}
where $\bar{\rho}_{k}(j)$ is the Euclidean distance between the $j$-th sample $Y_j$ of $Y^{m}$ and its $k$-NN among $Y^{m}\setminus Y_{j}$.

By plugging in the $k$-NN density estimator into Eqs.~\eqref{def:knn_est_p}, \eqref{def:knn_est_q}, and \eqref{def:knn_est_q_y}, we derive the $k$-NN based $\gamma$-divergence estimator as
\begin{align}
\label{eq:default_gamma_est}
    &\widehat{D}_{\gamma}(X^{n}\|Y^{m}) =  \frac{1}{\gamma(1+\gamma)} \notag \\
    &\times\left(
    \log \frac{\bigg(\displaystyle{\frac{1}{n} \sum_{i=1}^{n} \left(\frac{\bar{c}}{k}\hat{p}_{k}(x_{i})\right)^{\gamma}\bigg) } \bigg(\frac{1}{m} \displaystyle{\sum_{j=1}^{m} \left(\frac{\bar{c}}{k}\hat{q}_{k}(y_{j})\right)^{\gamma} \bigg)^{\gamma}}}
    {\bigg(\displaystyle \frac{1}{n}\sum_{i=1}^{n} \left(\frac{\bar{c}}{k}\hat{q}_{k}(x_{i})\right)^{\gamma} \bigg)^{1+\gamma}}
    \right).
\end{align}
The details of the derivation is in Appendix \ref{proof:gamma_estimator}.

The estimator in Eq.~\eqref{eq:default_gamma_est} involves $2n$ and $2m$ operations of nearest neighbor search. If we implement them by $KD$ trees~\cite{Bentley75,Maneewongvatana01}, the time cost of finding $\widehat{D}_{\gamma}(X^{n}\|Y^{m})$ is $\mathcal{O}((n\lor m)\log (n\lor m))$, where $(n\lor m) = \max \{n,m\}$, which is among the fastest (up to log factors) of the existing robust discrepancy approximators (see Table~\ref{Related work}).
Furthermore, this estimator fortunately enjoys ideal asymptotic properties: \emph{asymptotic unbiasedness} and \emph{almost surely convergence} under mild assumptions.
We will show them in Section \ref{subsec:asymp_analysis}.

\subsection{Robustness Property of \texorpdfstring{$\gamma$}{Lg}-ABC against Outliers}
\label{sec:robust_analysis}
Here, we investigate the behavior of the \emph{sensitivity curve} (SC), which is a finite-sample analogue of the \emph{influence function} (IF), both of which are used in quantifying the robustness of statistics~\cite{futami18,Ruli20}.
We fix the observed data $X^{n}$ and consider a contamination by an outlier $X_{0}$. We define the contaminated data as $X_{[X_{0}]}^{n} \coloneqq (X_{0}, X_{1},\ldots,X_{n})$.
Then, the SC is defined as follows.
\begin{definition}[Sensitivity curve {\citep[2.1e]{HampelRobust2005}}]
Let \(\gamma, \epsilon > 0\).
Let us define the (population) pseudo-posterior as $\hat{\pi}(\theta|X^{n}) \coloneqq \pi(\theta|X^{n},\widehat{D}_\gamma, \epsilon)$.
The sensitivity curve of \(\PseudoPosterior\) is defined as
\begin{equation*}\begin{split}
\SCn(X_{0}) \coloneqq (n+1)\left(\PseudoPosterior(\theta|X_{[X_{0}]}^{n}) - \PseudoPosterior(\theta|X^{n})\right).
\end{split}\end{equation*}
\end{definition}
We consider SC instead of IF for two reasons: (i) we are interested in the pseudo-posterior distribution \(\hat \pi(\theta|\xn)\) with respect to a finite sample \(\xn\),
and (ii) the IF of the quantities based on the considered divergence estimator may not be even defined (a detailed explanation is in Remark \ref{remark2} in Appendix \ref{sec:remarks}).

Under this definition and some additional assumptions, we obtain the following theorem.
Our analysis is a finite-sample analogue of what is called the \emph{redescending property} \citep{MaronnaRobust2019} in the context of IF analysis.
\begin{thm}[Sensitivity curve analysis]
Assume \(k < \min\{n, m\}\).
Also assume that \(\Fth(\epsilon) \coloneqq \int \Ind{\widehat{D}_\gamma(X^{n}\|\Ym) < \epsilon} p_{\theta}(Y^{m}) \mathrm{d}Y^{m}\) is \(\beta\)-Lipschitz continuous for all \(\theta \in \Theta\).
Then, we have
\begin{equation*}\begin{split}
\lim_{\|X_{0}\| \rightarrow \infty} \SCn(X_{0}) \leq -\frac{\beta \pi(\theta)}{\Hn(1+\gamma)}\log\left(1 - \frac{1}{n^2}\right)^{n+1},
\end{split}\end{equation*}
where \(\Hn \coloneqq \int \pi(\thetap) \Fthp(\epsilon) \dthp\).
Furthermore, if \(\lim_{n \to \infty} \Hn\) exists and is non-zero, then the right-hand side of the above inequality converges to \(0\).
\label{thm:redescending-sensitivity-curve}
\end{thm}
The proof is in Appendix \ref{sec:thm_and_proof_for_sc}.
Through Theorem~\ref{thm:redescending-sensitivity-curve}, we can see that the influence of contamination is reduced when we have enough data, even if the magnitude of the outlier $X_0$ is very large.
\diff{Intuitively, an estimator has the redescending property if its IF first ascends away from zero as outliers become more pronounced, while the IF ``redescends'' towards zero as outliers become increasingly extreme.}
\diff{Since our analysis is a finite-sample analogue of the redescending property in the context of IF,} this result implies the robustness of our method that an extreme outlier is automatically ignored.

\begin{figure*}[t]
    \centering
    \includegraphics[scale=0.32]{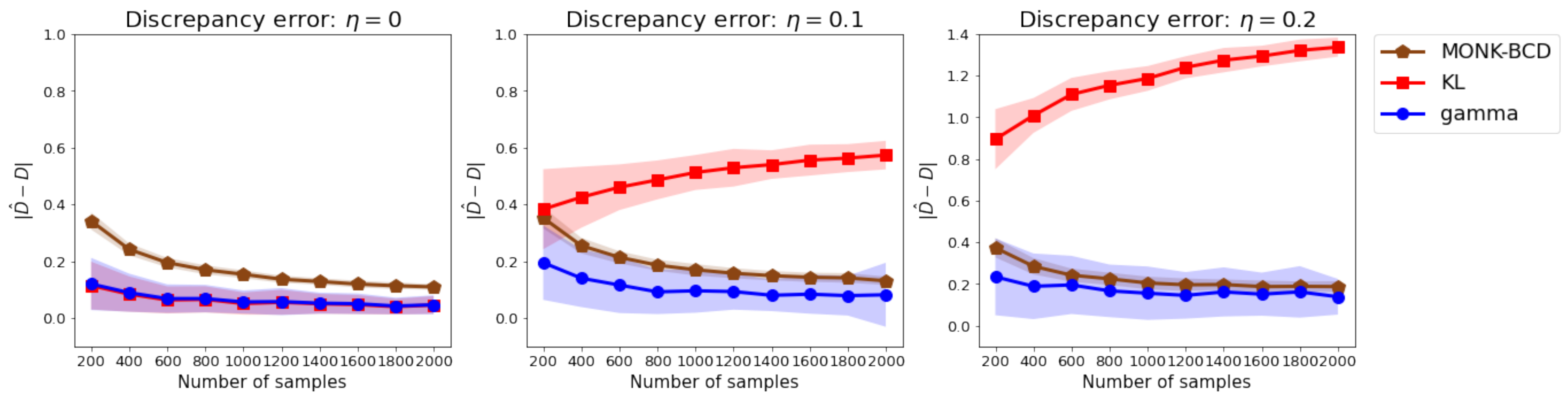}
    \caption{\diff{Experimental results for the absolute error of discrepancy.}
    We compared the robustness of the discrepancy measure based on robust MMD (MONK-BCD) and divergence-based discrepancies (KL and ours).
    The mean $\pm$ std. values of each error are plotted (solid lines are mean values and shaded areas represent the range of mean $\pm$ std values).
    Lower values are better. The true discrepancy $D$ was estimated by $10^{4}$ Monte Carlo samples. We conducted experiments on each discrepancy for $100$ times for various contamination rates ($\eta = 0, 0.1, 0.2$) and sample sizes ($200, 400, \ldots, 2000$).
    We set the hyper-parameter candidates of our $\gamma$-divergence estimator as $\gamma = (0.1,0.2,0.25,0.4,0.5,0.6,0.75,0.9)$ and the one with the smallest mean score among them is displayed.}
    \label{fig:disc_error}
\end{figure*}

\subsection{Robustness on Estimation Error of Discrepancy}
\diff{In Section~\ref{sec:robust_analysis}, we showed the theoretical robustness of our method.
Here, we experimentally investigate the robustness based on the estimation error of discrepancy.}
Figure \ref{fig:disc_error} shows the discrepancy estimation error. 
In this figure, we show the errors between robust MMD and true MMD, and the KL- or $\gamma$-divergence estimator and true KL-divergence.
We considered the $1$-dimensional standard normal distribution $\mathcal{N}(0,1)$ and the contaminated $1$-dimensional standard normal distribution $(1-\eta)\mathcal{N}(0,1) + \eta \mathcal{N}(10,1)$.
\diff{Whereas the outliers negatively affect the KL divergence estimator~\cite{PrezCruz08}, MONK-BCD~\cite{Lerasle19} and our $\gamma$-divergence estimator are robust to an increase in the contamination rate $\eta$.}
Furthermore, if we choose the hyper-parameter $\gamma$ properly, our discrepancy estimator achieves comparable accuracy to the KL-divergence estimator in the non-contaminated case.

\section{Asymptotic Analysis on ABC}
\label{sec:asm_anals}
In this section, we elucidate essential asymptotic properties, such as asymptotic unbiasedness and almost sure convergence, for our divergence estimator defined by Eq.~\eqref{eq:default_gamma_est}.
Furthermore, we analyze an asymptotic behavior of the ABC posterior distributions built on our divergence estimator.

\subsection{Theoretical Analysis for \texorpdfstring{$\gamma$}{Lg}-divergence Estimator}
\label{subsec:asymp_analysis}
To confirm the validity of the proposed estimator in Eq.~\eqref{eq:default_gamma_est}, we show two properties: the asymptotic unbiasedness and the almost sure convergence.

We show the asymptotic unbiasedness by assuming that $\mathcal{M}$, i.e., the support of $p$, has the following mild regularity condition.
These conditions are commonly used for investigating the asymptotic properties of divergence estimators, e.g., the $\alpha$-divergence estimator in \citet{poczos11a}.
\begin{assump}[Restrictions on the domain $\mathcal{M}$~\cite{poczos11a}]
\label{assmp:restrict_M}
We assume
\begin{align*}
    \inf_{0<\delta<1} \inf_{x \in \mathcal{M}} \frac{\mathcal{V}(\mathcal{B}(x,\delta)\cap\mathcal{M})}{\mathcal{V}(\mathcal{B}(x,\delta))} \coloneqq r_{\mathcal{M}} > 0.
\end{align*}
\end{assump}
\diff{Assumption~\ref{assmp:restrict_M} states that the intersection of $\mathcal{M}$ with an arbitrary small ball having the center in $\mathcal{M}$ has a volume that cannot be arbitrarily small relative to the volume of the ball.}
It intuitively means that almost all points of $\mathcal{M}$ are in its interior.

Furthermore, we define the following function:
\normalsize
\begin{align}
\label{func:powered}
    &H(x,p,\delta,\omega) \coloneqq \sum_{j=0}^{k-1} \bigg( \frac{1}{j!} \bigg)^{\omega} \Gamma(\kappa + j\omega) \bigg( \frac{p(x) + \delta}{p(x) - \delta} \bigg)^{j\omega} \notag \\
    &\ \ \ \ \ \ \ \ \ \ \ \ \ \ \ \ \ \ \ \ \ \ \ \ \ \ \ \ \ \ \ \ \ \ \ \ \times 
    (p(x) - \delta)^{-\gamma}((1-\delta)\omega)^{-\kappa-j\omega},
\end{align}
\normalsize
where $\Gamma(\cdot)$ is the gamma function defined as $\Gamma(z) = \int_{0}^{\infty} t^{z-1}\exp(-t) \mathrm{d}t$.
\citet{poczos11a} used Assumption \ref{assmp:restrict_M} to show a uniform variant of Lebesgue's lemma and to show that Definition~\ref{def:uniform_lebesgue_approx_func} below is well-defined.
The function in Eq.~\eqref{func:powered} appears in the upper bound of the moment for the $\omega$-powered conditional distribution function (CDF) when Assumption~\ref{assmp:restrict_M} holds (see Theorem 37 in \citet{poczos11a}).
In addition, we impose some reasonable assumptions, which are also assumed in \citet{poczos11a}.
\begin{definition}[Uniformly Lebesgue-approximable function~\cite{poczos11a}]
\label{def:uniform_lebesgue_approx_func}
\diff{Denote by $L_{1}(\mathcal{M})$ the set of Lebesgue integrable functions on $\mathcal{M}$ and} let $g \in L_{1}(\mathcal{M})$.
The function $g$ is uniformly Lebesgue approximable on $\mathcal{M}$ if for any series $R_{n} \rightarrow 0$ and any $\delta >0$, there exists $n_0=n_{0}(\delta) \in \mathbb{Z}^{+}$ 
such that if $n > n_{0}$, then for almost all $x \in \mathcal{M}$,
\normalsize
\begin{align}
    g(x) - \delta < \frac{\int_{\mathcal{B}(x,R_{n})\cap\mathcal{M}}g(t)\mathrm{d}t}{\mathcal{V}(\mathcal{B}(x,R_{n})\cap\mathcal{M})} < g(x) + \delta.
\end{align}
\normalsize
\end{definition}
\begin{assump}[Condition for $p$ and $q$ from \citet{poczos11a}]
\label{assmp:p_bounded}
The positive functions $p$ and $q$ are bounded away from zero and uniformly Lebesgue approximable.
Furthermore, the expectations of the $l_{2}$-norm powered by $\kappa$ over $p$ and $q$ are bounded, i.e.,
\normalsize
\begin{align*}
    \int_{\mathcal{M}} \|x - y \|^{\kappa}p(y) \mathrm{d}y < \infty, \ \
    \int_{\mathcal{M}} \|x - y \|^{\kappa}q(y) \mathrm{d}y < \infty,
\end{align*}
\normalsize
for almost all $x \in \mathcal{M}$.
Furthermore, the following conditions hold:
\footnotesize
\begin{align*}
    &\int \int_{\mathcal{M}^{2}} \|x - y \|^{\kappa}p(y)p(x) \mathrm{d}y \mathrm{d}x < \infty, \\
    &\int \int_{\mathcal{M}^{2}} \|x - y \|^{\kappa}q(y)p(x) \mathrm{d}y \mathrm{d}x < \infty.
\end{align*}
\normalsize
\end{assump}

\begin{assump}[Condition for powered CDF in $\mathcal{M}$~\cite{poczos11a}]
\label{assmp:p_bounded_cdf}
\diff{The expectations of $H(x,p,\delta,1)$ and $H(x,q,\delta,1)$ are bounded as follows: $\exists \delta_{0} \ \ \mathrm{s.t.} \ \ \forall \delta \in (0,\delta_{0})$,}

\footnotesize
\begin{align*}
\int_{\mathcal{M}} H(x,q,\delta,1)p(x)\mathrm{d}x < \infty,
\int_{\mathcal{M}} H(x,p,\delta,1)p(x)\mathrm{d}x < \infty.
\end{align*}
\normalsize
\end{assump}
\diff{This assumption indicates that the expectations of the $\omega$-powered CDFs of $p$ and $q$ are bounded, respectively.}
The expectations appear in the upper bound of the moment of the $\omega$-powered CDFs.

Since our method has a term involving an expectation with respect to $q$, we set the following assumption that is almost the same as the condition for the support of $p$.
\begin{assump}[Extra condition for powered CDF in $\mathcal{M}'$]
\label{assmp:p_bounded_extra}
\diff{The expectation of $H(y,q,\delta,1)$ is bounded as}
\begin{align*}
    &\exists \delta_{0} \ \ \mathrm{s.t.} \ \ \forall \delta \in (0,\delta_{0}), \ \ \int_{\mathcal{M}'} H(y,q,\delta,1)q(y)\mathrm{d}y < \infty.
\end{align*}
\end{assump}
\diff{This assumption means that the expectation of the $\omega$-powered CDF with respect to $q$ is bounded.}
The above expectation appears in the upper bound of the moment for the $\omega$-powered CDFs.
Under these assumptions, the following theorem holds.
\begin{thm}[Asymptotic unbiasedness]
\label{thm:asymp_unbiased_others_p}
Let $0 < \gamma < k$ or $-k < \gamma < 0$. Suppose that Assumption~\ref{assmp:p_bounded} holds with $\kappa = \gamma$ and that Assumptions~\ref{assmp:p_bounded_cdf} and \ref{assmp:p_bounded_extra} hold.
Also assume that $q$ is bounded from above. Then, $\widehat{D}_{\gamma}(X^{n}\|Y^{m})$ defined in Eq.~\eqref{eq:default_gamma_est} is asymptotically unbiased, i.e.,
\begin{align*}
    \lim_{n,m \rightarrow \infty} \mathbb{E}\bigg[\widehat{D}_{\gamma}(X^{n}\|Y^{m}) \bigg] = D_{\gamma}(p\|q).
\end{align*}
\end{thm}
\diff{From this result, we can see that the asymptotic unbiasedness holds even if we set $\gamma$ as $-k < \gamma < 0$ (see Theorems \ref{thm:asymp_unbiased_positive} and \ref{thm:asymp_unbiased_q} in Appendix \ref{proof:unbiasedness}).}

Next, we establish the almost sure convergence of our estimator.
\begin{thm}[Almost sure convergence]
\label{thm:almost_surely_conv}
Let $\gamma < k$.
Suppose that Assumption~\ref{assmp:p_bounded} holds with $\kappa = \gamma$ and that Assumptions~\ref{assmp:p_bounded_cdf} and \ref{assmp:p_bounded_extra} hold.
Also assume that $p$ and $q$ are bounded from above.
Let $k(n)$ denote the number of neighbors applied at sample size $n$ such that $\lim_{n \rightarrow \infty} k(n) = \infty$, $\lim_{n \rightarrow \infty} n/k(n) = \infty$, $\lim_{m \rightarrow \infty} k(m) = \infty$ and $\lim_{m \rightarrow \infty} m/k(m) = \infty$. Then, our estimator converges almost surely to $D_{\gamma}(p\|q)$, that is,
\begin{align*}
\widehat{D}_{\gamma}(X^{n}\|Y^{m}) \overset{\textrm{a.s.}}{\rightarrow} D_{\gamma}(p\|q).
\end{align*}
\end{thm}

The proofs for these theorems are in Appendices \ref{proof:unbiasedness} and \ref{proof:almost_surely_convergence}.
Note that in the proofs of Theorems \ref{thm:asymp_unbiased_others_p} and \ref{thm:almost_surely_conv}, we cannot reuse the known theoretical results for the divergence estimators in the $f$-divergence class~\cite{PrezCruz08,poczos11a} and it required us to newly show several asymptotic properties specifically for $\gamma$-divergence estimation, which are given in Appendices \ref{app:asymptotic_analysis} and \ref{app:proof_asymp_anal}.

\subsection{Asymptotic Property of ABC Posterior Distributions with \texorpdfstring{$\gamma$}{Lg}-divergence Estimator}
\label{sec:asymp_quasi_posterior}
Now we analyze whether the ABC posterior based on our robust discrepancy measure can accurately estimate the parameter $\theta$ with small exact $\gamma$-divergence $D_{\gamma}(p_{\theta^{*}}\|p_{\theta})$ asymptotically.

According to Theorem 1 in \cite{Jiang18}, the asymptotic ABC posterior is a restriction of the prior $\pi$ to the region $\{ \theta \in \Theta : D(p_{\theta^{*}}\| p_{\theta}) < \epsilon \}$ with appropriate scaling.
Combining this with the almost sure convergence of $\widehat{D}_{\gamma}(X^{n}\|Y^{m})$ established in Theorem \ref{thm:almost_surely_conv}, we can obtain the following corollary.
\begin{cor}[Asymptotic ABC posterior with \texorpdfstring{$\gamma$}{Lg}-divergence estimator]
\label{cor:asymp_quasi_posterior_gamma}
Suppose that Assumptions \ref{assmp:p_bounded}-\ref{assmp:p_bounded_extra} are satisfied with $\kappa = \gamma$.
Let $n \rightarrow \infty$ and $m/n \rightarrow \alpha > 0$.
\diff{Let $\pi(\theta | D_{\gamma}(p_{\theta^{*}}\|p_{\theta})<\epsilon)$ be the posterior under $D_{\gamma}(p_{\theta^{*}}\|p_{\theta})<\epsilon$.}
If $\widehat{D}_{\gamma}(X^{n}\|Y^{m})$ is used as the data discrepancy measure in Algorithm \ref{alg:reject_abc}, the ABC posterior distribution satisfies
\begin{align*}
   \pi(\theta| X^{n};\widehat{D}_{\gamma},\epsilon) \rightarrow \pi(\theta | D_{\gamma}(p_{\theta^{*}}\|p_{\theta}) < \epsilon),
\end{align*}
almost surely, and therefore
\begin{align*}
    \lim_{n,m \rightarrow \infty} \pi(\theta| X^{n};\widehat{D}_{\gamma},\epsilon) 
    \propto \pi(\theta) \Ind{D_{\gamma}(p_{\theta^{*}}\|p_{\theta}) < \epsilon},
\end{align*}
almost surely.
\begin{proof}[Proof sketch]
In the same way as \citet{Jiang18},
we use L\'{e}vy's upward theorem (enabled by Theorem~\ref{thm:almost_surely_conv}; see Theorem \ref{thm:levy_upward} in Appendix \ref{app:asymptotic_analysis}) to $Z_{n} = \Ind{\widehat{D}_{\gamma}(X^{n}\|Y^{\infty}) < \epsilon}$ and apply the dominated convergence theorem~\cite{vaart_1998} to complete the proof.
\end{proof}
\end{cor}
Corollary~\ref{cor:asymp_quasi_posterior_gamma} shows that the ABC posterior based on our estimator converges to \diff{the maximum likelihood estimator minimizes the \emph{exact} $\gamma$-divergence between the empirical distribution of $p_{\theta^{*}}$ and $p_{\theta}$.}
Thus, ABC with our $\gamma$-divergence estimator asymptotically collects the $\theta$ with small $\gamma$-divergence.

\begin{figure*}[th]
    \centering
    \includegraphics[scale=0.28]{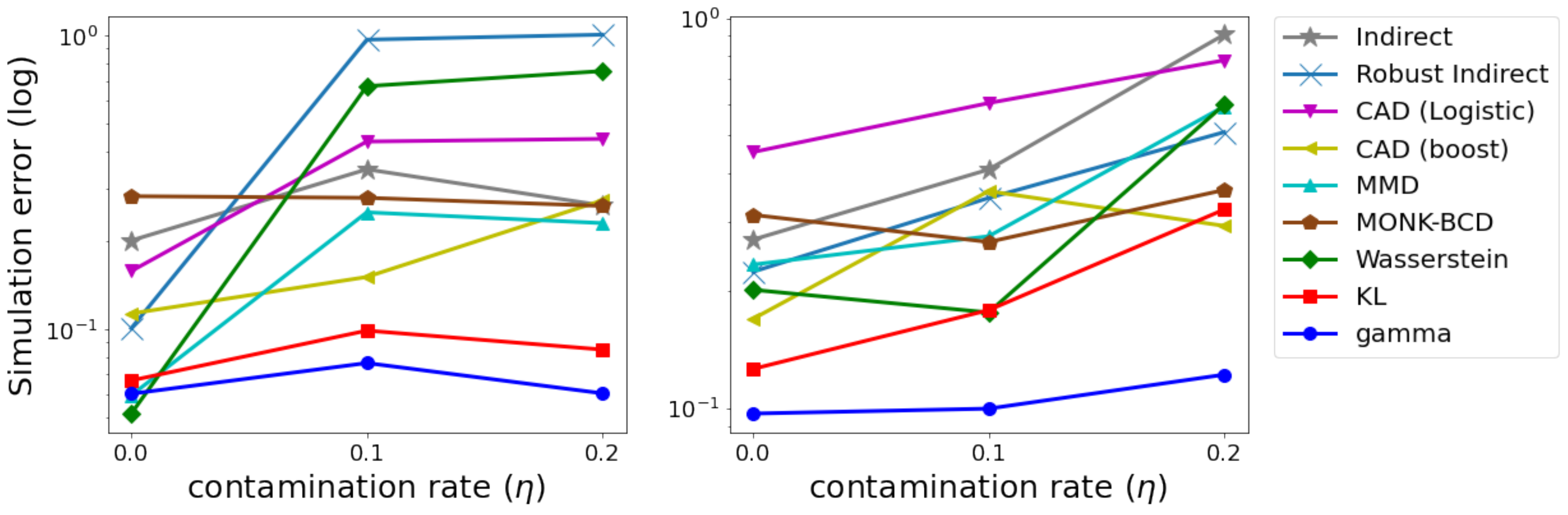}
    \caption{Simulation error for the GM (left) and the MG1 model (right) experiments.}
    \label{fig:energy_summary_mgm}
\end{figure*}
\section{Experiments}
\label{sec:experiments}
In this section, we report the performance of our estimator through five benchmark experiments of ABC.
Here, we confirm that the ABC with our discrepancy measure has immunity against heavily contaminated data.

\subsection{Settings}
\diff{We set $n=m$ following \citet{Jiang18} to prevent the resulting ABC posterior of the indirect method to be over-precise~\cite{Drovandi15} and to avoid arbitrariness in the experiments.}
The tolerance threshold $\epsilon$ was adaptively initialized so that $0.5\%$ of proposed 
parameters $\theta$ were accepted in each experiment by calculating each discrepancy measure $10^{3}$ times.
\diff{Furthermore, we artificially generated the $n$ \IID observed data from $G(X_{i})$ and replaced them by some outliers from $\mathcal{N}(10,1)$, where $G(X_{i})$ is an observed data distribution. 
In short, the contaminated data can be expressed as $(1-\eta)G(X_{i}) + \eta \mathcal{N}(10,1)$ in each dimension.}
In addition, we varied the contamination level $\eta$ in $\{ 0, 0.1, 0.2\}$ to confirm the robustness.
The hyper-parameter $\gamma$ was selected from $\{0.1, 0.2, 0.25, 0.4, 0.5, 0.6, 0.75, 0.9 \}$ for our method.

\diff{We measured the accuracy by the simulation error based on the \emph{energy distance}, which is a standard metric for distributions in statistics and has been used in the ABC literature, e.g., \citet{kajihara18a}.}
\diff{This allows us to directly compare the distributions between the non-contaminated observed data and the synthesized data simulated with the estimated parameter.}
We approximated the MAP estimator $\hat{\theta}_{\textrm{MAP}}$ of the ABC posterior by kernel density estimation with the Gaussian kernel \diff{with the bandwidth parameter $n^{-1/(d+4)}$, that is known as Scott's Rule~\cite{Scott15}.}

From each of the $10$ different models, we sampled the data and performed the ABC (Algorithm~\ref{alg:reject_abc}) with $T = 10^{5}$.
We repeated the procedure independently for $10$ times, and reported the average results.
The results with the standard errors are reported in Appendix \ref{app:simulation_error}, and the results of the mean-squared error (MSE) between $\hat{\theta}_{\textrm{MAP}}$ and the true parameter are also reported in Appendix \ref{app:full_MSE_scores}.

\diff{For our method, we conducted experiments independently for several $\gamma$ and displayed the one with the smallest mean score of the energy distance and the MSE among them.
The full results are reported in Figures \ref{fig:gamma_gaussian}--\ref{fig:gamma_gk} and \ref{fig:gamma_gm_energy}--\ref{fig:gamma_gk_energy} in Appendix~\ref{app:full_MSE_scores}.}
\diff{Furthermore, we compared the ABC posteriors of our method and that of the second-best method. These results are reported in Figures~\ref{fig:pos_gm}--\ref{fig:pos_gk} in Appendix~\ref{app:abc_post}.
}
\subsection{Models and Results}
Here, we summarize the model settings and the results of each experiment.
The details of the baseline discrepancies and the model architectures are shown in Appendices \ref{app:detail_discrepancy} and \ref{app:detail_models_exp}.

\paragraph{Gaussian Mixture Model (GM):}
The univariate Gaussian mixture model is the most basic benchmark setup in the ABC literature~\cite{Wilkinson13,Jiang18}.
We adopted a bivariate Gaussian mixture model with the true parameters $p^{*} = 0.3$, $\mu_{0}^{*} = (0.7, 0.7)$ and $\mu_{1}^{*} = (-0.7, -0.7)$, where $p^{*}$ is the mixture weight and $\mu_{0}^{*}, \mu_{1}^{*}$ are the means of the component distributions.
The variances are fixed as $0.5I-0.3I^\top$ and $0.25 I$, where $I$ is the identity matrix of size $(2, 2)$.

From the experimental results in Figure~\ref{fig:energy_summary_mgm}, we can see that our method achieves a better performance when the observed data are contaminated, whereas the other methods fail to give good scores.
In addition, in terms of the MSE, our method outperforms the baseline methods (see Appendix~\ref{app:full_MSE_scores}) under contamination.
\diff{From the results in Figure~\ref{fig:posterior_sum1}, we can confirm that the ABC posterior with our method places high density around the ground-truth parameter, whereas the baseline method fails to do so.}

\begin{figure*}[th]
    \centering
    \includegraphics[scale=0.31]{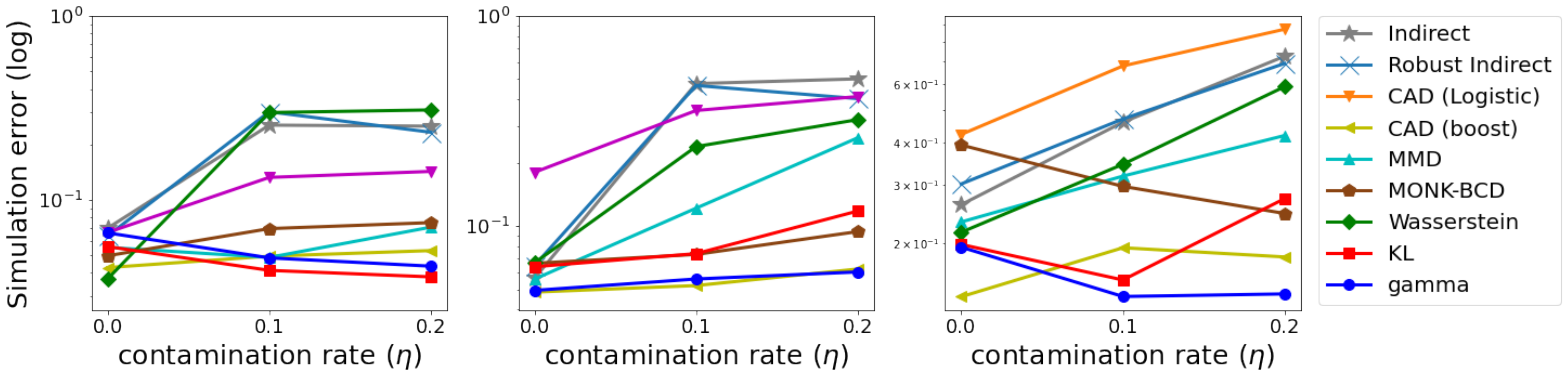}
    \caption{Simulation error for the BB (left), the MA2 (center), and the GK model (right) experiments.}
    \label{fig:energy_summary}
\end{figure*}
\paragraph{\textit{M}/\textit{G}/1-queuing Model (MG1):}
Queuing models are an example of stochastic models which are easy to sample
from but have intractable likelihoods~\cite{fearnhead12}.
The $M$/$G$/$1$-queuing model has been often used in the ABC literature~\cite{fearnhead12,Jiang18}.
This model has three parameters: $\theta = (\theta_1, \theta_2, \theta_3)$.
We adopted this model with the true parameter $\theta^{*} = (1,5,0.2)$.

From the experimental results in Figure~\ref{fig:energy_summary_mgm}, we can see that our method outperforms the other methods even if the data has no contamination.
In addition, our method also achieves better performance in terms of the MSE scores than the baseline methods (see Appendix \ref{app:full_MSE_scores}).
\diff{Figure~\ref{fig:posterior_sum2} indicates that the ABC posterior with our method places high density around the ground-truth parameter, e.g., for $\theta_{2}$.
On the other hand, the CAD via boosting places higher density around the wrong parameter than our method, e.g., for $\theta_{3}$.}


\begin{figure*}[th]
    \centering
    \includegraphics[scale=0.29]{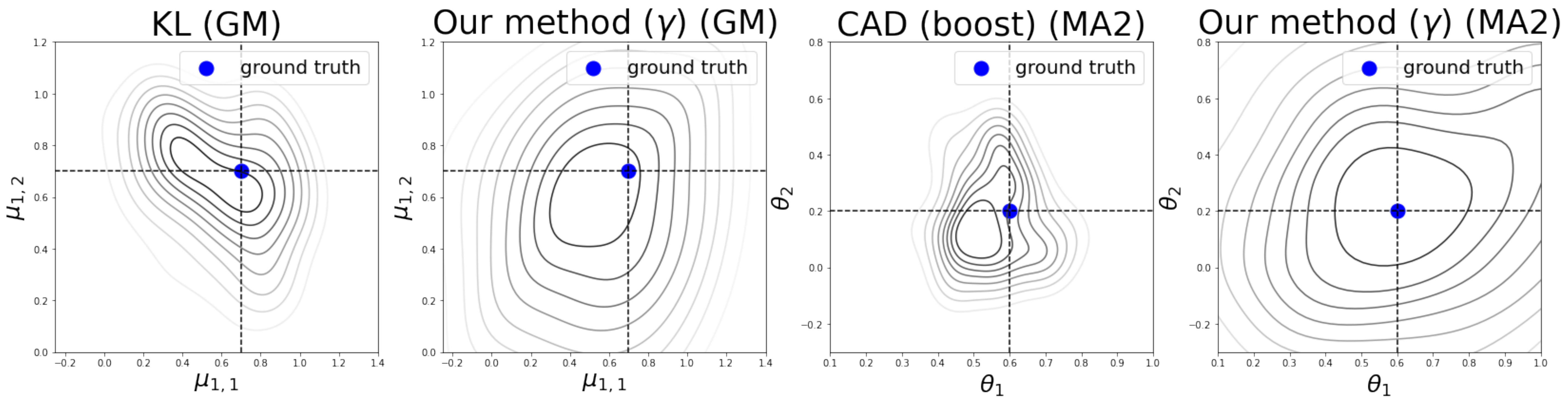}
    \caption{ABC posterior distributions of the GM and MA2 model experiments for $\eta = 0.2$ (excerpted).}
    \label{fig:posterior_sum1}
\end{figure*}
\begin{figure*}[th]
    \centering
    \includegraphics[scale=0.22]{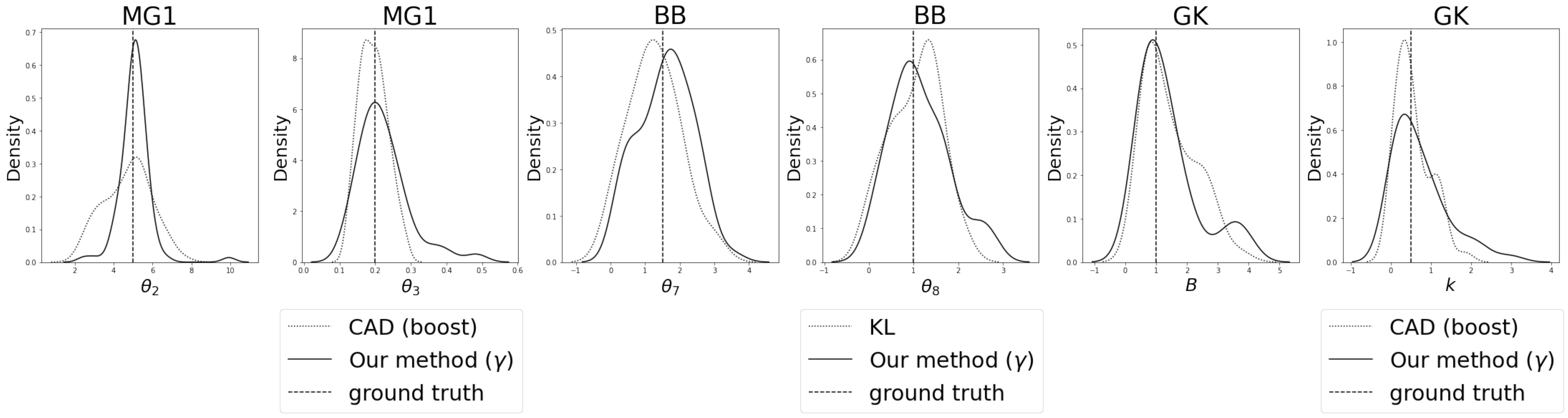}
    \caption{ABC posterior distributions of the MG1, BB, and GK model experiments for $\eta = 0.2$ (excerpted).}
    \label{fig:posterior_sum2}
\end{figure*}
\paragraph{Bivariate Beta Model (BB):}
The bivariate beta model can be used to model data sets exhibiting positive or negative correlation~\cite{Arnold11}.
This model was originally proposed as a model with $8$ parameters $\theta = (\theta_{1},\ldots,\theta_{8})$ by \citet{Arnold11}, and \citet{Crackel17} later reconsidered its $5$-parameter sub-model by restricting to $\theta_{3},\theta_{4},\theta_{5} = 0$. \citet{Jiang18} used the $5$-parameter models for ABC experiments and therefore we also adopted this with the true parameter $\theta^{*} = (3,2.5,2,1.5,1)$ as a benchmark model.

From the experimental results in Figure \ref{fig:energy_summary}, the KL- and the $\gamma$-divergence based methods achieve better performances than those of the baseline methods when the observed data are heavily contaminated.
The Wasserstein method achieves a better performance than the others when the data has no contamination; however, the performance becomes significantly worse when the contamination occurs.
In addition, the CAD with boosting method achieves a better performance in terms of the MSE; however, the simulation error is worse than the KL- and the $\gamma$-divergence based methods (see Appendix~\ref{app:full_MSE_scores}).
\diff{Figure~\ref{fig:posterior_sum2} shows that the ABC posterior with our method places higher density around the ground-truth parameter than the KL method.
In addition, the KL method sometimes places high density around the wrong parameter, e.g., for $\theta_{8}$.
However, in this experiment, the simulation error of the KL method is slightly better than that of our method.
This indicates that carefully tuning the hyperparameter $\gamma$ is important in our method.
}

\paragraph{Moving-average Model of Order 2 (MA2):}
The moving-average model is often used for modeling univariate time series.
\citet{Marin12} used the moving-average model of order $2$ as a benchmark model for ABC.
We adopted this model with $10$-length time series. For the unobserved noise distribution, we used Student's t-distribution with $5$ degrees of freedom.
We set the true parameter $\theta^{*} = (0.6,0.2)$.

From the experimental results in Figure \ref{fig:energy_summary}, our method and the CAD with boosting outperform the other methods.
The CAD with boosting achieves comparable performance to our method even if the observed data are contaminated; however, in terms of the time cost, our method is better than that of this method \diff{because the gradient boosting has $\mathcal{O}(Kd(n+m) + (n+m)\log B)$ time cost, where $K$ is the total number of trees and $B$ is s the maximum number of rows in each block (see \citet{Chen16})}.
In addition, in terms of the MSEs, our method achieves better performance when the observed data are heavily contaminated (see Appendix \ref{app:full_MSE_scores}).
\diff{In Figure~\ref{fig:posterior_sum1}, we found that the ABC posterior with our method places high density around the ground-truth parameter.
On the other hand, the CAD via boosting places higher density around the wrong parameter than our method.}

\paragraph{Multivariate $g$-and-$k$ Distribution (GK):}
The univariate $g$-and-$k$ distribution is a generalization of the standard normal distribution with extra parameters: the \emph{skewness} and the \emph{kurtosis}.
This distribution is known to have no analytical form of the density function, and the numerical evaluation of the likelihood function is costly \cite{rayner02}.
Thus, it is a model for which ABC is specifically suited \cite{fearnhead12,Allingham09}.
Some studies \cite{Drovandi11,Li15} also considered the multivariate $g$-and-$k$ distribution.
We adopted the multivariate model proposed by \citet{Drovandi11} with the true parameters $A^{*} = 3$, $B^{*} = 1$, $g^{*} = 2$, $k^{*} = 0.5$ and $\rho^{*} = -0.3$, where $A^{*},B^{*},g^{*},k^{*}$ control the location, the scale, the skewness and the kurtosis, respectively.

From the experimental results in Figure \ref{fig:energy_summary}, we can see that our method achieves better performance even if the observed data are contaminated, although the other baseline methods fail to give good scores.
In addition, in terms of the MSE, our method outperforms the other baseline methods when the observed data have heavy contamination (see Appendix \ref{app:full_MSE_scores}).
\diff{Figure~\ref{fig:posterior_sum2} shows that the ABC posterior with our method places slightly higher density on the ground-truth parameter than that of the CAD via boosting.
}




\section{Conclusion and Discussion}
\label{sec:conclusion}
We have proposed a $\gamma$-divergence estimator and used it as a robust data discrepancy for ABC.
We theoretically have guaranteed its robustness against outliers and its desirable asymptotic properties, i.e., the asymptotic unbiasedness and the almost sure convergence to the approximate posterior.
In addition, we have shown the redescending property of the ABC posterior of our method, indicating the high robustness of our method against extreme outliers.
Through the experiments on benchmark models, we empirically confirmed that our method is robust against heavy contamination by outliers.

\diff{Our work has two limitations: (i) our method can become statistically inefficient in high-dimensional cases (the curse of dimensionality) due to the $k$-NN based density estimation, and therefore (ii) its performance has only been confirmed in some low-dimensional cases.
To overcome these limitations, in the future, we will consider extending our method to a non-parametric estimation that can handle high dimensions and conduct experiments for more realistic high-dimensional cases.
Furthermore, on the basis of our idea, we plan to develop outlier-robust methods for other ABC approaches, e.g., ABC without discrepancies~\cite{george16,greenberg19a,Thomas20}.}

\diff{It is worth mentioning that there have been several studies on general losses within Bayesian procedures~\cite{Bissiri16,Knoblauch19}, and recent studies have connected the ideas of these studies with ABC~\cite{schmon21}.
Following these studies, it would be interesting to see whether our method stands in the framework of generalized approximate Bayesian inference under the condition of ABC with a general loss function~\cite{schmon21}.
}

\diff{In addition, research on the consistency and robustness of Bayesian estimation against model misspecification has attracted attention recently.
For example, \citet{Cherief20} has proposed Bayesian estimation based on a pseudo-likelihood by using MMD and has theoretically shown that it is effective in this problem setting.
We also explore the potential of our method in the context of model misspecification.
}

\subsubsection*{Acknowledgements}
MF and TT were supported by RIKEN Junior Research Associate Program.
MF and TT were supported by Toyota/Dwango AI scholarship.
MF was supported by JST CREST including AIP challenge program.
TT was supported by Masason Foundation.
MS was supported by KAKENHI 17H00757.
We appreciate Dr. Ikko Yamane, Kento Nozawa, Dr. Yoshihiro Nagano, and Han Bao for their kind effort to maintain the experimental environment.

\bibliography{gamma_abc,gamma_abc_teshima}

\begin{thebibliography}{83}
\providecommand{\natexlab}[1]{#1}
\providecommand{\url}[1]{\texttt{#1}}
\expandafter\ifx\csname urlstyle\endcsname\relax
  \providecommand{\doi}[1]{doi: #1}\else
  \providecommand{\doi}{doi: \begingroup \urlstyle{rm}\Url}\fi

\bibitem[Allingham et~al.(2009)Allingham, King, and Mengersen]{Allingham09}
D.~Allingham, Robert King, and K.~Mengersen.
\newblock Bayesian estimation of quantile distributions.
\newblock \emph{Statistics and Computing}, 19:\penalty0 189--201, 2009.

\bibitem[Arnold and Ng(2011)]{Arnold11}
B.~C. Arnold and H.~K.~T. Ng.
\newblock {Flexible Bivariate Beta Distributions}.
\newblock \emph{Journal of Multivariate Analysis}, 102\penalty0 (8):\penalty0
  1194--1202, 2011.

\bibitem[Aurenhammer and Klein(2000)]{AurenhammerVoronoi2000}
Franz Aurenhammer and Rolf Klein.
\newblock Voronoi {{Diagrams}}.
\newblock In \emph{Handbook of {{Computational Geometry}}}, pages 201--290.
  {Elsevier}, 2000.
\newblock ISBN 978-0-444-82537-7.
\newblock \doi{10.1016/B978-044482537-7/50006-1}.
\newblock URL
  \url{https://linkinghub.elsevier.com/retrieve/pii/B9780444825377500061}.

\bibitem[Basu et~al.(1998)Basu, Harris, Hjort, and Jones]{Basu98}
Ayanendranath Basu, Ian~R. Harris, Nils~L. Hjort, and M.~C. Jones.
\newblock {Robust and efficient estimation by minimising a density power
  divergence}.
\newblock \emph{Biometrika}, 85\penalty0 (3):\penalty0 549--559, 09 1998.

\bibitem[Bentley(1975)]{Bentley75}
Jon~Louis Bentley.
\newblock Multidimensional binary search trees used for associative searching.
\newblock \emph{Commun. ACM}, 18:\penalty0 509--517, September 1975.
\newblock ISSN 0001-0782.

\bibitem[Berlinet and Thomas-Agnan(2004)]{berlinet04}
Alain Berlinet and Christine Thomas-Agnan.
\newblock \emph{{Reproducing Kernel Hilbert Space in Probability and
  Statistics}}.
\newblock Springer Science \& Business Media, 2004.

\bibitem[Bernton et~al.(2017)Bernton, Jacob, Gerber, and Robert]{Bernton17}
Espen Bernton, Pierre Jacob, Mathieu Gerber, and Christian Robert.
\newblock {Inference in generative models using the Wasserstein distance}.
\newblock \emph{arXiv preprint arXiv:1701.05146}, abs/1701.05146, 2017.

\bibitem[Biau and Devroye(2015)]{Biauknearest2015}
G^^c3^^a9rard Biau and Luc Devroye.
\newblock The k-nearest neighbor density estimate.
\newblock In G^^c3^^a9rard Biau and Luc Devroye, editors, \emph{Lectures on the
  {{Nearest Neighbor Method}}}, Springer {{Series}} in the {{Data Sciences}},
  pages 25--32. {Springer International Publishing}, {Cham}, 2015.
\newblock ISBN 978-3-319-25388-6.
\newblock \doi{10.1007/978-3-319-25388-6_3}.
\newblock URL \url{https://doi.org/10.1007/978-3-319-25388-6_3}.

\bibitem[Billingsley(1995)]{BillingsleyProbability1995}
Patrick Billingsley.
\newblock \emph{Probability and {{Measure}}}.
\newblock {Wiley}, third edition, 1995.
\newblock ISBN 0-471-00710-2.
\newblock URL
  \url{http://www.amazon.com/exec/obidos/redirect?tag=citeulike07-20\&path=ASIN/0471007102}.
\newblock Published: Hardcover.

\bibitem[Bissiri et~al.(2016)Bissiri, Holmes, and Walker]{Bissiri16}
P.~G. Bissiri, C.~C. Holmes, and S.~G. Walker.
\newblock A general framework for updating belief distributions.
\newblock \emph{Journal of the Royal Statistical Society: Series B (Statistical
  Methodology)}, 78\penalty0 (5):\penalty0 1103--1130, 2016.

\bibitem[Blum and Tran(2010)]{Blum20}
M.~G.~B. Blum and V.~C. Tran.
\newblock {HIV with contact-tracing: a case study in Approximate Bayesian
  Computation}.
\newblock \emph{{Biostatistics}}, 11\penalty0 (4):\penalty0 644--660, 2010.

\bibitem[Blum et~al.(2013)Blum, Nunes, Prangle, and Sisson]{blum13}
M.~G.~B. Blum, M.~A. Nunes, D.~Prangle, and S.~A. Sisson.
\newblock {A Comparative Review of Dimension Reduction Methods in Approximate
  Bayesian Computation}.
\newblock \emph{Statistical Science}, 28\penalty0 (2):\penalty0 189--208, 2013.

\bibitem[Blum and Fran^^c3^^a7ois(2010)]{Blum10}
Michael Blum and Olivier Fran^^c3^^a7ois.
\newblock {Non-linear regression models for Approximate Bayesian Computation}.
\newblock \emph{Statistics and Computing}, 20:\penalty0 63--73, 2010.

\bibitem[Burkard et~al.(2009)Burkard, Dell’Amico, and Martello]{Burkard09}
Rainer Burkard, Mauro Dell’Amico, and Silvano Martello.
\newblock \emph{{Assignment Problems}}.
\newblock Society for Industrial and Applied Mathematics, 2009.

\bibitem[Cameron and Pettitt(2012)]{Cameron12}
E.~Cameron and A.~N. Pettitt.
\newblock {Approximate Bayesian Computation for astronomical model analysis: a
  case study in galaxy demographics and morphological transformation at high
  redshift}.
\newblock \emph{Monthly Notices of the Royal Astronomical Society},
  425\penalty0 (1):\penalty0 44--65, 2012.

\bibitem[Chen and Guestrin(2016)]{Chen16}
Tianqi Chen and Carlos Guestrin.
\newblock {XGBoost: A Scalable Tree Boosting System}.
\newblock In \emph{Proceedings of the 22nd ACM SIGKDD International Conference
  on Knowledge Discovery and Data Mining (KDD)}, page 785^^e2^^80^^93794, 2016.

\bibitem[Cherief-Abdellatif and Alquier(2020)]{Cherief20}
Badr-Eddine Cherief-Abdellatif and Pierre Alquier.
\newblock {MMD-Bayes: Robust Bayesian Estimation via Maximum Mean Discrepancy}.
\newblock In \emph{Proceedings of The 2nd Symposium on Advances in Approximate
  Bayesian Inference}, volume 118, pages 1--21, 2020.

\bibitem[Christopher et~al.(2017)Christopher, Lapointe, Wimer, Hayden, Grooms,
  Rieker, and Hamlington]{Jason17}
Jason Christopher, Caelan Lapointe, Nicholas Wimer, Torrey Hayden, Ian Grooms,
  Gregory Rieker, and Peter Hamlington.
\newblock {Parameter Estimation for a Turbulent Buoyant Jet using Approximate
  Bayesian Computation}.
\newblock In \emph{55th AIAA Aerospace Sciences Meeting}, 2017.

\bibitem[Crackel and Flegal(2017)]{Crackel17}
Roberto Crackel and James Flegal.
\newblock Bayesian inference for a flexible class of bivariate beta
  distributions.
\newblock \emph{Journal of Statistical Computation and Simulation}, 87\penalty0
  (2):\penalty0 295--312, 2017.

\bibitem[Croux and Dehon(2001)]{croux01}
Christophe Croux and Catherine Dehon.
\newblock {Robust linear discriminant analysis using S-estimators}.
\newblock \emph{Canadian Journal of Statistics}, 29\penalty0 (3):\penalty0
  473--493, 2001.

\bibitem[Cuturi(2013)]{Cuturi13}
Marco Cuturi.
\newblock {Sinkhorn Distances: Lightspeed Computation of Optimal Transport}.
\newblock In \emph{Advances in Neural Information Processing Systems 26
  (NeurIPS)}, pages 2292--2300, 2013.

\bibitem[Cuturi and Doucet(2014)]{cuturi14}
Marco Cuturi and Arnaud Doucet.
\newblock {Fast Computation of Wasserstein Barycenters}.
\newblock In \emph{Proceedings of the 31st International Conference on Machine
  Learning (ICML)}, pages 685--693, 2014.

\bibitem[Dang(2015)]{DangComplex2015}
Nguyen~Viet Dang.
\newblock Complex powers of analytic functions and meromorphic renormalization
  in {{QFT}}.
\newblock \emph{arXiv:1503.00995 [math-ph]}, March 2015.
\newblock URL \url{http://arxiv.org/abs/1503.00995}.

\bibitem[Drovandi and Pettitt(2011)]{Drovandi11}
Christopher~C. Drovandi and Tony Pettitt.
\newblock {Estimation of parameters for macroparasite population evolution
  using approximate Bayesian computation}.
\newblock \emph{Biometrics}, 67\penalty0 (1):\penalty0 225--233, 2011.

\bibitem[Drovandi et~al.(2015)Drovandi, Pettitt, and Lee]{Drovandi15}
Christopher~C. Drovandi, Anthony Pettitt, and Anthony Lee.
\newblock {Bayesian Indirect Inference Using a Parametric Auxiliary Model}.
\newblock \emph{Statistical Science}, 30\penalty0 (1):\penalty0 72--95, 2015.

\bibitem[Edelsbrunner et~al.(1986)Edelsbrunner, O'Rourke, and
  Seidel]{EdelsbrunnerConstructing1986}
Herbert Edelsbrunner, Joseph O'Rourke, and Raimund Seidel.
\newblock Constructing arrangements of lines and hyperplanes with applications.
\newblock \emph{SIAM Journal on Computing}, 15\penalty0 (2), 1986.
\newblock URL \url{https://research-explorer.app.ist.ac.at/record/4105}.

\bibitem[Fearnhead and Prangle(2012)]{fearnhead12}
Paul Fearnhead and Dennis Prangle.
\newblock Constructing summary statistics for approximate {Bayesian}
  computation: semi-automatic approximate {Bayesian} computation:
  {Semi}-automatic {Approximate} {Bayesian} {Computation}.
\newblock \emph{Journal of the Royal Statistical Society: Series B (Statistical
  Methodology)}, 74\penalty0 (3):\penalty0 419--474, 2012.

\bibitem[Feng et~al.(2014)Feng, Xu, Mannor, and Yan]{feng14}
Jiashi Feng, Huan Xu, Shie Mannor, and Shuicheng Yan.
\newblock {Robust Logistic Regression and Classification}.
\newblock In \emph{Advances in Neural Information Processing Systems 27
  (NeurIPS)}, pages 253--261, 2014.

\bibitem[Fujisawa and Eguchi(2008)]{Fujisawa08}
Hironori Fujisawa and Shinto Eguchi.
\newblock {Robust Parameter Estimation with a Small Bias Against Heavy
  Contamination}.
\newblock \emph{Journal of Multivariate Analysis}, 99\penalty0 (9):\penalty0
  2053--2081, October 2008.

\bibitem[Futami et~al.(2018)Futami, Sato, and Sugiyama]{futami18}
Futoshi Futami, Issei Sato, and Masashi Sugiyama.
\newblock {Variational Inference based on Robust Divergences}.
\newblock In \emph{International Conference on Artificial Intelligence and
  Statistics (AISTATS)}, 2018.

\bibitem[Gleim and Pigorsch(2013)]{Gleim13}
A.~Gleim and C.~Pigorsch.
\newblock {Approximate Bayesian computation with indirect summary statistics}.
\newblock \emph{Draft paper:
  http://ect-pigorsch.mee.uni-bonn.de/data/research/papers}, 2013.

\bibitem[Greenberg et~al.(2019)Greenberg, Nonnenmacher, and
  Macke]{greenberg19a}
David Greenberg, Marcel Nonnenmacher, and Jakob Macke.
\newblock Automatic posterior transformation for likelihood-free inference.
\newblock In \emph{Proceedings of the 36th International Conference on Machine
  Learning (ICML)}, 2019.

\bibitem[Gretton et~al.(2012)Gretton, Borgwardt, Rasch, Sch{\"o}lkopf, and
  Smola]{Gretton12}
A.~Gretton, K.~Borgwardt, M.~Rasch, B.~Sch{\"o}lkopf, and A.~Smola.
\newblock {A Kernel Two-Sample Test }.
\newblock \emph{Journal of Machine Learning Research}, 13:\penalty0 723--773,
  2012.

\bibitem[Grimmett and Stirzaker(2001)]{grimmett01}
G.R. Grimmett and D.R. Stirzaker.
\newblock \emph{Probability and random processes}.
\newblock Oxford university press, 2001.

\bibitem[Gutmann et~al.(2018)Gutmann, Dutta, Kaski, and Corander]{Gutmann18}
{Michael U.} Gutmann, Ritabrata Dutta, Samuel Kaski, and Jukka Corander.
\newblock {Likelihood-free inference via classification}.
\newblock \emph{Statistics and Computing}, 28\penalty0 (2):\penalty0 411--425,
  2018.

\bibitem[Hampel(2005)]{HampelRobust2005}
Frank~R. Hampel, editor.
\newblock \emph{Robust {{Statistics}}: {{The Approach Based}} on {{Influence
  Functions}}}.
\newblock Wiley Series in Probability and Mathematical Statistics. {Wiley},
  {New York}, digital print edition, 2005.
\newblock ISBN 978-0-471-73577-9.
\newblock OCLC: 255133771.

\bibitem[H\"{a}rdle et~al.(2006)H\"{a}rdle, M^^c3^^bcller, Sperlich, and
  Werwatz]{Hardle06}
Wolfgang~Karl H\"{a}rdle, Marlene M^^c3^^bcller, Stefan Sperlich, and Axel
  Werwatz.
\newblock \emph{{Nonparametric and Semiparametric Models}}.
\newblock Springer, 01 2006.
\newblock ISBN 978-3-642-62076-8.

\bibitem[Huber(1964)]{Huber64}
Peter~J. Huber.
\newblock Robust estimation of a location parameter.
\newblock \emph{Annals of Mathematical Statistics}, 35\penalty0 (1):\penalty0
  73--101, 1964.

\bibitem[Huber et~al.(1981)Huber, Wiley, and InterScience]{Huber09}
P.J. Huber, J.~Wiley, and W.~InterScience.
\newblock \emph{{Robust statistics}}.
\newblock Wiley New York, 1981.

\bibitem[jean Kim et~al.(2006)jean Kim, Magnani, and Boyd]{kim06}
Seung jean Kim, Alessandro Magnani, and Stephen Boyd.
\newblock {Robust Fisher Discriminant Analysis}.
\newblock In \emph{Advances in Neural Information Processing Systems 18
  (NeurIPS)}, pages 659--666, 2006.

\bibitem[Jewson et~al.(2018)Jewson, Smith, and Holmes]{Jewson18}
Jack Jewson, Jim~Q. Smith, and Chris Holmes.
\newblock {Principles of Bayesian Inference Using General Divergence Criteria}.
\newblock \emph{Entropy}, 20\penalty0 (6), 2018.
\newblock ISSN 1099-4300.

\bibitem[Jiang et~al.(2018)Jiang, Wu, and Wong]{Jiang18}
Bai Jiang, Tung-Yu Wu, and Wing~Hung Wong.
\newblock {Approximate Bayesian Computation with Kullback-Leibler Divergence as
  Data Discrepancy}.
\newblock In \emph{Proceedings of the 21th International Conference on
  Artificial Intelligence and Statistics (AISTATS)}, volume~84, 2018.

\bibitem[Kajihara et~al.(2018)Kajihara, Kanagawa, Yamazaki, and
  Fukumizu]{kajihara18a}
Takafumi Kajihara, Motonobu Kanagawa, Keisuke Yamazaki, and Kenji Fukumizu.
\newblock {Kernel Recursive {ABC}: Point Estimation with Intractable
  Likelihood}.
\newblock In \emph{Proceedings of the 35th International Conference on Machine
  Learning (ICML)}, pages 2400--2409, 2018.

\bibitem[Kanamori et~al.(2009)Kanamori, Hido, and Sugiyama]{kanamori09}
Takafumi Kanamori, Shohei Hido, and Masashi Sugiyama.
\newblock {A Least-Squares Approach to Direct Importance Estimation}.
\newblock \emph{Journal of Machine Learning Research}, 10:\penalty0
  1391^^e2^^80^^931445, 2009.

\bibitem[Knoblauch et~al.(2018)Knoblauch, Jewson, and Damoulas]{Knoblauch18}
Jeremias Knoblauch, Jack~E Jewson, and Theodoros Damoulas.
\newblock {Doubly Robust Bayesian Inference for Non-Stationary Streaming Data
  with $\beta$-Divergences}.
\newblock In \emph{Advances in Neural Information Processing Systems
  (NeurIPS)}, 2018.

\bibitem[Knoblauch et~al.(2019)Knoblauch, Jewson, and Damoulas]{Knoblauch19}
Jeremias Knoblauch, J.~Jewson, and T.~Damoulas.
\newblock {Generalized Variational Inference: Three arguments for deriving new
  Posteriors}.
\newblock \emph{arXiv preprint arXiv:1904.02063}, abs/1904.02063, 2019.

\bibitem[Kremer et~al.(2017)Kremer, Stensbo-Smidt, Gieseke, Pedersen, and
  Igel]{Kremer17}
Jan Kremer, Kristoffer Stensbo-Smidt, Fabian Gieseke, Kim Pedersen, and
  Christian Igel.
\newblock {Big Universe, Big Data: Machine Learning and Image Analysis for
  Astronomy}.
\newblock \emph{IEEE Intelligent Systems}, 32:\penalty0 16--22, 03 2017.

\bibitem[Leonenko et~al.(2008)Leonenko, Pronzato, and Savani]{Leonenko08a}
Nikolai Leonenko, Luc Pronzato, and Vippal Savani.
\newblock {A class of {R}\'{e}nyi information estimators for multidimensional
  densities}.
\newblock \emph{Annals of Statistics}, 36\penalty0 (5):\penalty0 2153--2182,
  2008.

\bibitem[Lerasle et~al.(2019)Lerasle, Szabo, Mathieu, and Lecue]{Lerasle19}
Matthieu Lerasle, Zoltan Szabo, Timoth{\'e}e Mathieu, and Guillaume Lecue.
\newblock {{MONK} Outlier-Robust Mean Embedding Estimation by Median-of-Means}.
\newblock In \emph{Proceedings of the 36th International Conference on Machine
  Learning (ICML)}, volume~97 of \emph{Proceedings of Machine Learning
  Research}, pages 3782--3793, 2019.

\bibitem[Li et~al.(2015)Li, Nott, Fan, and Sisson]{Li15}
Jingjing Li, David Nott, Yanan Fan, and Scott Sisson.
\newblock {Extending approximate Bayesian computation methods to high
  dimensions via Gaussian copula}.
\newblock \emph{Computational Statistics \& Data Analysis}, 106:\penalty0
  77--89, 2015.

\bibitem[Loftsgaarden and Quesenberry(1965)]{loftsgaarden65}
D.~O. Loftsgaarden and C.~P. Quesenberry.
\newblock {A Nonparametric Estimate of a Multivariate Density Function}.
\newblock \emph{The Annals of Mathematical Statistics}, 36\penalty0
  (3):\penalty0 1049--1051, 06 1965.

\bibitem[Maneewongvatana and Mount(2001)]{Maneewongvatana01}
Songrit Maneewongvatana and David~M. Mount.
\newblock {On the Efficiency of Nearest Neighbor Searching with Data Clustered
  in Lower Dimensions}.
\newblock In \emph{Computational Science -- ICCS 2001}, pages 842--851, Berlin,
  Heidelberg, 2001. Springer Berlin Heidelberg.

\bibitem[Marin et~al.(2012)Marin, Pudlo, Robert, and Ryder]{Marin12}
Jean~Michel Marin, Pierre Pudlo, Christian~P. Robert, and Robin~J. Ryder.
\newblock {Approximate Bayesian Computational methods}.
\newblock \emph{Statistics and Computing}, 22:\penalty0 1167--1180, 2012.

\bibitem[Marjoram et~al.(2003)Marjoram, Molitor, Plagnol, and
  Tavar{\'e}]{Marjoram03}
Paul Marjoram, John Molitor, Vincent Plagnol, and Simon Tavar{\'e}.
\newblock {Markov chain Monte Carlo without likelihoods}.
\newblock \emph{Proceedings of the National Academy of Sciences}, 100\penalty0
  (26):\penalty0 15324--15328, 2003.

\bibitem[Maronna(2019)]{MaronnaRobust2019}
Ricardo~A. Maronna.
\newblock \emph{Robust Statistics: Theory and Methods (with {{R}})}.
\newblock Wiley Series in Probability and Statistics. {WIley}, {Hoboken, NJ},
  second edition edition, 2019.
\newblock ISBN 978-1-119-21467-0 978-1-119-21466-3.

\bibitem[Mityagin(2015)]{MityaginZero2015}
Boris Mityagin.
\newblock The {{Zero Set}} of a {{Real Analytic Function}}.
\newblock \emph{arXiv:1512.07276 [math]}, December 2015.
\newblock URL \url{http://arxiv.org/abs/1512.07276}.

\bibitem[Moral et~al.(2012)Moral, Doucet, and Jasra]{Pierre12}
Pierre Moral, Arnaud Doucet, and Ajay Jasra.
\newblock {An Adaptive Sequential Monte Carlo Method for Approximate Bayesian
  Computation}.
\newblock \emph{Statistics and Computing}, 22\penalty0 (5):\penalty0
  1009^^e2^^80^^931020, 2012.

\bibitem[Nakagawa and Hashimoto(2020)]{Nakagawa20}
Tomoyuki Nakagawa and Shintaro Hashimoto.
\newblock {Robust Bayesian inference via $\gamma$-divergence}.
\newblock \emph{Communications in Statistics - Theory and Methods}, 49\penalty0
  (2):\penalty0 343--360, 2020.

\bibitem[Papamakarios and Murray(2016)]{george16}
George Papamakarios and Iain Murray.
\newblock Fast $\epsilon$-free inference of simulation models with bayesian
  conditional density estimation.
\newblock In \emph{Advances in Neural Information Processing Systems 29
  (NeurIPS)}, pages 1028--1036, 2016.

\bibitem[Park et~al.(2016)Park, Jitkrittum, and Sejdinovic]{Park16}
Mijung Park, Wittawat Jitkrittum, and Dino Sejdinovic.
\newblock {K2-ABC: Approximate Bayesian Computation with Kernel Embeddings}.
\newblock In \emph{Proceedings of the 19th International Conference on
  Artificial Intelligence and Statistics (AISTATS)}, volume~51 of
  \emph{Proceedings of Machine Learning Research}, pages 398--407, 2016.

\bibitem[P{\'e}rez-Cruz(2008)]{PrezCruz08}
Fernando P{\'e}rez-Cruz.
\newblock {Kullback-Leibler divergence estimation of continuous distributions}.
\newblock \emph{2008 IEEE International Symposium on Information Theory}, pages
  1666--1670, 2008.

\bibitem[Peters and Sisson(2006)]{Peters06}
G.W. Peters and S.A. Sisson.
\newblock {Bayesian inference, Monte Carlo sampling and operational risk}.
\newblock \emph{Journal of Operational Risk}, 1\penalty0 (3):\penalty0 27--50,
  2006.

\bibitem[Peters et~al.(2012)Peters, Sisson, and Fan]{Peters12}
G.W. Peters, S.A. Sisson, and Y.~Fan.
\newblock {Likelihood-free Bayesian inference for $\alpha$-stable models}.
\newblock \emph{Computational Statistics \& Data Analysis}, 56\penalty0
  (11):\penalty0 3743--3756, 2012.

\bibitem[Poczos and Schneider(2011)]{poczos11a}
Barnabas Poczos and Jeff Schneider.
\newblock {On the Estimation of $\alpha$-Divergences}.
\newblock In \emph{Proceedings of the 14th International Conference on
  Artificial Intelligence and Statistics (AISTATS)}, volume~15, 2011.

\bibitem[Pritchard et~al.(1999)Pritchard, Seielstad, Perez-Lezaun, and
  Feldman]{Pritchard99}
J.~K. Pritchard, M.~T. Seielstad, A.~Perez-Lezaun, and M.~W. Feldman.
\newblock {Population growth of human Y chromosomes: a study of Y chromosome
  microsatellites.}
\newblock \emph{Molecular Biology and Evolution}, 16\penalty0 (12):\penalty0
  1791--1798, 1999.

\bibitem[Rayner and Macgillivray(2002)]{rayner02}
G.~Rayner and H.~Macgillivray.
\newblock {Numerical maximum likelihood estimation for the g-and-k and
  generalized g-and-h distributions}.
\newblock \emph{Statistics and Computing}, 12\penalty0 (1):\penalty0 57--75,
  2002.

\bibitem[Rudin(1976)]{RudinPrinciples1976}
Walter Rudin.
\newblock \emph{Principles of {{Mathematical Analysis}}}.
\newblock McGraw-Hill, 1976.

\bibitem[Ruli et~al.(2020)Ruli, Sartori, and Ventura]{Ruli20}
Erlis Ruli, Nicola Sartori, and Laura Ventura.
\newblock {Robust approximate Bayesian inference}.
\newblock \emph{Journal of Statistical Planning and Inference}, 205:\penalty0
  10--22, 2020.

\bibitem[Schmon et~al.(2021)Schmon, Cannon, and Knoblauch]{schmon21}
Sebastian~M Schmon, Patrick~W Cannon, and Jeremias Knoblauch.
\newblock {Generalized Posteriors in Approximate Bayesian Computation}.
\newblock In \emph{Third Symposium on Advances in Approximate Bayesian
  Inference}, 2021.

\bibitem[Scott(2015)]{Scott15}
David~W Scott.
\newblock \emph{{Multivariate density estimation: theory, practice, and
  visualization; 2nd ed.}}
\newblock Wiley series in probability and statistics. Wiley, Hoboken, NJ, 2015.
\newblock \doi{10.1002/9781118575574}.

\bibitem[Sisson et~al.(2007)Sisson, Fan, and Tanaka]{Sisson07}
S.~A. Sisson, Y.~Fan, and Mark~M. Tanaka.
\newblock {Sequential Monte Carlo without likelihoods}.
\newblock \emph{Proceedings of the National Academy of Sciences}, 104\penalty0
  (6):\penalty0 1760--1765, 2007.

\bibitem[Sisson et~al.(2009)Sisson, Fan, and Tanaka]{Sisson09}
S.~A. Sisson, Y.~Fan, and Mark~M. Tanaka.
\newblock {Correction for Sisson et al., "Sequential Monte Carlo without
  likelihoods"}.
\newblock \emph{Proceedings of the National Academy of Sciences}, 106\penalty0
  (39):\penalty0 16889--16889, 2009.

\bibitem[Smola et~al.(2007)Smola, Gretton, Song, and Sch{\"o}lkopf]{smola07}
A.~Smola, A.~Gretton, L.~Song, and B.~Sch{\"o}lkopf.
\newblock {A Hilbert Space Embedding for Distributions}.
\newblock \emph{Algorithmic Learning Theory: 18th International Conference (ALT
  2007)}, pages 13--31, 2007.

\bibitem[Staerman et~al.(2020)Staerman, Laforgue, Mozharovskyi, and
  d'Alch^^c3^^a9 Buc]{staerman20}
Guillaume Staerman, Pierre Laforgue, Pavlo Mozharovskyi, and Florence
  d'Alch^^c3^^a9 Buc.
\newblock {When OT meets MoM: Robust estimation of Wasserstein Distance}.
\newblock \emph{arXiv preprint arXiv:2006.10325}, abs/2006.10325, 2020.

\bibitem[Sugiyama et~al.(2008)Sugiyama, Suzuki, Nakajima, Kashima, von
  B^^c3^^bcnau, and Kawanabe]{sugi08}
Masashi Sugiyama, Taiji Suzuki, Shinichi Nakajima, Hisashi Kashima, Paul von
  B^^c3^^bcnau, and Motoaki Kawanabe.
\newblock Direct importance estimation for covariate shift adaptation.
\newblock \emph{Annals of the Institute of Statistical Mathematics},
  60:\penalty0 699--746, 02 2008.

\bibitem[Tavare et~al.(1997)Tavare, Balding, Griffiths, and Donnelly]{Tavare97}
Simon Tavare, {David J} Balding, {Robert C} Griffiths, and Peter Donnelly.
\newblock {Inferring Coalescence Times from DNA Sequence Data}.
\newblock \emph{Genetics}, 162\penalty0 (2):\penalty0 505--518, 1997.

\bibitem[Thomas et~al.(2020)Thomas, Dutta, Corander, Kaski, and
  Gutmann]{Thomas20}
O.~Thomas, R.~Dutta, J.~Corander, S.~Kaski, and M.U. Gutmann.
\newblock {Likelihood-Free Inference by Ratio Estimation}.
\newblock \emph{Bayesian Analysis}, 2020.

\bibitem[van~der Vaart(1998)]{vaart_1998}
A.~W. van~der Vaart.
\newblock \emph{{Asymptotic Statistics}}.
\newblock Cambridge Series in Statistical and Probabilistic Mathematics.
  Cambridge University Press, 1998.

\bibitem[{Wang} et~al.(2009){Wang}, {Kulkarni}, and {Verdu}]{Wang09}
Q.~{Wang}, S.~R. {Kulkarni}, and S.~{Verdu}.
\newblock {Divergence Estimation for Multidimensional Densities Via
  $k$-Nearest-Neighbor Distances}.
\newblock \emph{IEEE Transactions on Information Theory}, 55\penalty0
  (5):\penalty0 2392--2405, 2009.

\bibitem[Wegmann et~al.(2009)Wegmann, Leuenberger, and Excoffier]{Wegmann09}
Daniel Wegmann, Christoph Leuenberger, and Laurent Excoffier.
\newblock {Efficient Approximate Bayesian Computation Coupled With Markov Chain
  Monte Carlo Without Likelihood}.
\newblock \emph{Genetics}, 182\penalty0 (4):\penalty0 1207--1218, 2009.

\bibitem[Wilkinson(2013)]{Wilkinson13}
Richard Wilkinson.
\newblock {Approximate Bayesian Computation (ABC) gives exact results under the
  assumption of model error}.
\newblock \emph{Statistical applications in genetics and molecular biology},
  12:\penalty0 1--13, 2013.

\bibitem[Windham(1995)]{windham95}
Michael~P Windham.
\newblock Robustifying model fitting.
\newblock \emph{Journal of the Royal Statistical Society. Series B
  (Methodological)}, pages 599--609, 1995.

\bibitem[Wood(2010)]{Simon10}
Simon~N. Wood.
\newblock Statistical inference for noisy nonlinear ecological dynamic systems.
\newblock \emph{Nature}, 466\penalty0 (7310):\penalty0 1102--1104, 2010.

\end{thebibliography}
\bibliographystyle{plainnat}

\clearpage
\appendix
\onecolumn
\aistatstitle{Appendix of
``$\gamma$-ABC: Outlier-Robust Approximate Bayesian Computation Based on a Robust Divergence Estimator''}
\tableofcontents
\clearpage
\addtocontents{toc}{\protect\setcounter{tocdepth}{2}}

\clearpage
\section{Derivation of \texorpdfstring{$\gamma$}{Lg}-divergence Estimator}
\label{proof:gamma_estimator}
\setcounter{definition}{0}
We show how to derive the $k$-NN based $\gamma$-divergence estimator in  \eqref{eq:default_gamma_est}.
The $k$-NN based $\gamma$-divergence estimator and its derivation is as follows.
\begin{align*}
    &\widehat{D}_{\gamma}(X^{n}\|Y^{m}) =  \frac{1}{\gamma(1+\gamma)}
    \times\left(
    \log \frac{\bigg(\displaystyle{\frac{1}{n} \sum_{i=1}^{n} (\frac{\bar{c}}{k}\hat{p}_{k}(x_{i}))^{\gamma}\bigg) } \bigg(\frac{1}{m} \displaystyle{\sum_{j=1}^{m} (\frac{\bar{c}}{k}\hat{q}_{k}(y_{j}))^{\gamma} \bigg)^{\gamma}}}
    {\bigg(\displaystyle \frac{1}{n}\sum_{i=1}^{n} (\frac{\bar{c}}{k}\hat{q}_{k}(x_{i}))^{\gamma} \bigg)^{1+\gamma}}
    \right),
\end{align*}
where $\gamma (\in \mathbb{R}) > 0$.

We rewrite Eq.~\eqref{gamma_div} as 
\normalsize
\begin{align}
\label{eq:gamma_div_est_formulate1}
    &\frac{1}{\gamma(1+\gamma)} \log \int_{\mathcal{M}} p(x) p^{\gamma}(x) \mathrm{d}x - \frac{1}{\gamma} \log \int_{\mathcal{M}} p(x)q^{\gamma}(x) \mathrm{d}x \frac{1}{1+\gamma} \log \int_{\mathcal{M}'}
    q(y) q^{\gamma}(y) \mathrm{d}y \notag \\
    &= \frac{1}{\gamma(1+\gamma)} \bigg( \log \mathbb{E}_{p(x)} \bigg[ p^{\gamma}(x) \bigg] - (1+\gamma) \log \mathbb{E}_{p(x)} \bigg[ q^{\gamma}(x) \bigg] + \gamma \log \mathbb{E}_{q(y)} \bigg[ q^{\gamma}(y) \bigg] \bigg),
\end{align}
\normalsize
where $\mathcal{M}$ and $\mathcal{M}'$ are the supports of $p$ and $q$.
By simply plugging Eqs.~\eqref{def:knn_est_p} and \eqref{def:knn_est_q} into Eq.~\eqref{eq:gamma_div_est_formulate1}, 
we estimate $D_{\gamma}(p\|q)$ with
\footnotesize
\begin{align*}
    &\widehat{D}_{\gamma}(X^{n}\|Y^{m}) \notag \\
    &= \frac{1}{\gamma(1+\gamma)} \bigg[ \log \bigg( \frac{1}{n} \sum_{i=1}^{n} \bigg(\frac{k}{(n-1)\bar{c}\rho_{k}^{d}(i)} \bigg)^{\gamma} \bigg)
    - (1+\gamma) \log \bigg(\frac{1}{n} \sum_{i=1}^{n} \bigg(\frac{k}{m\bar{c}\nu_{k}^{d}(i)} \bigg)^{\gamma} \bigg) + \gamma \log \bigg(\frac{1}{m} \sum_{j=1}^{m} \bigg(\frac{k}{(m-1)\bar{c}\bar{\rho}_{k}^{d}(j)} \bigg)^{\gamma} \bigg) \bigg] \notag \\
    &= \frac{1}{\gamma(1+\gamma)} \left(\log \frac{\bigg(\displaystyle\frac{1}{n}\sum_{i=1}^{n} \frac{1}{(n-1)^{\gamma} \rho_{k}^{d\gamma}(i)} \bigg) \bigg( \frac{1}{m} \displaystyle\sum_{j=1}^{m} \frac{1}{(m-1)^{\gamma}\bar{\rho}_{k}^{d\gamma}(j)} \bigg)^{\gamma}}{\bigg( \displaystyle\frac{1}{n}\displaystyle\sum_{i=1}^{n} \frac{1}{m^{\gamma} \nu_{k}^{d\gamma}(i)} \bigg)^{1+\gamma}} \right) \\
    &= \frac{1}{\gamma(1+\gamma)} \left(\log \frac{\bigg(\displaystyle\frac{1}{n}\displaystyle\sum_{i=1}^{n} \frac{\bar{c}^{\gamma}}{k^{\gamma}} \cdot \frac{k^{\gamma}}{(n-1)^{\gamma} \bar{c}^{\gamma} \rho_{k}^{d\gamma}(i)} \bigg) \bigg( \displaystyle\frac{1}{m} \displaystyle\sum_{j=1}^{m} \frac{\bar{c}^{\gamma}}{k^{\gamma}} \cdot \frac{k^{\gamma}}{(m-1)^{\gamma} \bar{c}^{\gamma}\bar{\rho}_{k}^{d\gamma}(j)} \bigg)^{\gamma}}{\bigg( \displaystyle\frac{1}{n}\sum_{i=1}^{n} \frac{\bar{c}^{\gamma}}{k^{\gamma}} \cdot \frac{k^{\gamma}}{m^{\gamma} \bar{c}^{\gamma} \nu_{k}^{d\gamma}(i)} \bigg)^{1+\gamma}} \right) \\
    &= \frac{1}{\gamma(1+\gamma)} \left(
    \log \frac{\bigg(\displaystyle\frac{1}{n} \sum_{i=1}^{n} (\frac{\bar{c}}{k}\hat{p}_{k}(x_{i}))^{\gamma}\bigg)  \bigg(\displaystyle\frac{1}{m} \sum_{j=1}^{m} (\frac{\bar{c}}{k}\hat{q}_{k}(y_{j}))^{\gamma} \bigg)^{\gamma}}
    {\bigg(\displaystyle \frac{1}{n}\sum_{i=1}^{n} (\frac{\bar{c}}{k}\hat{q}_{k}(x_{i}))^{\gamma} \bigg)^{1+\gamma}}
    \right).
\end{align*}
In second equation, because of logarithm, $k/\bar{c}$ in first term is vanished.
The third equation holds because $k^{\gamma}/\bar{c}^{\gamma} \cdot \bar{c}^{\gamma}/k^{\gamma} = 1$.

Therefore, the definition holds.
\normalsize

\section{Robust properties on ABC with our method}
\label{sec:details_of_robustproperty}
We investigate the behavior of the \emph{sensitivity curve} (SC), which is an empirical analogue of \emph{influence function} (IF) both of which are used in quantifying the robustness of statistics.
The analysis corresponds to a finite-sample analogue of what is called \emph{redescending property} \citep{MaronnaRobust2019} in the context of influence function analysis.
Note that we refer to the redescending property in the asymptotic sense, where some authors use the term \emph{redescending} only when there exists a finite threshold \(\rho > 0\) such that the influence function \(\psi\) satisfies \(\forall |x| > \rho, \psi(x) = 0\) \citep{HampelRobust2005}.
\subsection{Notation}
\label{sec:orgd2fde74}
Let \(\Re, \Na\), and \(\Rgeqzero\) denote the set of real numbers, positive integers, and non-negative real numbers, respectively.
Let \(\Ind{\cdot}\) denote the indicator function.
For \(m \in \Na\), define \([m] \coloneqq \{1, \ldots, m\}\).

We fix \(X^{n} \coloneqq (X_1, \ldots, X_n)\).
For \(\Ym = (Y_1, \ldots, Y_m) \in \ReDm\), define \(\colinftynrm{\Ym} \coloneqq \max_{j \in [m]} \|\Yj\|\).
Let \(\Theta\) be the parameter space, \(\dGth(\Ym) \coloneqq \prodm p_{\theta}(\Yj)\dYj\), and define
\(\Pth(A) \coloneqq \int \Indicator\{\Ym \in A\} \dGth(\Ym)\)
for (Borel) measurable set \(A \subset \ReDm\).

\begin{definition}[Population pseudo-posterior]
The population pseudo-posterior for \(\DivMark, \epsilon, \pi\) is defined as
\begin{equation*}\begin{split}
\PseudoPosterior(\theta | X^{n}) \coloneqq \frac{\pi(\theta) \Pth(\Div{X^{n}}{\Ym} < \epsilon)}{\int \pi(\thetap) \Pthp(\Div{X^{n}}{\Ym} < \epsilon) \dthp}.
\end{split}\end{equation*}
\end{definition}

For convenience of notation, we define \(X_{[X_{0}]}^{n}\) as \(X_{[X_{0}]}^{n} \coloneqq (X_{0}, X_1, \ldots, X_n)\), i.e., the data \(X^{n}\) combined with the contamination \(X_{0}\).
We consider the behavior of \(\PseudoPosterior\) under a contamination \(X_{0}\), i.e., the properties of \(\PseudoPosterior(\theta | X_{[X_{0}]}^{n})\).

\setcounter{definition}{2}
\begin{definition}[Sensitivity curve {\citep[2.1e]{HampelRobust2005}}]
The sensitivity curve of \(\PseudoPosterior\) is defined as
\begin{equation*}\begin{split}
\SCn(X_{0}) \coloneqq (n+1)\left(\PseudoPosterior(\theta|X_{[X_{0}]}^{n}) - \PseudoPosterior(\theta|X^{n})\right).
\end{split}\end{equation*}
\end{definition}
\subsection{Theorem and Proof}
\label{sec:thm_and_proof_for_sc}
In the following theorem, we will see how \(\SCn\) behaves when the outlier \(\xz\) goes far away from the origin.
\setcounter{theorem}{0}
\begin{theorem}[Sensitivity curve analysis]
Assume \(k < \min\{n, m\}\). Also assume that \(\Fth(\epsilon) \coloneqq \Pth(\Div{X^{n}}{\Ym} < \epsilon)\) is \(\beta\)-Lipschitz continuous for all \(\theta \in \Theta\).
Let $\widehat{D}$ be the $\gamma$-divergence estimator in Eq.~\eqref{eq:default_gamma_est}.
Then we have
\begin{equation*}\begin{split}
\lim_{\|X_{0}\| \rightarrow \infty} \SCn(X_{0}) \leq -\frac{\beta \pi(\theta)}{\Hn(1+\gamma)}\log\left(1 - \frac{1}{n^2}\right)^{n+1},
\end{split}\end{equation*}
where \(\Hn \coloneqq \int \pi(\thetap) \Fthp(\epsilon) \dthp\).
Furthermore, if \(\Hn\) converges to \(\H \neq 0\) for \(n \to \infty\), then the right-hand side expression converges to \(0\).
\end{theorem}
\begin{proof}
For simplicity, define \(\Dnm \coloneqq \Divnm\) and \(\Dznm \coloneqq \Divznm\).
Let us start by considering \(\limz \int \Ind{\Dznm < \epsilon} \dGthY\).
To obtain this limit, observe that we only need to take an arbitrary sequence \(\{\xzj\}_{j=1}^\infty\) satisfying \(\|\xzj\| \to \infty\)
and calculate \(\limj \int \Ind{\Dzjnm < \epsilon} \dGthY\) (see Remark~\ref{remark:diverging-sequence-limit}).
Fix such a sequence \(\{\xzj\}_{i=1}^\infty\).

We first consider the point-wise limit \(\limj \Ind{\Dzjnm < \epsilon}\) for each value of \(\Ym\) because we later interchange the limit and the integration by applying the bounded convergence theorem \citep[11.32]{RudinPrinciples1976}:
\(\limj \int \Ind{\Dzjnm < \epsilon} \dGthY = \int \limj \Ind{\Dzjnm < \epsilon} \dGthY\) using the boundedness of \(|\Ind{\Dzjnm < \epsilon}|\) (bounded by \(1\)) and the finiteness of the measure \(\dGth(\Ym)\).

Fix \(\Ym\).
Since \(\{\xzj\}_{i=1}^\infty\) is diverging, if \(j\) is large enough,
\(\xzj\) is never within the \(k\)-nearest neighbors of any of the points in \(\xn\) or \(\Ym\) (here, we used the assumption \(k < n, m\)),
hence \(\rhokdi\) and \(\nukdi\) (\(i = 1, \ldots, n\)) do not depend on \(\xzj\) if \(j\) is large enough.
Let \(\Aone \coloneqq \sumin\left(\frac{1}{\rhokdi}\right)^\gamma\) and \(\Atwo \coloneqq \sumin\left(\frac{1}{\nukdi}\right)^\gamma\),
and by abuse of notation, substitute \(\xz \coloneqq \xzj\) so as to enable using the convenient notation \(\rhokz\) and \(\nukz\).
We can rewrite the event \(\{\Dzjnm < \epsilon\}\) in terms of \(\Dnm\)
based on the following calculation:
\begin{equation*}\begin{aligned}
&\gamma (1+\gamma) \left(\Dzjnm - \Dnm\right) \\
&= \left\{\log\left(\frac{k}{((n + 1) - 1) \cbar}\right)^\gamma \frac{1}{n+1}\sumizn \left(\frac{1}{\rhokdi}\right)^\gamma - \log\left(\frac{k}{(n - 1) \cbar}\right)^\gamma \frac{1}{n}\sumin \left(\frac{1}{\rhokdi}\right)^\gamma\right\} \\
&\quad - (1+\gamma)\left\{\log\left(\frac{k}{m\cbar}^\gamma \frac{1}{n+1}\sumizn\left(\frac{1}{\nukdi}\right)^\gamma\right) - \log\left(\frac{k}{m\cbar}^\gamma \frac{1}{n}\sumin\left(\frac{1}{\nukdi}\right)^\gamma\right)\right\} \\
&= \log\left(\frac{n-1}{n}\right)^\gamma \left(\frac{1}{n+1} \rhokzinvdg + \frac{1}{n+1}\Aone\right)\left(\frac{1}{n}\Aone\right)^{-1} - (1+\gamma)\log\left(\frac{1}{n+1}\nukzinvdg + \frac{1}{n+1}\Atwo\right)\left(\frac{1}{n}\Atwo\right)^{-1}\\
&= \left\{\gamma \log \frac{n-1}{n} + \log \frac{n}{n+1} - (1+\gamma) \log \frac{n}{n+1}\right\} + \left\{\log \left(\Aone^{-1}\rhokzinvdg + 1\right) - (1+\gamma) \log\left(\Atwo^{-1}\nukzinvdg + 1\right)\right\} \\
&= \gamma \log(1 - \frac{1}{n^2}) + \left\{\log \left(\Aone^{-1}\rhokzinvdg + 1\right) - (1+\gamma) \log\left(\Atwo^{-1}\nukzinvdg + 1\right)\right\}.
\end{aligned}\end{equation*}
Therefore, \(\Dzjnm < \epsilon \Leftrightarrow \Dnm < \tepsilon + \phi(\xzj)\) if \(j\) is large enough,
where
\begin{equation*}\begin{aligned}
\tepsilon \coloneqq \epsilon - \frac{1}{1+\gamma} \log(1 - \frac{1}{n^2}),
\qquad\phi(\xzj) \coloneqq \log \left(\Aone^{-1}\rhokzinvdg + 1\right) - (1+\gamma) \log\left(\Atwo^{-1}\nukzinvdg + 1\right)
\end{aligned}\end{equation*}
and \(\rhoki, \nuki\) are based on the temporary notation \(\xz = \xzj\).
In terms of indicator functions, we have just shown that
\begin{equation}\label{eq:sensitivity-proof:1}\begin{aligned}
\Ind{\Dzjnm < \epsilon} = \Ind{\Dnm < \tepsilon + \phi(\xzj)}
\end{aligned}\end{equation}
holds if \(j\) is large enough. We have \(\limj \phi(\xzj) = 0\) as well.

Now we show that, for each fixed distinct set of points \((Y_2, \ldots, Y_m)\), we have \(\limj \Ind{\Dnm < \tepsilon + \phi(\xzj)} = \Ind{\Dnm < \tepsilon}\) for almost all \(\Yone\).
Fix distinct points \(Y_2, \ldots, Y_m\).
First, we can show that
\begin{equation}\label{eq:sensitivity-proof:3}\begin{aligned}
\Ind{\Dnm < \tepsilon} \ \leq\  \Ind{\Dnm < \tepsilon + \phi(\xzj)} \ \leq\  \Ind{\Dnm < \tepsilon} + \left(\Ind{\Dnm = \tepsilon} - \Ind{\Dnm = \tepsilon + \phi(\xzj)}\right)
\end{aligned}\end{equation}
holds if \(j\) is large enough.
To see the first inequality, observe the following: if \(\Yone\) is such that \(\Dnm < \tepsilon\), there exists \(J\) such that for all \(j > J\) it holds that \(|\phi(\xzj)| < \tepsilon - \Dnm\), and hence \(\Dnm < \tepsilon - |\phi(\xzj)| \leq \tepsilon + \phi(\xzj)\).
Therefore, if \(j\) is large enough, \(\Ind{\Dnm < \tepsilon} \leq \Ind{\Dnm < \tepsilon + \phi(\xzj)}\) as functions of \(\Yone\).
The second inequality can be shown by similarly obtaining \(\Ind{\Dnm > \tepsilon} \leq \Ind{\Dnm > \tepsilon + \phi(\xzj)}\) for large enough \(j\) and rearranging the terms.
By Equation~\eqref{eq:sensitivity-proof:3}, defining \(\Nullset \coloneqq \{\Yone: \Dnm = \tepsilon\} \cup \left(\bigcup_j \{\Yone: \Dnm = \tepsilon + \phi(\xzj)\}\right)\),
we have \(\Ind{\Dnm < \tepsilon + \phi(\xzj)} = \Ind{\Dnm < \tepsilon}\) if \(j\) is large enough, for each \(\Yone \not \in \Nullset\).
On the other hand, by Proposition~\ref{sensitivity-proof:prop:measure-zero},
each of \((\Dnm)^{-1}(\{\tepsilon\})\) and \((\Dnm)^{-1}(\{\tepsilon + \phi(\xzj)\})\) has zero Lebesgue measure,
hence their countable union \(\Nullset\) also has zero Lebesgue measure.
As a result,
\begin{equation}\label{eq:sensitivity-proof:2}\begin{aligned}
\limj \Ind{\Dnm < \tepsilon + \phi(\xzj)} = \Ind{\Dnm < \tepsilon} \quad \text{a.e.}\ \Yone
\end{aligned}\end{equation}
holds for all \((Y_2, \ldots, Y_m)\).

Now, apply the bounded convergence theorem \citep[11.32]{RudinPrinciples1976}, the Fubini-Tonelli theorem \citep[Theorem~18.3]{BillingsleyProbability1995},
and Equation~\eqref{eq:sensitivity-proof:2} to obtain
\begin{equation*}\begin{aligned}
&\limj \Pth(\Dzjnm < \epsilon)
= \limj \int \Ind{\Dzjnm < \epsilon} \dGthY
= \int \limj \Ind{\Dzjnm < \epsilon} \dGthY \\
&= \int \limj \Ind{\Dnm < \tepsilon + \phi(\xzj)} \dGthY
= \int \left(\int \limj \Ind{\Dnm < \tepsilon + \phi(\xzj)} \dGthYone\right) \prod_{j=2}^m \dGthYj \\
&= \int \left(\int\Ind{\Dnm < \tepsilon} \dGthYone\right) \prod_{j=2}^m \dGthYj
= \int \Ind{\Dnm < \tepsilon} \dGthY
= \Pth(\Dnm < \tepsilon),
\end{aligned}\end{equation*}
where we also took into account that the points \(Y_2, \ldots, Y_m\) are almost surely distinct.
Since the choice of \(\{\xzj\}_{i=1}^\infty\) was arbitrary, the above calculation implies
\begin{equation*}\begin{aligned}
\limz \Pth(\Dznm < \epsilon) = \Pth(\Dnm < \tepsilon).
\end{aligned}\end{equation*}
Therefore, defining \(\alphafn{\theta}{\epsilon} \coloneqq \pi(\theta)\Pth(\Dnm < \epsilon)\),
\begin{equation*}\begin{aligned}
&\limz \PseudoPosterior(\theta | \xzn) = \limz \frac{\pi(\theta) \Pth(\Dznm < \epsilon)}{\int \pi(\thetap) \Pthp(\Dznm < \epsilon) \dthp} \\
&= \left(\limz \pi(\theta) \Pth(\Dznm < \epsilon)\right)\left(\limz \int \pi(\thetap) \Pthp(\Dznm < \epsilon) \dthp\right)^{-1} \\
&= \left(\alphafn{\theta}{\tepsilon}\right)\left(\int \alphafn{\thetap}{\tepsilon}\dthp\right)^{-1} \\
\end{aligned}\end{equation*}
where we applied the bounded convergence theorem \citep[11.32]{RudinPrinciples1976} to the integration in the denominator as \(\Pth \leq 1\).
As a result, denoting \(\Deltafn{\theta}{\tepsilon}{\epsilon} \coloneqq \alphafn{\theta}{\tepsilon} - \alphafn{\theta}{\epsilon}\)
and noting that \(\tepsilon \geq \epsilon\) hence \(\Deltafn{\theta}{\tepsilon}{\epsilon} \geq 0\),
\begin{equation*}\begin{aligned}
&\limz\SCn(\xz)
= (n+1)\left(\limz \PseudoPosterior(\theta|\xzn) - \PseudoPosterior(\theta|\xn)\right) \\
&= (n+1)\left(\frac{\alphafn{\theta}{\tepsilon}}{\int \alphafn{\thetap}{\tepsilon} \dthp} - \frac{\alphafn{\theta}{\epsilon}}{\int \alphafn{\thetap}{\epsilon}\dthp}\right)
= (n+1)\frac{\Hn(\alphafnthep + \Deltafnthep) - \alphafnthep \left(\Hn + \int \Deltafn{\thetap}{\tepsilon}{\epsilon} \dthp\right)}{\left(\Hn + \int \Deltafn{\thetap}{\tepsilon}{\epsilon} \dthp\right) \Hn}\\
&= (n+1) \frac{\Hn \Deltafn{\theta}{\tepsilon}{\epsilon} - \alphafn{\theta}{\tepsilon}\int \Deltafn{\thetap}{\tepsilon}{\epsilon} \dthp}{\left(\Hn + \int \Deltafn{\thetap}{\tepsilon}{\epsilon} \dthp\right) \Hn}
\leq (n+1) \frac{\Hn \Deltafnthep}{\left(\Hn + \int \Deltafn{\thetap}{\tepsilon}{\epsilon} \dthp\right) \Hn} \\
&\leq (n+1)\frac{\Hn \Deltafnthep}{\Hn^2}
= \frac{1}{\Hn} (n+1)(\alphafn{\theta}{\tepsilon} - \alphafn{\theta}{\epsilon}).
\end{aligned}\end{equation*}
Finally, applying \(\alphafn{\theta}{\tepsilon} - \alphafn{\theta}{\epsilon} \leq \beta \pi(\theta) (\tepsilon - \epsilon)\),
we obtain
\begin{equation*}\begin{aligned}
\limz\SCn(\xz) &\leq - \frac{\beta \pi(\theta)}{\Hn(1+\gamma)} \log\left(1 - \frac{1}{n^2}\right)^{n+1}
\end{aligned}\end{equation*}
as desired.

If \(\Hn\) converges to a nonzero value \(H\), we have
\begin{equation*}\begin{aligned}
\lim_{n \to \infty} - \frac{\beta \pi(\theta)}{\Hn(1+\gamma)} \log\left(1 - \frac{1}{n^2}\right)^{n+1}
&= - \frac{\beta \pi(\theta)}{\H(1+\gamma)} \left(\lim_{n \to \infty} \log\left(1 - \frac{1}{n^2}\right) \left(\left(1 - \frac{1}{n^2}\right)^{n^2}\right)^\frac{1}{n}\right) \\
&= - \frac{\beta \pi(\theta)}{\H(1+\gamma)} \log \left((1 - 0) \left(\frac{1}{e}\right)^0\right)
= 0.
\end{aligned}\end{equation*}
\end{proof}
The following Proposition~\ref{sensitivity-proof:prop:measure-zero} is used in the proof of Theorem~\ref{thm:redescending-sensitivity-curve}.
Proposition~\ref{sensitivity-proof:prop:measure-zero} reflects the smoothness of \(\Dnm\) to show that the transformation \(\Yone \mapsto \Dnm\) results in a continuous random variable.
\begin{prop}[\(\{\Dnm = c\}\) has zero measure.]
Fix distinct points \((Y_2, \ldots, Y_m)\) and define \(f(\Yone) \coloneqq \Dnm\).
Then, for any \(c \in \Re\), the set \(f^{-1}(\{c\})\) has Lebesgue measure zero.
\label{sensitivity-proof:prop:measure-zero}
\end{prop}
\begin{proof}
We start by observing that the space of \(\Yone\), namely \(\ReD\), can be split into a finite family of disjoint open sets \(\{\Ul\}_{l=1}^L\) such that
\(\Uc \coloneqq \ReD \setminus \left(\bigcup_{l=1}^L \Ul\right)\) has measure zero
and that for all \(\Yone \in \Ul\), the \(k\)-NN (more precisely, the index of the \(k\)-NN point) of \(X_i\) (\(i \in [n]\)) among \(\{Y_j\}_{j=1}^m\) and that of \(Y_j\) among \(\{Y_{j'}\}_{j' \neq j}\) are identical.
Such a partition makes the problem easier because within each partition cell, \(\Ul\), the \(k\)-NN distances \(\nuki\) and \(\mukj\) take the simple form as mere Euclidean distances between two predetermined points.

Such \(\{\Ul\}_{l=1}^L\) can be constructed as follows.
Define \(A_{ji} = \|Y_j - X_i\|\) and \(B_{jj'} = \|Y_j - Y_{j'}\|\) and consider the distance matrices
\begin{equation*}\begin{aligned}
A =
\begin{pmatrix}
  A_{11} & \cdots & \cdots & A_{1n} \\
  \vdots & A_{22} & \cdots & A_{2n} \\
  \vdots & \vdots & \ddots & \vdots \\
  A_{m1} & A_{m2} & \cdots & A_{mn} \\
\end{pmatrix},
B =
\begin{pmatrix}
B_{11} & \cdots & \cdots & B_{1m} \\
\vdots & B_{22} & \cdots & B_{2m} \\
\vdots & \vdots & \ddots & \vdots \\
B_{m1} & B_{m2} & \cdots & B_{mm} \\
\end{pmatrix}.
\end{aligned}\end{equation*}
For each point \(X_i\) or \(Y_{j'}\), the corresponding \(k\)-NN points are determined by the order of the elements in the corresponding columns \(A_{\cdot, i}\) and \(B_{\cdot, j'}\).
In \(A\), the only variables with respect to \(\Yone\) are the first row. Similarly, the variables in \(B\) are the first row and the first column.
In other words, the bottom-right blocks obtained by removing the first rows and first columns are constant with respect to \(\Yone\).

Let us first consider \(A\).
The \(k\)-NN points for each \(X_i\) can be determined by finding where \(A_{1i}\) is ranked among the ranking of column \(i\).
Since the elements of column \(i\) except the first element, \((A_{2i}, \ldots, A_{mi})\), is constant with respect to \(\Yone\), they can be sorted as \((A_{(2),i}, \ldots, A_{(m),i})\) in ascending order to define a partitioning of \(\ReD\) in each of which \(A_{1i}\) has the same ranking among the elements in the column \(i\): \(V_{j}^{i} = \{\Yone \in \ReD: \|\Yone - X_i\| \in (A_{(j),i}, A_{(j+1),i})\}\) (\(j \in [m]\)), where \(A_{(1),i} = 0\) and \(A_{(m+1), i} = \infty\).
By taking the intersections of such partitions,
\(V_{(j_1, \ldots, j_n)} = V_{j_1}^{1} \cap \cdots \cap V_{j_n}^{n}\),
we obtain a family of disjoint open sets \(\mathcal{V} \coloneqq \{V_{(j_1, \ldots, j_n)}\}_{(j_1, \ldots, j_n) \in [m]^n}\) that covers almost everywhere \(\ReD\) because each \(V^{(i)\mathrm{c}} \coloneqq \ReD \setminus \bigcup_{j \in [m]} V_j^i\) has Lebesgue measure zero and
\begin{equation*}\begin{aligned}
\ReD = \bigcap_i \ReD = \bigcap_i \left(V^{(i)\mathrm{c}} \cup \bigcup_{j_i} V_{j_i}^i\right) = V^\mathrm{c} \cup \bigcap_i \bigcup_{j_i} V_{j_i}^i
\end{aligned}\end{equation*}
where \(V^\mathrm{c}\) is a set with less Lebesgue measure than the sum of the measures of \(V^{(i)\mathrm{c}}\) hence has zero measure.

Similarly, let us consider \(B\).
The second-to-last columns of \(B\) can be treated in the same way as \(A\) to obtain the almost-everywhere finite partition \(\mathcal{W}_j\) of \(\ReD\) for each column \(j = 2, \ldots, m\) in which the ranking of \(B_{1j}\) remains invariant for each column (note that, although the diagonal elements of \(B\) are not used for determining the \(k\)-NN points, their existence does not affect the above construction).
Now we consider the first column and construct an almost-everywhere partition of \(\ReD\) in each of which the ordering of \(\|Y_2 - Y_1\|, \ldots, \|Y_m - Y_1\|\) does not change.
The existence of such a finite partition is guaranteed by the existence of \(l\)-th degree Voronoi diagrams for \(l = 1, \ldots, m-1\) \cite{AurenhammerVoronoi2000,EdelsbrunnerConstructing1986}.
In \(l\)-th degree Voronoi diagram \(\{W^{(l)}_a\}_{a}\), each cell \(W^{(l)}_a\) represents a region in which \(\Yone\) has the same set of points as the \(l\)-nearest neighbors.
Therefore, by taking the intersections \(W_{(a_1, \ldots, a_{m-1})} = W^{(1)}_{a_1} \cap \cdots \cap W^{(m-1)}_{a_{m-1}}\), we obtain regions in each of which the ordering of the distances \(\|Y_2 - Y_1\|, \ldots, \|Y_m - Y_1\|\) remain the same.
There are only finite regions in the \(l\)-th degree Voronoi diagram for all \(l = 1, \ldots, m\), hence the family of their intersections are also finite, and the boundaries of Voronoi cells have zero Lebesgue measure as they correspond to the sets where two of the sites are at an equal distance.
Therefore, we have obtained the desired partition which we denote by \(\mathcal{W}_1\).

By taking all intersections of the above partitions, \(\mathcal{V}, \{\mathcal{W}_j\}_{j=1}^m\), we obtain the desired finite partition \(\mathcal{U} = \{\Ul\}_{l=1}^L\) that covers almost everywhere \(\ReD\) and in each \(\Ul\), the indices of the \(k\)-NN points remain the same.

Let us define \(\fl \coloneqq f\restrict{\Ul}\). Now, we show that each \(f^{-1}(\{c\}) \cap \Ul\) has zero measure.
In each \(\Ul\), the distances \(\nuki\) and \(\mukj\) are strictly positive as no two points overlap.
Therefore, \(\fl : \Ul \to \Re\) is a real analytic function since it is a composition of analytic functions:
\begin{equation*}\begin{aligned}
&\fl(Y_1) = - \frac{1}{\gamma} \log \left(\sumin \nukdigammainv\right) + \frac{1}{1+\gamma} \log \left(\sumjm \mukdjgammainv\right) + \const \\
&= -\frac{1}{\gamma} \log \left(\sumin \exp\left(-\gamma d \frac{1}{2}\log (\nuki)^2\right)\right) + \frac{1}{1+\gamma} \log\left(\sumjm \exp\left(- \gamma d \frac{1}{2} \log (\mukj)^2\right)\right) + \const,
\end{aligned}\end{equation*}
and \((\nuki)^2, (\mukj)^2\) are either quadratic forms of \(\Yone\) or constants.
As a result, \(\fl^{-1}(\{c\}) = f^{-1}(\{c\}) \cap \Ul\) is a zero set of a real analytic function \(\fl - c\) that is not a constant function, hence has zero Lebesgue measure \citep[Lemma~1.2]{DangComplex2015}, \citep{MityaginZero2015}.

Finally, the assertion of the proposition follows immediately from
\begin{equation*}\begin{aligned}
\lambda(f^{-1}(\{c\})) = \lambda\left(f^{-1}(\{c\}) \cap \left(\Uc \cup \bigcup_{l=1}^L \Ul\right)\right)
\leq \lambda\left(f^{-1}(\{c\}) \cap \Uc\right) + \sum_{l=1}^L \lambda\left(f^{-1}(\{c\}) \cap \Ul\right),
\end{aligned}\end{equation*}
where we denoted the Lebesgue measure by \(\lambda\).
\end{proof}
\subsection{Remarks}
\label{sec:remarks}
\begin{remark}[Relation to redescending property of influence functions]
It should be noted that the above theorem is a finite-sample analogue of the redescending property of influence functions.
In the case of influence functions, redescending property is defined as convergence to zero under \(\|\xz\| \to \infty\) \cite{MaronnaRobust2019}.
The discrepancy that the limit in our case is nonzero (only converges to zero with \(n \to \infty\)) stems from the fact that we consider the finite sample analogue, namely, the sensitivity curve.
This is intuitively comprehensible since the influence function reflects the response to contamination in the underlying distribution, i.e., a population quantity.
\end{remark}

\begin{remark}[The reason to consider sensitivity curve instead of influence functions.]
\label{remark2}
The reason we consider SC instead of IF is two-fold: (1) we are interested in the pseudo-posterior distribution \(\hat \pi(\theta|\xn)\) with respect to a finite sample \(\xn\), hence the SC can more precisely provide the information of our interest, and (2) the IF of the quantities based on the considered divergence estimator may not even exist.
The definition of the considered divergence estimator is based on \(k\)-NN density estimators, and it does not have a straightforward representation as a statistical functional (i.e., a functional of the underlying data distribution). Furthermore, even if we consider the divergence estimator as a functional of the underlying probability density function of the data,
the \(k\)-NN density estimator is not square-integrable if \(k=1\) \citep[Proposition~3.1]{Biauknearest2015},
hence the standard definition of influence functions as a dual point in the Hilbert space \(L^2\) is not applicable.
Therefore, we consider the sensitivity curve for the theoretical analysis, which can directly reflect the detailed procedure to construct the estimate from given data points.
\end{remark}

\begin{remark}[Diverging limit and diverging sequence limit]
In the proof, we used the fact that if \(\lim_{i \to \infty} f(\xzdummyj) = L\) for any diverging sequence \(\{\xzdummyj\}_{i=1}^\infty\) (i.e., \(\|\xzdummyj\| \to \infty\)), we have \(\limz f(\xz) = L\).
We show a proof by contradiction.
First recall that \(\limz f(\xz) = L\) means that for any \(\epsilon > 0\), there exists \(B > 0\) such that for any \(\xz\) satisfying \(\|\xz\| > B\) it holds that \(|f(\xz) - L| < \epsilon\).
To show this by contradiction, assume that there exists \(\epsilon > 0\) such that for any \(B > 0\) there exists \(\xz\) satisfying \(\|\xz\| > B\) and \(|f(\xz) - L| \geq \epsilon\).
Now fix such an \(\epsilon\) and define \(B_i \coloneqq 2^i\) for \(i \in \Na\). By assumption, there exist a sequence \(\{x_i\}_{i=1}^\infty\) such that \(\|x_i\| > B_i\) and \(|f(x_i) - L| \geq \epsilon\).
Because \(\{x_i\}_{i=1}^\infty\) is a diverging sequence, it has to hold that \(\lim_{i \to \infty} f(x_i) = L\). This is a contradiction.
\label{remark:diverging-sequence-limit}
\end{remark}

\begin{remark}[Exchanging the limits]
The current statement of the theorem takes the limit of \(\lim_{\|\xz\| \to \infty}\) for each fixed \(n\).
One should note that \(\Aone\) and \(\Atwo\) in the proof depend on \(n\) and the sample \(\xn\). Similarly, \(\rhokzinvdg\) and \(\nukzinvdg\) depend on the sample.
Therefore, care should be taken if one wants to merge the two limit operations \(\lim_{n \to \infty}\) and \(\lim_{\|\xz\| \to \infty}\).
\end{remark}

\section{Preliminaries for Asymptotic Analysis}
\label{app:asymptotic_analysis}
In this section, we summarize several specific lemmas and theorems to show the asymptotic properties of the proposed discrepancy in Eq.~\eqref{eq:default_gamma_est}.
Here, we denote $\rightarrow_{w}$, $\rightarrow_{d}$ and $\rightarrow_{p}$ as the \emph{weak convergence} of distribution functions, the convergence of random variables \emph{in distribution} and the convergence of random variables \emph{in probability}, respectively.

Remembering the fact that $\rho_{k}(i)$ is a random variable, which is the measure of discrepancy between $X_{i}$ and its $k$-th nearest neighbor in $X^{n} \backslash X_{i}$, the following lemmas and theorems hold.

\begin{lem}
Let $\zeta_{n,k,1} \coloneqq \log (n-1) \rho_{k}^{d}(1)$ be a random variable, and let $F_{n,k,x}(u) \coloneqq \mathrm{Pr}(\zeta_{n,k,1} < u | X_{1} = x)$ denotes its conditional distribution function.
Then,
\begin{align*}
    F_{n,k,x}(u) = 1 - \sum_{j=0}^{k-1} \left(
    \begin{array}{cc}
      n-1 \\
      j
    \end{array}
  \right)(P_{n,u,x})^{j} (1 - P_{n,u,x})^{n-1-j},
\end{align*}
where $P_{n,u,x} \coloneqq \int_{\mathcal{M}\cap \mathcal{B}(x,R_{n}(u))} p(t)\mathrm{d}t$ and $R_{n}(u) \coloneqq (e^{u}/(n-1))^{1/d}$.
\begin{proof}
We can obtain
\begin{align*}
    &F_{n,k,x}(u) = \mathrm{Pr}(\zeta_{n,k,1} < u | X_{1} = x) \\
    &= \mathrm{Pr}(\log (n-1) \rho_{k}^{d}(1) < u | X_{1} = x)
    = \mathrm{Pr}\bigg(\rho_{k}(1) < \bigg( \frac{e^{u}}{n-1} \bigg)^{1/d}| X_{1} = x \bigg) \\
    &= \mathrm{Pr}\bigg(\rho_{k}(1) < R_{n}(u)| X_{1} = x \bigg) \ \ (\text{because } R_{n}(u) \coloneqq (e^{u}/(n-1))^{1/d}).
\end{align*}
The last expression can be interpreted as the probability of $k$ or more elements from $ \{X_{2} \ldots X_{n}\}$ being contained in $\mathcal{M} \cap \mathcal{B}(x,R_{n}(u))$ given $X_{1}=x$.
Since we have \IID observations, this condition can be ignored.
Therefore, we can see this probability as binomial distribution and obtain
\begin{align*}
    F_{n,k,x}(u) &= \mathrm{Pr}\bigg(\rho_{k}(1) < R_{n}(u)| X_{1} = x \bigg) \\
    &= \sum_{j=k}^{n-1} \left(
    \begin{array}{cc}
      n-1 \\
      j
    \end{array}
    \right)(P_{n,u,x})^{j} (1 - P_{n,u,x})^{n-1-j}\\
    &= 1 - \sum_{j=0}^{k-1} \left(
    \begin{array}{cc}
      n-1 \\
      j
    \end{array}
  \right)(P_{n,u,x})^{j} (1 - P_{n,u,x})^{n-1-j},
\end{align*}
and the claim holds.
\end{proof}
\end{lem}

\begin{lem}[Log-Erlang distribution]
Let $u$ be a random variable from the Erlang distribution as
\begin{align*}
    f_{x,k}(u) = \frac{1}{\Gamma(k)}\lambda(x)^{k}u^{k-1}\exp(-\lambda(x) u),
\end{align*}
where $\lambda(x) > 0$ and $k \in \mathbb{Z}^{+}$.
Here, $\mathbb{Z}^{+}$ denotes the set of positive integer.
Then, $l=\log u$ is a random variable from the log-Erlang distribution as
\begin{align*}
    g_{n,k}(l) = \frac{1}{\Gamma(k)} \lambda(x)^{k} (\exp(l))^{k} \exp(-\lambda(x) \exp(l)).
\end{align*}
\begin{proof}
If we set $l=\log u$, we obtain $u = \exp(l)$ and $\frac{\mathrm{d}l}{\mathrm{d}u} = \frac{1}{u} = \frac{1}{\exp(l)}$.
When we denote the distribution of $l$ as $g_{n,k}(l)$,
\begin{align*}
    g_{n,k}(l) &= f_{n,k}(u) \bigg| \frac{\mathrm{d}u}{\mathrm{d}l} \bigg|
    = \frac{1}{\Gamma(k)}\lambda(x)^{k}u^{k-1}\exp(-\lambda(x) u) \cdot \exp(l) \\
    &= \frac{1}{\Gamma(k)}\lambda(x)^{k} (\exp(l))^{k-1} \exp(-\lambda(x) \exp(l)) \cdot \exp(l)
    = \frac{1}{\Gamma(k)} \lambda(x)^{k}(\exp(l))^{k} \exp(-\lambda(x) \exp(l)).
\end{align*}
This is the same as the definition of the log-Gamma distribution.
Because of $k \in \mathbb{Z}^{+}$, we can see that $g_{n,k}(l)$ is the log-Erlang distribution.

The claim is proved.
\end{proof}
\end{lem}

\begin{lem}[Expectation of log-Erlang distribution]
\label{lem:exp_log_erlang}
Let $f_{x,k}(u) \coloneqq \frac{1}{\Gamma(k)} \lambda(x)^{k}(\exp(l))^{k} \exp(-\lambda(x) \exp(l))$ be the density of the log-Erlang distribution with parameters $\lambda(x) > 0$ and $k \in \mathbb{Z}^{+}$.
Then, the $1$-th moments of the log-Erlang distribution can be calculated as
\begin{align*}
    \int_{0}^{\infty} u f_{x,k}(u) \mathrm{d}u = \psi(k) - \log(\lambda(x)),
\end{align*}
where $\psi(\cdot)$ is a digamma function.
\begin{proof}
Because the function $f_{x,k}(u)$ is the density of the log-Erlang distribution, we obtain
\begin{align*}
    \int_{\mathbb{R}} (\exp(u))^{k} \exp(-\lambda(x)\exp(u)) \mathrm{d}u &= \int_{\mathbb{R}} \exp(ku - \lambda(x)\exp(u))\mathrm{d}u
    = \Gamma(k)\lambda(x)^{-k}.
\end{align*}
Differentiating the inside of the above integration by $k$, we obtain
\begin{align*}
    \frac{\mathrm{d}}{\mathrm{d}k}\exp(ku - \lambda(x)\exp(u)) &=
    u \exp(ku - \lambda(x)\exp(u))
    = u \cdot \Gamma(k)\lambda(x)^{-k} f_{x,k}(u).
\end{align*}
Therefore, the expectation of $u$ is written as
\begin{align*}
    &\mathbb{E}[u] = \int_{0}^{\infty} u f_{x,k}(u) \mathrm{d}u
    = \int_{0}^{\infty} u \frac{1}{\Gamma(k)} \lambda(x)^{k}\exp(ku - \lambda(x)\exp(u)) \mathrm{d}u \\
    &= \frac{\lambda(x)^{k}}{\Gamma(k)} \int_{0}^{\infty} u \exp(ku - \lambda(x)\exp(u)) \mathrm{d}u
    = \frac{\lambda(x)^{k}}{\Gamma(k)} \int_{0}^{\infty}\frac{\mathrm{d}}{\mathrm{d}k}\exp(ku - \lambda(x)\exp(u)) \mathrm{d}u \\
    &= \frac{\lambda(x)^{k}}{\Gamma(k)} \frac{\mathrm{d}}{\mathrm{d}k} \int_{0}^{\infty}\exp(ku - \lambda(x)\exp(u)) \mathrm{d}u
    = \frac{\lambda(x)^{k}}{\Gamma(k)}\frac{\mathrm{d}}{\mathrm{d}k} \Gamma(k)\lambda(x)^{-k} \\
    &= \frac{\lambda(x)^{k}}{\Gamma(k)} \bigg(
    \frac{\mathrm{d}}{\mathrm{d}k} \Gamma(k) \cdot \lambda(x)^{-k} - \Gamma(k) \cdot \lambda(x)^{-k} \log (\lambda(x)) \bigg) \\
    &= \frac{1}{\Gamma(k)}\frac{\mathrm{d}}{\mathrm{d}k} \Gamma(k) - \log (\lambda(x))
    = \psi(k) - \log (\lambda(x)).
\end{align*}
The claim is hold.
\end{proof}
\end{lem}

We show the following properties on the log-Erlang distribution according to standard proof techniques in \cite{Leonenko08a}.
\begin{lem}
\label{lem:weakly_convergence}
Suppose that Lebesgue-approximable function on $p$ in Assumptions \ref{assmp:p_bounded} and \ref{assmp:p_bounded_cdf} holds.
Let $u$ be fixed. Then, $F_{n,k,x}(u) \rightarrow_{w} F_{k,x}(u)$ for almost all $x \in \mathcal{M}$,
where
\begin{align*}
    F_{k,x}(u) \coloneqq 1 - \exp(-\lambda(x) \exp(u)) \sum_{j=0}^{k-1} \frac{1}{j!} (\lambda(x)\exp(u))^{j}
\end{align*}
is the log-Erlang distribution with $\lambda(x) = \bar{c}p(x)$.
\begin{proof}
According to Assumptions \ref{assmp:p_bounded} and \ref{assmp:p_bounded_cdf}, we can see that for all $\delta > 0$ and almost all $x \in \mathcal{M}$ there exists $n_{0}(x,\delta,u) \in \mathbb{Z}_{+}$ such that if $n > n_{0}(x,\delta,u)$, then $\mathcal{B}(x,R_{n}) = \mathcal{B}(x,R_{n}) \cap \mathcal{M}$, and
\begin{align*}
    p(x) - \delta < \frac{\int_{\mathcal{B}(x,R_{n})\cap\mathcal{M}}p(t)\mathrm{d}t}{\frac{\exp(u)\bar{c}}{n-1}} < p(x) + \delta \ \ \ \bigg( \mathcal{V}(\mathcal{B}(x,R_{n})\cap\mathcal{M}) = \frac{\exp(u)\bar{c}}{n-1} \bigg).
\end{align*}
Therefore, if $n > n_{0}(x,\delta,u)$,
\begin{align*}
    F_{n,k,u}(u) &= 1 - \sum_{j=0}^{k-1} \left(
    \begin{array}{cc}
      n-1 \\
      j
    \end{array}
  \right)(P_{n,u,x})^{j} (1 - P_{n,u,x})^{n-1-j} \\
  &= 1 - \sum_{j=0}^{k-1} \left(
    \begin{array}{cc}
      n-1 \\
      j
    \end{array}
  \right) \bigg(\int_{\mathcal{B}(x,R_{n})\cap\mathcal{M}}p(t)\mathrm{d}t \bigg)^{j} \bigg( 1 - \int_{\mathcal{B}(x,R_{n})\cap\mathcal{M}}p(t)\mathrm{d}t \bigg)^{n-1-j} \\
  &\geq 1 - \sum_{j=0}^{k-1} \left(
    \begin{array}{cc}
      n-1 \\
      j
    \end{array}
  \right) \bigg( \frac{\exp(u)}{n-1} \bar{c} (p(x) + \delta) \bigg)^{j} \bigg( 1 - \frac{\exp(u)}{n-1} \bar{c} (p(x) - \delta) \bigg)^{n-1-j} \\
  &= 1 - \sum_{j=0}^{k-1} \frac{(n-1)!}{j!(n-1-j)!} \bigg( \frac{\exp(u)}{n-1} \bar{c} (p(x) + \delta) \bigg)^{j} \bigg( 1 - \frac{\exp(u)}{n-1} \bar{c} (p(x) - \delta) \bigg)^{n-1-j} \\
  &= 1 - \sum_{j=0}^{k-1} \frac{1}{j!} \frac{(n-1)!}{(n-1-j)!(n-1)^{j}} \bigg(\exp(u) \bar{c} (p(x) + \delta) \bigg)^{j} \bigg( 1 - \frac{\exp(u)}{n-1} \bar{c} (p(x) - \delta) \bigg)^{n-1-j}.
\end{align*}
Because of the fact that
\begin{align*}
    &\lim_{n \rightarrow \infty} \frac{(n-1)!}{(n-1-j)!(n-1)^{j}} = 1, \\
    &\lim_{n \rightarrow \infty} \bigg( 1 - \frac{\exp(u)}{n-1} \bar{c} (p(x) - \delta) \bigg)^{n-1-j} = \exp(- \exp(u)\bar{c} (p(x) - \delta)),
\end{align*}
we obtain for all $\delta > 0$ and for almost all $x \in \mathcal{M}$,
\begin{align*}
    \liminf_{n \rightarrow \infty} F_{n,k,u}(u) \geq 1 - \sum_{j=0}^{k-1} \frac{1}{j!} \bigg(\exp(u) \bar{c} (p(x) + \delta) \bigg)^{j}\exp(- \exp(u)\bar{c} (p(x) - \delta)).
\end{align*}
By choosing $\delta \rightarrow 0$, we can see that
\begin{align*}
    \liminf_{n \rightarrow \infty} F_{n,k,u}(u) \geq 1 - \sum_{j=0}^{k-1} \frac{1}{j!} (\exp(u)\lambda(x))^{j}\exp(- \exp(u)\lambda(x)),
\end{align*}
where $\lambda(x) \coloneqq \bar{c}p(x)$.

In the same way, we can show that for almost all $x \in \mathcal{M}$
\begin{align*}
    \limsup_{n \rightarrow \infty} F_{n,k,u}(u) \leq 1 - \sum_{j=0}^{k-1} \frac{1}{j!} (\exp(u)\lambda(x))^{j}\exp(- \exp(u)\lambda(x)).
\end{align*}
When we define $F_{k,x}(u) \coloneqq 1 - \sum_{j=0}^{k-1} \frac{1}{j!} (\exp(u)\lambda(x))^{j}\exp(- \exp(u)\lambda(x))$, the claim is proved.
\end{proof}
\end{lem}

\begin{lem}
\label{lem:conv_in_dist}
Let $\xi_{n,k,x}$ and $\xi_{k,x}$ be random variables with $F_{n,k,x}$ and $F_{k,x}$ distribution functions, and let $\kappa \in \mathbb{R}$ be arbitrary. Then for almost all $x \in \mathcal{M}$ we have that $\xi_{n,k,x}^{\kappa} \rightarrow_{d} \xi_{k,x}^{\kappa}$, where $f_{n} \rightarrow_{d} f$ indicates convergence of random variable $f_{n}$ in distribution.
\begin{proof}
According to Lemma \ref{lem:weakly_convergence}, we obtain $F_{n,k,x}(u) \rightarrow_{w} F_{k,x}(u)$ for almost all of $x \in \mathcal{M}$.
This is equal to the fact that $F_{n,k,x}(u) \rightarrow_{d} F_{k,x}(u)$ for almost all of $x \in \mathcal{M}$.
Since the function of $(\cdot)^{\kappa}$ is continuous on $(0, \infty)$ and $X_{i} \in (0, \infty)$ almost surely, by using the continuous mapping theorem (\cite{vaart_1998}), the claim is proved.
\end{proof}
\end{lem}

For proving Corollary \ref{cor:asymp_quasi_posterior_gamma}, we introduce the L\'evy's Upward Theorem as follow.
\begin{thm}[L\'evy's Upward Theorem]
\label{thm:levy_upward}
Let $\{Z_{n} \}_{n \geq 0}$ be a collection of random variables, and let $\mathcal{F}_{n}$ be a filtration on the same probability space.
If $\sup_{n \geq 0}|Z_{n}|$ is integrable, $Z_{n} \rightarrow Z_{\infty}$ almost surely as $n \rightarrow \infty$ and $\mathcal{F}_{n} \uparrow F_{\infty}$, then $\mathbb{E}[Z_{n}|\mathcal{F}_{n}] \rightarrow \mathbb{E}[Z_{\infty}|\mathcal{F}_{\infty}]$ both almost surely and in mean.
\end{thm}

To show Theorem \ref{thm:asymp_unbiased_others_p}, we analyze the following asymptotic behavior of the logarithm of random variable.
\begin{thm}[Theorem 21 in \citet{poczos11a}]
\label{thm:poczos_21}
Suppose that the boundedness of an expectation on $p$ in Assumptions \ref{assmp:p_bounded} and \ref{assmp:p_bounded_cdf} holds.
If $0 \leq \kappa$ and $\xi_{n,k,x}^{\kappa} \rightarrow_{d} \xi_{k,x}^{\kappa}$, or $-k < \kappa < 0$ and $\xi_{n,k,x}^{\kappa} \rightarrow_{d} \xi_{k,x}^{\kappa}$, then $\lim_{n\rightarrow \infty}\mathbb{E}[\xi_{n,k,x}^{\kappa}] = \mathbb{E}[\xi_{k,x}^{\kappa}]$.
\end{thm}
\begin{thm}[The asymptotic expectation]
\label{thm:asymp_expect}
Suppose that the boundedness of an expectation on $p$ in Assumptions \ref{assmp:p_bounded} and \ref{assmp:p_bounded_cdf} holds.
If $-k < \kappa < 0$, or $0 \leq \kappa$, then we obtain
\begin{align*}
    &\lim_{n \rightarrow \infty} \mathbb{E} \bigg[ \log (n-1)^{\kappa} \rho_{k}^{d\kappa}(1) | X_{1} = x \bigg] = \kappa (\psi(k) - \log(\bar{c}p(x))), \\
    &\lim_{m \rightarrow \infty} \mathbb{E} \bigg[ \log m^{\kappa} \nu_{k}^{d\kappa}(1) | X_{1} = x \bigg] = \kappa (\psi(k) - \log(\bar{c}q(x))).
\end{align*}
\begin{proof}
It is enough to show the first equation because the second equation can be showed in the same way.
According to Lemma \ref{lem:conv_in_dist}, we obtain $\xi_{n,k,x}^{\kappa} \rightarrow_{d} \xi_{k,x}^{\kappa}$ for almost all $x \in \mathcal{M}$.
Then,
\begin{align*}
    &\lim_{n \rightarrow \infty} \mathbb{E} \bigg[ \log (n-1)^{\kappa} \rho_{k}^{d\kappa}(1) | X_{1} = x \bigg] 
    = \kappa \lim_{n \rightarrow \infty} \mathbb{E} \bigg[ \log (n-1) \rho_{k}^{d}(1) | X_{1} = x \bigg] \\
    &= \kappa \lim_{n \rightarrow \infty} \mathbb{E} \bigg[ \zeta_{n,k,1} | X_{1} = x \bigg]
    = \kappa \lim_{n \rightarrow \infty} \mathbb{E} \bigg[ \xi_{n,k,x} \bigg]
    = \kappa \mathbb{E}[\lim_{n \rightarrow \infty} \xi_{n,k,x}] \ \ (\xi_{n,k,x}^{\kappa} \rightarrow_{d} \xi_{k,x}^{\kappa} \text{ by Lemma \ref{lem:conv_in_dist}}) \\
    &= \kappa \mathbb{E}[\xi_{k,x}] \ (\text{by Theorem~\ref{thm:poczos_21}}) \\
    &= \kappa \int_{0}^{\infty} u f_{x,k}(u)\mathrm{d}u
    = \kappa (\psi(k) - \log(\lambda(x))) \ (\text{by Lemma~\ref{lem:exp_log_erlang}}) \\
    &= \kappa (\psi(k) - \log(\bar{c}p(x))).
\end{align*}
Thus, the claim is proved.
\end{proof}
\end{thm}
\begin{thm}
\label{thm:asymp_expect_m}
Suppose that the boundedness of an expectation on $q$ in Assumption \ref{assmp:p_bounded}-\ref{assmp:p_bounded_extra} holds.
If $-k < \kappa < 0$, or $0 \leq \kappa$, then we obtain
\begin{align*}
    &\lim_{m \rightarrow \infty} \mathbb{E} \bigg[ \log (m-1)^{\kappa} \bar{\rho}_{k}^{d\kappa}(1) | Y_{1} = y \bigg] = \kappa (\psi(k) - \log(\bar{c}q(y))).
\end{align*}
\begin{proof}
We can show this in the same way of Theorem \ref{thm:asymp_expect} by substituting $n,\rho_{k}^{d},x$ to $m,\bar{\rho}_{k}^{d},y$.
\end{proof}
\end{thm}

To show Theorem \ref{thm:almost_surely_conv}, we focus on $\hat{p}_{k(n)}^{\gamma}(x)$, $\hat{q}_{k(n)}^{\gamma}(x)$ and $\hat{q}_{k(m)}^{\gamma}(y)$ and guarantee the convergence in probability of each estimators.

\begin{lem}[Moments of inverse Erlang distribution]
\label{lem:moments_inverse_Erlang}
Let $f_{x,k} = \frac{1}{\Gamma(k)} \lambda^{k}(x) u^{-1-k} \exp(-\lambda(x)/u)$ be the density of inverse Erlang distribution with parameters $\lambda(x) > 0$ and $k \in \mathbb{Z}^{+}$.
Let $\kappa \in \mathbb{R}$ such that $\kappa < k$.
Then, the $\kappa$-th moments of inverse Erlang distribution can be calcurated as
\begin{align*}
    \int_{0}^{\infty} u^{\kappa} f_{x,k}(u) \mathrm{d}u = \lambda^{\kappa}(x) \frac{\Gamma(k-\kappa)}{\Gamma(k)}.
\end{align*}
\begin{proof}
The $\kappa$-th moments of $f_{x,k}$ is
\begin{align*}
    \int_{0}^{\infty} u^{\kappa} f_{x,k}(u) \mathrm{d}u &= \int_{0}^{\infty} u^{\kappa}\frac{1}{\Gamma(k)} \lambda^{k}(x) u^{-1-k} \exp(-\lambda(x)/u)\mathrm{d}u \\
    &= \frac{\lambda^{k}(x)}{\Gamma(k)}\int_{0}^{\infty} u^{-1-(k-\kappa)} \exp(-\lambda(x)/u)\mathrm{d}u.
\end{align*}
If $k > \kappa$, the integral term in the above equals to the marginalization of inverse gamma distribution.
Thus,
\begin{align*}
    \int_{0}^{\infty} u^{\kappa} f_{x,k}(u) \mathrm{d}u &=\frac{\lambda^{k}(x)}{\Gamma(k)}\int_{0}^{\infty} u^{-1-(k-\kappa)} \exp(-\lambda(x)/u)\mathrm{d}u \\
    &= \frac{\lambda^{k}(x)}{\Gamma(k)} \frac{\Gamma(k-\kappa)}{\lambda^{k-\kappa}(x)}
    = \lambda^{\kappa}(x)\frac{\Gamma(k-\kappa)}{\Gamma(k)}.
\end{align*}
The claim is proved.
\end{proof}
\end{lem}
\begin{lem}[$\hat{p}_{k(n)}^{\gamma}(x)$ converges to $p^{\gamma}(x)$ in probability]
\label{lem:p_converge_prob}
Suppose that Assumptions \ref{assmp:p_bounded} and \ref{assmp:p_bounded_cdf} are satisfied.
Let $\kappa = \gamma < k$.
If $k(n)$ denotes the number of neighbors applied at sample size $n$, $\lim_{n \rightarrow \infty} k(n) = \infty$ and $\lim_{n \rightarrow \infty} n/k(n) = \infty$, then $\hat{p}_{k(n)}^{\gamma}(x) \rightarrow_{p} p_{k(n)}^{\gamma}(x)$ for almost all $x$.
\begin{proof}
According to the Chebyshev's inequality, if we set $X_{i} = x$, $k(n) = k$ and $\epsilon > 0$, we obtain
\begin{align*}
    \mathbb{P}(|\hat{p}_{k}^{\gamma}(x) -  p_{k}^{\gamma}(x)| > \epsilon)
    &\leq \frac{1}{\epsilon^{2}} \mathbb{V}[\hat{p}_{k}^{\gamma}(x)]
    = \frac{1}{\epsilon^{2}}\mathbb{V} \bigg[\bigg(\frac{k}{(n-1)\bar{c}\rho_{k}^{d}(i)} \bigg)^{\gamma} \bigg] \\
    &= \frac{1}{\epsilon^{2}} \bigg(\frac{k}{(n-1)\bar{c}}\bigg)^{2\gamma} \mathbb{V} \bigg[ \frac{1}{\rho_{k}^{d\gamma}(i)} \bigg]
    = \frac{1}{\epsilon^{2}} \bigg(\frac{1}{\bar{c}}\bigg)^{2\gamma} \bigg(\frac{k}{n-1}\bigg)^{2\gamma} \mathbb{V} \bigg[ \frac{1}{\rho_{k}^{d\gamma}(i)} \bigg].
\end{align*}
According to Corollary 1 of \citet{PrezCruz08}, the random variable $\rho_{k}^{d}(i)$ measures the waiting time between the origin and the $k$-th event of a uniformly spaced distribution, and this waiting time is distributed as an Erlang distribution or a unit-mean and $1/k$ variance gamma distribution.
Therefore, the random variable $1/\rho_{k}^{d}(i)$ is distributed as an inverse Erlang distribution.

According to Lemma \ref{lem:moments_inverse_Erlang} and $\gamma < k$, the moments of $1/ \rho_{k}^{d\gamma}(i)$ can be calculated.
Therefore, we can see $\mathbb{V} \bigg[ \frac{1}{\rho_{k}^{d\gamma}(i)} \bigg] < \infty$.
According to the assumption that $\lim_{n \rightarrow \infty} n/k(n) = \infty$, we obtain $\lim_{n \rightarrow \infty} k(n)/n = 0$ and therefore
\begin{align*}
    \lim_{n \rightarrow \infty}\mathbb{P}(|\hat{p}_{k}^{\gamma}(x) -  p_{k}^{\gamma}(x)| > \epsilon) &\leq  \lim_{n \rightarrow \infty}\frac{1}{\epsilon^{2}} \bigg(\frac{1}{\bar{c}}\bigg)^{2\gamma} \bigg(\frac{k}{n-1}\bigg)^{2\gamma} \mathbb{V} \bigg[ \frac{1}{\rho_{k}^{d\gamma}(i)} \bigg] 
    = 0,
\end{align*}
for any $x$ in the support of $p(x)$ and any $\epsilon$.
The claim is proved.
\end{proof}
\end{lem}
\begin{lem}[$\hat{q}_{k(n)}^{\gamma}(x)$ converges to $q^{\gamma}(x)$ in probability]
\label{lem:q_converge_prob}
Suppose that Assumptions \ref{assmp:p_bounded} and \ref{assmp:p_bounded_cdf} are satisfied.
Let $0 < \kappa = \gamma < k$.
If $k(n)$ denotes the number of neighbors applied at sample size $n$, $\lim_{n \rightarrow \infty} k(n) = \infty$ and $\lim_{n \rightarrow \infty} n/k(n) = \infty$, then $\hat{q}_{k(n)}^{\gamma}(x) \rightarrow_{p} q_{k(n)}^{\gamma}(x)$ for almost all $x$.
\begin{proof}
It can be shown in the same way of Lemma \ref{lem:p_converge_prob}.
\end{proof}
\end{lem}
\begin{lem}[$\hat{q}_{k(m)}^{\gamma}(y)$ converges to $q^{\gamma}(y)$ in probability]
\label{lem:q_converge_prob_m}
Suppose that Assumptions \ref{assmp:p_bounded}-\ref{assmp:p_bounded_extra} are satisfied.
Let $\kappa = \gamma < k$.
If $k(n)$ denotes the number of neighbors applied at sample size $m$, $\lim_{m \rightarrow \infty} k(m) = \infty$ and $\lim_{m \rightarrow \infty} n/k(m) = \infty$, then $\hat{q}_{k(m)}^{\gamma}(y) \rightarrow_{p} q_{k(m)}^{\gamma}(y)$ for almost all $y$.
\begin{proof}
It can be shown in the same way of Lemma \ref{lem:p_converge_prob} by substituting $n,\rho_{k}^{d},x$ to $m,\bar{\rho}_{k}^{d},y$.
\end{proof}
\end{lem}

\section{Proofs for Asymptotic Analysis}
\label{app:proof_asymp_anal}
In this section, we summarize the essential theoretical analysis for our estimator to guarantee the main characteristics.

\subsection{Proof of Theorem \ref{thm:asymp_unbiased_others_p}}
\label{proof:unbiasedness}
The following lemma is necessary to show Theorem \ref{thm:asymp_unbiased_others_p}.
\begin{lem}[Switching limit and expectation]
\label{lem:switch_lim}
Let $\kappa > 0$ or $-k<\kappa<0$.Then, the following equality holds.
\begin{align*}
    &\lim_{n \rightarrow \infty}\int_{\mathcal{M}}f_{n}(x)p(x)\mathrm{d}x = \int_{\mathcal{M}} \lim_{n \rightarrow \infty}f_{n}(x)p(x)\mathrm{d}x, \\
    &\lim_{m \rightarrow \infty}\int_{\mathcal{M}}g_{m}(x)p(x)\mathrm{d}x = \int_{\mathcal{M}} \lim_{m \rightarrow \infty}g_{m}(x)p(x)\mathrm{d}x, \\
    &\lim_{m \rightarrow \infty}\int_{\mathcal{M}'}\bar{g}_{m}(y)q(y)\mathrm{d}y = \int_{\mathcal{M}'} \lim_{m \rightarrow \infty}\bar{g}_{m}(y)q(y)\mathrm{d}y,
\end{align*}
where
\begin{align*}
    f_{n}(x) \coloneqq \mathbb{E}\bigg[\log (n-1)^{\kappa}\rho_{k}^{d\kappa}(1)|X_{1}=x\bigg], \
    g_{m}(x) \coloneqq \mathbb{E}\bigg[\log m^{\kappa}\nu_{k}^{d\kappa}|X_{1}=x\bigg], \
    \bar{g}_{m}(y) \coloneqq \mathbb{E}\bigg[\log (m-1)^{\kappa}\bar{\rho}_{k}^{d\kappa}|Y_{1}=y\bigg].
\end{align*}
\begin{proof}
\citet{poczos11a} proved in Theorem 37
\begin{align*}
    &f'_{n}(x) \coloneqq \int_{0}^{\infty} u^{\kappa} F'_{n,k,x_{1}}\mathrm{d}u \leq \kappa L(x,1,\kappa,k,p,\delta,\delta_{1})
    < \infty \ (\kappa > 0), \\
    &f'_{n}(x) \coloneqq \leq \kappa \Bigg[\frac{\hat{L}(\bar{p},1)}{k+\kappa} - \frac{1}{\kappa} \Bigg] < \infty \ (-k<\kappa<0),
\end{align*}
where
\begin{align*}
  L(x,\omega,\kappa,k,p,\delta,\delta_{1}) \coloneqq \delta_{1} + \delta_{1}\int \|x-y\|^{\kappa}p(y)\mathrm{d}y + (\bar{c}r(x))^{-\kappa}H(x,p,\delta,\omega),
\end{align*}
and
\begin{align*}
    f'_{n}(x) \coloneqq \mathbb{E}\bigg[(n-1)^{\kappa}\rho_{k}^{d\kappa}(1)|X_{1} = x\bigg],
\end{align*}
and $F'_{n,k,x_{1}}$ is the conditional density function for $\zeta_{n,k,x_{1}}^{'\kappa} = (n-1)\rho_{k}^{d}(1)$.
According to the fact that if $a(x) \leq b(x)$ then $\mathbb{E}[a(x)] \leq \mathbb{E}[b(x)]$, we can obtain
\begin{align*}
    f_{n}(x) \leq f'_{n}(x) < \infty.
\end{align*}
We can also obtain
\begin{align*}
    g'_{m}(x) < \infty, \ \bar{g}'_{m}(y) < \infty,
\end{align*}
where
\begin{align*}
    g'_{m}(x) \coloneqq \mathbb{E}\bigg[m^{\kappa}\nu_{k}^{d\kappa}(1)|X_{1} = x\bigg], \ 
    \bar{g}'_{m}(y) \coloneqq \mathbb{E}\bigg[(m-1)^{\kappa}\bar{\rho}_{k}^{d\kappa}(1)|Y_{1} = y\bigg],
\end{align*}
In the same way as Theorem 37 of \citet{poczos11a}.
Therefore, the following inequality holds:
\begin{align*}
    g_{m}(x) \leq g'_{m}(x) < \infty, \ \bar{g}_{m}(y) \leq \bar{g}'_{m}(y) < \infty.
\end{align*}
From these, for $0 < \kappa < k$ or $-k < \kappa < 0$, we can see that under the conditions in Theorem \ref{thm:asymp_unbiased_others_p}, there exist some functions $J_{1},J_{2},J_{3}$ and threshold numbers $N_{p,q,1},N_{p,q,2},N_{p,q,3}$ such that if $n,m > N_{p,q,1}$, $n,m > N_{p,q,2}$ and $n,m > N_{p,q,3}$, then for almost all $x \in \mathcal{M}$ and $y \in \mathcal{M}'$, $f_{n}(x)\leq J_{1}(x)$, $g_{m}(x)\leq J_{2}(x)$ and $\bar{g}_{m}(y)\leq J_{3}(y)$ and $\int_{M}J_{1}(x)p(x) \mathrm{d}x < \infty$, $\int_{M}J_{2}(x)p(x) \mathrm{d}x < \infty$ and $\int_{M'}J_{3}(x)q(y) \mathrm{d}y < \infty$.
By applying the Lebesgue dominated convergence theorem, the claim is proved.
\end{proof}
\end{lem}

By using these lemmas and theorem in Appendix \ref{app:asymptotic_analysis} and Lemma \ref{lem:switch_lim}, we show asymptotic unbiasedness of our estimator claimed in Theorem \ref{thm:asymp_unbiased_positive} and \ref{thm:asymp_unbiased_q}.
\begin{thm}[Asymptotic unbiasedness]
\label{thm:asymp_unbiased_positive}
Let $\kappa \coloneqq \gamma$ and suppose $0 < \gamma < k$.
Suppose that Assumptions \ref{assmp:p_bounded}-\ref{assmp:p_bounded_extra} are satisfied, and that $q$ is bounded from above. Then, $\widehat{D}_{\gamma}(X^{n}\|Y^{m})$ is asymptotically unbiased, i.e.,
\begin{align*}
    \lim_{n,m \rightarrow \infty} \mathbb{E}\bigg[\widehat{D}_{\gamma}(X^{n}\|Y^{m}) \bigg] = D_{\gamma}(p\|q),
\end{align*}
where $\widehat{D}_{\gamma}(X^{n}\|Y^{m})$ is defined in Eq.~\eqref{eq:default_gamma_est}.
\begin{proof}
Now, we want to show that
\begin{align*}
    D_{\gamma}(p \| q) = \lim_{n,m \rightarrow \infty} \mathbb{E}\bigg[\widehat{D}_{\gamma}(p(X^{n}) \| q(Y^{m}))\bigg].
\end{align*}
If we use Eq.~\eqref{eq:default_gamma_est} as the $\gamma$-divergence estimator, it can be rewritten as
\begin{align}
\label{eq:estimator_naive}
    &\widehat{D}_{\gamma}(p(X^{n}) \| q(Y^{m})) \notag \\
    &= \frac{1}{\gamma(1+\gamma)} \bigg[ \log \bigg( \frac{1}{n} \sum_{i=1}^{n} \bigg(\frac{k}{(n-1)\bar{c}\rho_{k}^{d}(i)} \bigg)^{\gamma} \bigg) - (1+\gamma) \log \bigg(\frac{1}{n} \sum_{i=1}^{n} \bigg(\frac{k}{m\bar{c}\nu_{k}^{d}(i)} \bigg)^{\gamma} \bigg) \notag \\
    &\ \ \ \ \ \ \ \ \ \ \ \ \ \ \ \ \ \ \ \ \ \ \ \ \ \ \ \ \ \ \ \ \ \ \ \ \ \ \ \ \ \ \ \ \ \ \ \ \ \ \ \ \ \ \ \ \ \ \ \ \ \ \ \ \ \ \ \ \ \ \ \ \ \ \ \ \ \ \ \ \ \ \ \ \ \ \ \ \ \ \ \ \ \ \ \ \ \ \ \ \ \ \ \ \ \ \ \ \ \ \ \ \ + \gamma \log \bigg(\frac{1}{m} \sum_{j=1}^{m} \bigg(\frac{k}{(m-1)\bar{c}\bar{\rho}_{k}^{d}(j)} \bigg)^{\gamma}\bigg) \bigg] \notag \\
    &= \frac{1}{\gamma(1+\gamma)} \bigg[ \log \bigg(\frac{k}{\bar{c}}\bigg)^{\gamma} + \log \bigg( \frac{1}{n} \sum_{i=1}^{n} \bigg(\frac{1}{(n-1)\rho_{k}^{d}(i)} \bigg)^{\gamma} \bigg) - (1+\gamma) \log \bigg(\frac{k}{\bar{c}}\bigg)^{\gamma} \notag \\
    &\ \ \ \ \ \ \ \ \ \ \ \ \ \ \ \ \ \ \ \ \ \ \ - (1+\gamma) \log \bigg( \frac{1}{n} \sum_{i=1}^{n} \bigg(\frac{1}{m\nu_{k}^{d}(i)} \bigg)^{\gamma}\bigg) + \gamma \log \bigg(\frac{k}{\bar{c}}\bigg)^{\gamma} + \gamma \log \bigg(  \frac{1}{m} \sum_{j=1}^{m} \bigg(\frac{1}{(m-1)\bar{\rho}_{k}^{d}(j)} \bigg)^{\gamma} \bigg) \bigg] \notag \\
    &= \frac{1}{\gamma(1+\gamma)} \bigg[\log \bigg( \frac{1}{n} \sum_{i=1}^{n} \bigg(\frac{1}{(n-1)\rho_{k}^{d}(i)} \bigg)^{\gamma} \bigg) - (1+\gamma) \log \bigg( \frac{1}{n}\sum_{i=1}^{n} \bigg(\frac{1}{m\nu_{k}^{d}(i)} \bigg)^{\gamma} \bigg) \notag \\
    &\ \ \ \ \ \ \ \ \ \ \ \ \ \ \ \ \ \ \ \ \ \ \ \ \ \ \ \ \ \ \ \ \ \ \ \ \ \ \ \ \ \ \ \ \ \ \ \ \ \ \ \ \ \ \ \ \ \ \ \ \ \ \ \ \ \ \ \ \ \ \ \ \ \ \ \ \ \ \ \ \ \ \ \ \ \ \ \ \ \ \ \ \ \ \ \ \ \ \ \ \ \ \ \ \ \ \ \ \ \ \ \ \ + \gamma \log \bigg(  \frac{1}{m} \sum_{j=1}^{m} \bigg(\frac{1}{(m-1)\bar{\rho}_{k}^{d}(j)} \bigg)^{\gamma} \bigg) \bigg].
\end{align}

Taking expectation and a limit and switching the limit and expectation by using Lemma \ref{lem:switch_lim}, we can obtain
\begin{align*}
    &\lim_{n,m \rightarrow \infty} \mathbb{E}\bigg[\widehat{D}_{\gamma}(p(X^{n}) \| q(Y^{m}))\bigg] \notag \\
    &= \lim_{n,m \rightarrow \infty} \frac{1}{\gamma(1+\gamma)} \bigg[\log \bigg( \frac{1}{n} \sum_{i=1}^{n} \bigg(\frac{1}{(n-1)\rho_{k}^{d}(i)} \bigg)^{\gamma} \bigg) - (1+\gamma) \log \bigg( \frac{1}{n}\sum_{i=1}^{n} \bigg(\frac{1}{m\nu_{k}^{d}(i)} \bigg)^{\gamma} \bigg) \notag \\
    &\ \ \ \ \ \ \ \ \ \ \ \ \ \ \ \ \ \ \ \ \ \ \ \ \ \ \ \ \ \ \ \ \ \ \ \ \ \ \ \ \ \ \ \ \ \ \ \ \ \ \ \ \ \ \ \ \ \ \ \ \ \ \ \ \ \ \ \ \ \ \ \ \ \ \ \ \ \ \ \ \ \ \ \ \ \ \ \ \ \ \ \ \ \ \ \ \ \ \ \ \ \ \ \ \ \ \ \ \ \ \ \ \ + \gamma \log \bigg(  \frac{1}{m} \sum_{j=1}^{m} \bigg(\frac{1}{(m-1)\bar{\rho}_{k}^{d}(j)} \bigg)^{\gamma} \bigg) \bigg] \\
    &= \lim_{n,m \rightarrow \infty} \frac{1}{\gamma(1+\gamma)} \mathbb{E}_{X_{1}\sim p} \Bigg[ \mathbb{E}\bigg[\log \bigg( \frac{1}{(n-1)^{\gamma} \rho_{k}^{d\gamma}(1)} \bigg) \bigg| X_{1} = x \bigg] - (1+\gamma) \mathbb{E}\bigg[\log \bigg( \frac{1}{m^{\gamma}\nu_{k}^{d\gamma}(1)} \bigg) \bigg| X_{1} = x \bigg] \Bigg] \notag \\
    &\ \ \ \ \ \ \ \ \ \ \ \ \ \ \ \ \ \ \ \ \ \ \ \ \ \ \ \ \ \ \ \ \ \ \ \ \ \ \ \ \ \ \ \ \ \ \ \ \ \ \ \ \ \ \ \ \ + \frac{1}{1+\gamma} \mathbb{E}_{Y_{1} \sim q} \Bigg[ \mathbb{E}\bigg[\log \bigg(\frac{1}{(m-1)^{\gamma}\bar{\rho}_{k}^{d\gamma}(j)} \bigg) \bigg| Y_{1} = y \bigg] \Bigg] \\
    &= \frac{1}{\gamma(1+\gamma)} \mathbb{E}_{X_{1}\sim p} \Bigg[ \lim_{n \rightarrow \infty} \mathbb{E}\bigg[\log \bigg( \frac{1}{(n-1)^{\gamma} \rho_{k}^{d\gamma}(1)} \bigg) \bigg| X_{1} = x \bigg] - (1+\gamma) \lim_{m \rightarrow \infty} \mathbb{E}\bigg[\log \bigg( \frac{1}{m^{\gamma}\nu_{k}^{d\gamma}(1)} \bigg) \bigg| X_{1} = x \bigg] \Bigg] \notag \\
    &\ \ \ \ \ \ \ \ \ \ \ \ \ \ \ \ \ \ \ \ \ \ \ \ \ \ \ \ \ \ \ \ \ \ \ \ \ \ \ \ \ \ \ \ \ \ \ \ \ \ \ \ \ \ \ \ \
    + \frac{1}{1+\gamma}\mathbb{E}_{Y_{1} \sim q} \Bigg[ \lim_{m \rightarrow \infty} \mathbb{E}\bigg[\log \bigg(\frac{1}{(m-1)^{\gamma}\bar{\rho}_{k}^{d\gamma}(j)} \bigg) \bigg| Y_{1} = y \bigg] \Bigg].
\end{align*}
According to Theorem \ref{thm:asymp_expect_m}, we obtain
\begin{align*}
\lim_{n,m \rightarrow \infty} &\mathbb{E}\bigg[\widehat{D}_{\gamma}(p(X^{n}) \| q(Y^{m}))\bigg] \\
&= \frac{1}{\gamma(1+\gamma)}\mathbb{E}_{X_{1}\sim p} \Bigg[ -\gamma(\psi(k) - \log(\bar{c}p(X_{1}))) + \gamma(1+\gamma)(\psi(k) - \log(\bar{c}q(X_{1}))) \Bigg]  \notag \\
    &\ \ \ \ \ \ \ \ \ \ \ \ \ \ \ \ \ \ \ \ \ \ \ \ \ \ \ \ \ \ \ \ \ \ \ \ \ \ \ \ \ \ \ \ \ \ \ \ \ \ \ \ \ \ \ \ \
    - \frac{1}{1+\gamma}
    \mathbb{E}_{Y_{1} \sim q} \Bigg[ \gamma (\psi(k) - \log(\bar{c}q(Y_{1}))) \Bigg] \\
     &= \frac{1}{\gamma(1+\gamma)}\mathbb{E}_{X_{1}\sim p} \Bigg[ \gamma \log \bar{c} + \gamma \log p(X_{1}) - \gamma(1+\gamma)\log \bar{c} - \gamma(1+\gamma) \log q(X_{1}) + \gamma^{2} \psi(k) \Bigg] \notag \\
     &\ \ \ \ \ \ \ \ \ \ \ \ \ \ \ \ \ \ \ \ \ \ \ \ \ \ \ \ \ \ \ \ \ \ \ \ \ \ \ \ \ \ \ \ \ \ \ \ \ \ \ \ \ \ \ \ \
     - \frac{1}{1+\gamma}
    \mathbb{E}_{Y_{1} \sim q} \Bigg[ \gamma \psi(k) - \gamma \log \bar{c} - \gamma \log q(Y_{1}) \Bigg]  \\
     &= \frac{1}{\gamma(1+\gamma)}\mathbb{E}_{X_{1}\sim p} \bigg[\log p^{\gamma}(X_{1}) - (1+\gamma) \log q^{\gamma}(X_{1})  \bigg] + \frac{1}{1+\gamma}
    \mathbb{E}_{Y_{1} \sim q} \bigg[ \log q^{\gamma}(Y_{1}) \bigg] \\
    &\ \ \ \ \ \ \ \ \ \ \ \ \ \ \ \ \
    - \frac{\gamma}{(1+\gamma)} \log \bar{c} + \frac{\gamma}{(1+\gamma)} \psi(k) + \frac{\gamma}{(1+\gamma)} \log \bar{c} - \frac{\gamma}{(1+\gamma)} \psi(k) \\
    &= \frac{1}{\gamma(1+\gamma)}\mathbb{E}_{X_{1}\sim p} \bigg[\log p^{\gamma}(X_{1})\bigg] - \frac{1}{\gamma}\mathbb{E}_{X_{1}\sim p} \bigg[ \log q^{\gamma}(X_{1})  \bigg] + \frac{1}{1+\gamma}
    \mathbb{E}_{Y_{1} \sim q} \bigg[ \log q^{\gamma}(Y_{1}) \bigg].
\end{align*}
Therefore, Eq.~\eqref{eq:estimator_naive} is asymptotically unbiased.
The claim is proved.
\end{proof}
\end{thm}

If $-k < \kappa \coloneqq \gamma < 0$, the asymptotic unbiasedness also holds.
\begin{thm}[Asymptotic unbiasedness]
\label{thm:asymp_unbiased_q}
Let $-k < \kappa \coloneqq \gamma < 0$. Suppose that Assumptions \ref{assmp:p_bounded}-\ref{assmp:p_bounded_extra} are satisfied.
Let $\exists \delta_{0}$ s.t. $\forall \delta \in (0,\delta_{0})$, $\int_{\mathcal{M}} H(x,p,\delta,1)q(x)\mathrm{d}x < \infty$, and that $p$ is bounded from above.
Let $\mathrm{supp}(p) \supseteq \mathrm{supp}(q)$.
Then, the estimator in Eq.~\eqref{eq:default_gamma_est} is asymptotically unbiased.
\begin{proof}
This theorem can be shown in the same way as Theorem \ref{thm:asymp_unbiased_others_p}.
\end{proof}
\end{thm}

By combining the results of Theorem \ref{thm:asymp_unbiased_positive} and \ref{thm:asymp_unbiased_q}, Theorem \ref{thm:asymp_unbiased_others_p} can be shown.

\subsection{Proofs of Theorem \ref{thm:almost_surely_conv}}
\label{proof:almost_surely_convergence}
\begin{proof}
Recalling the default formulation of $\gamma$-divergence estimator in Eq.~\eqref{eq:default_gamma_est}, we can see
\begin{align*}
    &\widehat{D}_{\gamma}(X^{n}\|Y^{m}) \\
    &= \frac{1}{\gamma(1+\gamma)} \log \bigg( \frac{1}{n} \sum_{i=1}^{n}\hat{p}_{k}^{\gamma}(X_{i}) \bigg) - \frac{1}{\gamma} \log \bigg(\frac{1}{n} \sum_{i=1}^{n}\hat{q}_{k}^{\gamma}(X_{i}) \bigg) + \frac{1}{1+\gamma} \log \bigg(\frac{1}{m} \sum_{j=1}^{m} \hat{q}_{k}^{\gamma}(Y_{j}) \bigg) \\
    &= \frac{1}{\gamma(1+\gamma)} \log \bigg( \frac{1}{n} \sum_{i=1}^{n}p^{\gamma}(X_{i}) \bigg) - \frac{1}{\gamma} \log \bigg(\frac{1}{n} \sum_{i=1}^{n}q^{\gamma}(X_{i}) \bigg) + \frac{1}{1+\gamma} \log \bigg(\frac{1}{m} \sum_{j=1}^{m} q_{k}^{\gamma}(Y_{j}) \bigg) \\
    &\ \ \ \ \ \ \ \ \ \ \ \ \ \
    - \frac{1}{\gamma(1+\gamma)} \log \bigg( \frac{1}{n} \sum_{i=1}^{n}p^{\gamma}(X_{i}) \bigg) + \frac{1}{\gamma} \log \bigg(\frac{1}{n} \sum_{i=1}^{n}q^{\gamma}(X_{i}) \bigg) - \frac{1}{1+\gamma} \log \bigg(\frac{1}{m} \sum_{j=1}^{m} q_{k}^{\gamma}(Y_{j}) \bigg) \\
    &\ \ \ \ \ \ \ \ \ \ \ \ \ \ \ \ \
    + \frac{1}{\gamma(1+\gamma)} \log \bigg( \frac{1}{n} \sum_{i=1}^{n}\hat{p}_{k}^{\gamma}(X_{i}) \bigg) - \frac{1}{\gamma} \log \bigg(\frac{1}{n} \sum_{i=1}^{n}\hat{q}_{k}^{\gamma}(X_{i}) \bigg) + \frac{1}{1+\gamma} \log \bigg(\frac{1}{m} \sum_{j=1}^{m} \hat{q}_{k}^{\gamma}(Y_{j}) \bigg) \\
    &= \frac{1}{\gamma(1+\gamma)} \log \bigg( \frac{1}{n} \sum_{i=1}^{n}p^{\gamma}(X_{i}) \bigg) - \frac{1}{\gamma} \log \bigg(\frac{1}{n} \sum_{i=1}^{n}q^{\gamma}(X_{i}) \bigg) + \frac{1}{1+\gamma} \log \bigg(\frac{1}{m} \sum_{j=1}^{m} q_{k}^{\gamma}(Y_{j}) \bigg) \\
    &\ \ \ \ \ \ \ \ \ \ \ \ \ \
    + \frac{1}{\gamma(1+\gamma)} \log \frac{\frac{1}{n} \sum_{i=1}^{n}\hat{p}_{k}^{\gamma}(X_{i})}{\frac{1}{n} \sum_{i=1}^{n}p^{\gamma}(X_{i})} - \frac{1}{\gamma} \log \frac{\frac{1}{n} \sum_{i=1}^{n}\hat{q}_{k}^{\gamma}(X_{i})}{\frac{1}{n} \sum_{i=1}^{n}q^{\gamma}(X_{i})} + \frac{1}{1+\gamma} \log \frac{\frac{1}{m} \sum_{j=1}^{m} \hat{q}_{k}^{\gamma}(Y_{j})}{\frac{1}{m} \sum_{j=1}^{m} q_{k}^{\gamma}(Y_{j}) }.
\end{align*}
The first, second and third terms converge to the expectation of $p^{\gamma}(x)$, $q^{\gamma}(x)$ and $q^{\gamma}(y)$, and therefore these terms converge to $D_{\gamma}(p\|q)$ almost surely because the sum of almost surely convergence terms also converges almost surely~\citep{grimmett01}.

(i) According to Lemma \ref{lem:p_converge_prob}, $\hat{p}^{\gamma}_{k}(x) \rightarrow_{p} p^{\gamma}(x)$ for almost all of $x$.
In addition, according to the fact that the sum of random variables that converge in probability converges almost surely~\citep{grimmett01}, we obtain
\begin{align*}
    \frac{1}{n} \sum_{i=1}^{n} \hat{p}_{k}^{\gamma}(X_{i}) \overset{\textrm{a.s.}}{\rightarrow} \mathbb{E}_{p(x)}[p^{\gamma}(x)].
\end{align*}
Therefore,
\begin{align*}
    \frac{1}{\gamma(1+\gamma)} \log \frac{\frac{1}{n} \sum_{i=1}^{n}\hat{p}_{k}^{\gamma}(X_{i})}{\frac{1}{n} \sum_{i=1}^{n}p^{\gamma}(X_{i})} \overset{\textrm{a.s.}}{\rightarrow} \frac{1}{\gamma(1+\gamma)} \log \frac{\mathbb{E}_{p(x)}[p^{\gamma}(x)]}{\mathbb{E}_{p(x)}[p^{\gamma}(x)]} = 0.
\end{align*}

(ii) According to Lemma \ref{lem:q_converge_prob}, $\hat{q}^{\gamma}_{k}(x) \rightarrow_{p} q^{\gamma}(x)$ for almost all of $x$.
In the same way of (i), we obtain
\begin{align*}
    \frac{1}{\gamma(1+\gamma)} \log \frac{\frac{1}{n} \sum_{i=1}^{n}\hat{q}_{k}^{\gamma}(X_{i})}{\frac{1}{n} \sum_{i=1}^{n}q^{\gamma}(X_{i})} \overset{\textrm{a.s.}}{\rightarrow} \frac{1}{\gamma(1+\gamma)} \log \frac{\mathbb{E}_{p(x)}[q^{\gamma}(x)]}{\mathbb{E}_{p(x)}[q^{\gamma}(x)]} = 0.
\end{align*}

(iii) According to Lemma \ref{lem:q_converge_prob_m}, we obtain
\begin{align*}
    \frac{1}{1+\gamma} \log \frac{\frac{1}{m} \sum_{j=1}^{m}\hat{q}_{k}^{\gamma}(Y_{i})}{\frac{1}{m} \sum_{j=1}^{m}q^{\gamma}(Y_{i})} \overset{\textrm{a.s.}}{\rightarrow} \frac{1}{1+\gamma} \log \frac{\mathbb{E}_{q(y)}[q^{\gamma}(y)]}{\mathbb{E}_{q(y)}[q^{\gamma}(y)]} = 0
\end{align*}

From (i) to (iii), we obtain
\begin{align*}
    \widehat{D}_{\gamma}(X^{n}\|Y^{m}) \overset{\textrm{a.s.}}{\rightarrow} D_{\gamma}(p\|q),
\end{align*}
and the claim is proved.
\end{proof}

\section{Detail of Data Discrepancy Measure}
\label{app:detail_discrepancy}
In this section, we introduce data discrepancy measures.

\subsection{Distance between Summary Statistics}
An ABC often uses the distance between the summary statistics: $S(X^{n})$ and $S(Y^{m})$ as the discrepancy measure.
If we use the Euclidian distance, the discrepancy measure can be expressed as
\begin{align*}
    D_{S}(X^{n},Y^{m}) = \|S(X^{n}) - S(Y^{m}) \|.
\end{align*}
However, it is difficult to choose the summary statistic $S$ for each task properly.
One way to bypass this difficulty is the Bayesian indirect inference method~\cite{Drovandi11,Drovandi15}.

\paragraph{Bayesian Indirect method}
The aim of the Bayesian indirect method is to construct the summary statistics from an auxiliary model: $\{p_{A}(x|\phi): \phi \in \Phi\}$ (see \citet{Drovandi15} for general review).
\citet{Drovandi11} proposed to use the maximum likelihood estimation (MLE) of the auxiliary model as summary statistics.
Formally,
\begin{align*}
    S(Y^{m}) = \hat{\phi}(Y^{m}) = \argmax_{\phi \in \Phi} \prod_{j=1}^{m} p_{A}(Y_{j}|\phi).
\end{align*}
We set $p_{A}(x|\phi)$ as $d$-dimensional Gaussian with parameter $\phi$ in our experiments.
In this setting, the summary statistics are merely the sample mean and covariance of $Y^{m}$.
Furthermore, we adopted the auxiliary likelihood (AL) proposed by \citet{Gleim13} as a data discrepancy:
\begin{align*}
    D_{\mathrm{AL}}(X^{n},Y^{m}) = \frac{1}{m} \log p_{A}(Y^{m}|\hat{\phi}(Y^{m})) - \frac{1}{m} \log p_{A}(Y^{m}|\hat{\phi}(X^{n})).
\end{align*}

\paragraph{Outlier-Robust Function as Summary Statistics}
\citet{Ruli20} proposed the robust M-estimator $\Psi$ as the summary statistics to deal with the outliers in the observed data.
For example, we can use the Huber function as
\begin{align*}
 \Psi(x - \mu) = \begin{cases}
 -c \ \ (x - \mu < -c), \\
 x - \mu \ \ (|x - \mu| \leq 0), \\
 c \ \ (x - \mu > c),
 \end{cases}
\end{align*}
where $\mu$ is a mean of $x$.
We adopted this function as the summary statistics and applied for the AL in the above.
Formally,
\begin{align*}
    D_{\mathrm{ALH}}(X^{n},Y^{m}) = \frac{1}{m} \log p_{A}(S_{\Phi}(Y^{m})|\hat{\phi}(S_{\Phi}(Y^{m}))) - \frac{1}{m} \log p_{A}(S_{\Phi}(Y^{m})|\hat{\phi}(S_{\Phi}(X^{n}))).
\end{align*}
Further, we set $c_{1} = 1.345$ for mean and $c_{2} = 2.07$ for covariance (see \citet{Huber09}).

\subsection{Maximum Mean Discrepancy (MMD) based Approach}
\paragraph{MMD method}
\citet{smola07} and \citet{berlinet04} defined the kernel embedding for a probability distribution $g(x)$ as
\begin{align*}
    \mu_{g} = \int k(\cdot,x)g(x) \mathrm{d}x,
\end{align*}
where $k$ is a positive definite kernel $k:\mathcal{X} \times \mathcal{X} \rightarrow \mathbb{R}$.
Therefore, $\mu_{g}$ is an element in the reproducing kernel Hilbert space (RKHS): $\mathcal{H}$.

The maximum mean discrepancy (MMD)~\cite{Gretton12} between the probability distributions $g_{0}$ and $g_{1}$ is the distance between the kernel embeddings $\mu_{g_{0}}$ and $\mu_{g_{1}}$ in RKHS $\mathcal{H}$, defined as
\begin{align*}
    \mathrm{MMD}^{2}(g_{0},g_{1}) = \|\mu_{g_{0}} - \mu_{g_{1}} \|_{\mathcal{H}}^{2}.
\end{align*}
\citet{Park16} applied an unbiased estimator of $\mathrm{MMD}^{2}(p_{\theta^{*}},q_{\theta})$ as the data discrepancy in ABC.
The squared estimator of MMD is defined as
\begin{align}
\label{eq:unbiased_MMD_est}
    D_{\mathrm{MM}}^{2}(X^{n},Y^{m}) = \frac{\sum_{1\leq i \neq j \leq n}k(X_{i},X_{j})}{n(n-1)} + \frac{\sum_{1\leq i \neq j \leq m}k(Y_{i},Y_{j})}{m(m-1)} - \frac{2\sum_{i=1}^{n}\sum_{j=1}^{m}k(X_{i},Y_{j})}{nm}.
\end{align}
In the same way of \citet{Park16} and \citet{Jiang18}, we chose a Gaussian kernel with the bandwidth being the median of $\{\| X_{i} - X_{j} : 1 \leq i \neq j \leq n\| \}$ in our experiments.
Then, the time cost of $D_{\mathrm{MM}}$ is $\mathcal{O}((n+m)^{2})$ which is caused to compute the $(n+m) \times (n+m)$ pairwise distance matrix.

\paragraph{Median-of-mean to Kernel (MONK) method}
\citet{Lerasle19} proposed the outlier-robust MMD estimator computed by using the median-of-mean (MON) estimator.
MON estimators are expected to enjoy the outlier-robustness thanks to the median step.

For any mapping function $h: \mathcal{X} \mapsto \mathbb{R}$ and any non-empty subset $S \subseteq \{1,2, \ldots, n \}$, 
denote by $\mathbb{P}_{S} = |S|^{-1} \sum_{i\in S} \delta_{X_{i}}$ the empirical measure associated to the subset $x_{S}$ and $\mathbb{P}_{S}h = |S|^{-1}\sum_{i\in S} h(X_{i})$.
For simplification, we express $\mu_{S} = \mu_{\mathbb{P}_{S}}$.
Let $n$ is divisible by $Q \in \mathbb{Z}^{+}$ and let $(S_{q})_{q \in Q}$ denote a partition of $\{1,2,\ldots, n\}$ into subsets with the same cardinality $|S_{q}| = N/Q$.
We also mention that $q$ is different from the distribution of $Y^{m}$ with parameter $\theta$ defined as $q_{\theta}$.
Then, the MON is defined as
\begin{align*}
    \mathrm{MON}_{Q}[h] = \med_{q} \{ \mathbb{P}_{S_{q}}, h \}
    = \med_{q} \{ k(h, \mu_{S_{q}}) \},
\end{align*}
where $h \in \mathcal{H}$ in the second equality is a consequence of the mean-reproducing property of $\mu_{\mathbb{P}}$.
When we choose $Q=1$, the MON estimator is equal to the classical mean as $\mathrm{MON}_{1} = n^{-1} \sum_{i=1}^{n}h(X_{i})$.

\citet{Lerasle19} defined the minimax MON-based estimator associated with Kernel $k$ (MONK) as
\begin{align*}
    \hat{\mu}_{\mathbb{P},Q} = \hat{\mu}_{\mathbb{P},Q}(X^{n}) \in \argmin_{f \in \mathcal{H}} \sup_{g \in \mathcal{H}} \tilde{J}(f,g),
\end{align*}
where for all $f,g \in \mathcal{H}$
\begin{align*}
    \tilde{J}(f,g) = \mathrm{MON}_{Q}\bigg[ x \mapsto \| f - k(\cdot,x) \|_{\mathcal{H}}^{2} - \| g - k(\cdot,x) \|_{\mathcal{H}}^{2} \bigg].
\end{align*}
When we choose $Q=1$, we
obtain the classical empirical mean based estimator as $\mu_{\mathbb{P},1} = n^{-1} \sum_{i=1}^{n}k(\cdot,X_{i})$.

The MON-based MMD estimator on $X^{n} \sim g_{0}$ and $Y^{m} \sim g_{1}$ is defined as
\begin{align*}
    \widehat{\mathrm{MMD}}_{Q}(g_{0},g_{1}) = \sup_{f} \med_{q\in Q} \{ k(f, \mu_{S_{q,g_{0}}} - \mu_{S_{q,g_{1}}}) \},
\end{align*}
where $\mu_{S_{q,g_{0}}} = \mu_{\mathbb{P}_{S_{q},X_{i}}}$ and $\mu_{S_{q,g_{1}}} = \mu_{\mathbb{P}_{S_{q},Y_{i}}}$.
Again, when we choose $Q=1$, this is equal to the classical V-statistic-based MMD estimator~\cite{Gretton12} in the previous paragraph.
The (unbiased) U-statistic based MONK estimator also could be obtained in the same way as Eq.~\eqref{eq:unbiased_MMD_est} (see \citet{Lerasle19}).

The MONK estimator has the time cost $\mathcal{O}(n^{3})$ and therefore $\mathcal{O}((n+m)^{3})$ when we use it as the data discrepancy in ABC.
It is too expensive to apply for a large sample size.
\citet{Leonenko08a} also proposed the faster algorithm to compute the MONK estimator, called MONK BCD-Fast, which has $\mathcal{O}((n+m)^{3}/Q^{2})$ time cost.
We adopted this algorithm in our experiments and set $Q=11$.
Furthermore, we adopted the RBF kernel with bandwidth $\sigma=1$, \diff{which is also used in \citet{Lerasle19}.}

\subsection{Wasserstein Distance}
\citet{Jiang18} mentioned that the estimator of the $q$-Wasserstein distance could be used as a data discrepancy for ABC.
Let $\psi$ be a distance on $\mathcal{X} \subseteq \mathbb{R}^{d}$.
The $q$-Wasserstein distance between $g_{0}$ and $g_{1}$ is defined as
\begin{align*}
    \mathcal{W}_{q}(g_{0},g_{1}) = \Bigg[\inf_{\tau \in \Gamma(g_{0},g_{1})} \int_{\mathcal{X} \times \mathcal{X}} \psi(x,y)^{q}\mathrm{d}\tau(x,y) \Bigg]^{1/q},
\end{align*}
where $\Gamma(g_{0},g_{1})$ is the set of all joint distribution $\tau(x,y)$ on $\mathcal{X} \times \mathcal{X}$ such that $\tau$ has marginals $g_{0}$ and $g_{1}$.
We also mention that $q$ is different from the distribution of $Y^{m}$ with parameter $\theta$ defined as $q_{\theta}$.
When we set $q=2$ and $\psi$ be the Euclidean distance, the data discrepancy based on the $q$-Wasserstein distance is given by
\begin{align*}
    D_{\mathcal{W}2}(X^{n},Y^{m}) = \min_{\tau} \Bigg[ \sum_{i=1}^{n}\sum_{j=1}^{m} \tau_{ij}\| X_{i} - Y_{j}\|^{2} \Bigg]^{1/2} \ \textrm{s.t.} \ \mathbf{\tau 1}_{m} = \mathbf{1}_{n}, \mathbf{\tau^{\top}}\mathbf{1}_{n} = \mathbf{1}_{m}, 0 \leq \tau_{ij} \leq 1,
\end{align*}
where $\tau = \{ \tau_{ij}; 1 \leq i \leq n, 1 \leq j \leq m \}$ is a $n \times m$ matrix and $\mathbf{1}_{n}, \mathbf{1}_{m}$ are vectors filled with $n$ pieces or $m$ pieces of 1, respectively.

When we want to solve the optimization problem of $D_{\mathcal{W}2}$ exactly on multivariate distributions ($d > 1$), we have the time cost $\mathcal{O}((n+m)^{3}\log(n+m))$~\cite{Burkard09}.
It is a high cost significantly and therefore \citet{Cuturi13} and \citet{cuturi14} proposed approximate optimization algorithms which reduce the time cost to $\mathcal{O}((n+m)^{2}$.
We used this algorithm in our experiments.
For univariate distributions, i.e., $d=1$, if $n=m$ and $\psi(x,y) = |x-y|$, the $q$-Wasserstein distance has an explicit form as
\begin{align*}
    \bigg(\frac{1}{n} \sum_{i=1}^{n} |X_{i} - Y_{i} |^{q} \bigg)^{1/q},
\end{align*}
and in this special case, the time cost is $\mathcal{O}(n \log n)$~\cite{Jiang18}.

\subsection{Classification Accuracy Method}
The classification accuracy discrepancy (CAD) has been proposed by \citet{Gutmann18}.
The idea of this method is on the basis of the belief that it is easier to distinguish the observed data $X^{n}$ and the synthetic data $Y^{m}$ when $\theta$ is different significantly to the true parameter $\theta^{*}$ than to do so when $\theta$ resembles $\theta^{*}$.

The CAD sets the labels of $\{X_{i}\}_{i=1}^{n}$ as class $0$ and $\{Y_{j}\}_{j=1}^{m}$ as class $1$ at first.
In short, it yields an augmented data set as
\begin{align*}
    \mathcal{D} = \{ (X_{1},0), (X_{2},0), \ldots, (X_{n},0), (Y_{1},1), (Y_{2},1), \ldots, (Y_{m},1) \},
\end{align*}
and then trains a prediction classifier $h: x \mapsto \{0,1 \}$.

\citet{Gutmann18} defined classifiability between $X^{n}$ and $Y^{m}$ as the $K$-fold cross-validation classification accuracy and proposed to use it for ABC as a data discrepancy.
The data discrepancy based on the CAD is defined as
\begin{align*}
    D_{\textrm{CAD}}(X^{n},Y^{m}) = \frac{1}{K} \sum_{k=1}^{K} \frac{1}{|\mathcal{D}_{k}|} \Bigg[ \sum_{i: (X_{i},0) \in \mathcal{D}_{k}}(1 - \hat{h}_{k}(X_{i})) + \sum_{j: (Y_{j},1) \in \mathcal{D}_{k}} \hat{h}_{k}(Y_{j}) \Bigg],
\end{align*}
where $\mathcal{D}_{k}$ is the $k$-fold subset of $\mathcal{D}$, $|\mathcal{D}_{k}|$ is the size of $\mathcal{D}_{k}$ and $\hat{h}_{k}$ is the trained predictor on the data set $\mathcal{D} \setminus \mathcal{D}_{k}$.

The discrepancy via linear Discriminant Analysis (LDA) is computationally cheaper than other classifiers, which is $\mathcal{O}(n+m)$; however, \citet{Gutmann18} explicitly noted that LDA does not work for some models, e.g., the moving average models (see Figure 2 in \citet{Gutmann18}).
Therefore, in our experiments, we set $K = 5$ and $h$ to be the logistic regression with $L_{1}$ regularization and the gradient boosting classifier.


\subsection{KL-divergence estimation via \texorpdfstring{$k$}{Lg}-NN}

KL-divergence between the density functions $p$ and $q$ is defined as
\begin{align}
    \label{KL-div}
    D_{\mathrm{KL}}(p\|q) = \int_{\mathcal{M}} p(x) \log \frac{p(x)}{q(x)} \mathrm{d}x,
\end{align}
where $\mathcal{M}$ is a support of $p$.
It indicates zero if and only if $p=q$ for almost everywhere.
\citet{PrezCruz08} proposed to estimate the density firstly by using k-NN density estimation
and plug these estimators into Eq.~\eqref{KL-div}.
Given \IID samples, $X^{n}$ and $Y^{m}$, we can estimate $D_{\mathrm{KL}}(p\|q)$
by using the $k$-NN density estimator expressed in Eqs.~\eqref{def:knn_est_p} and \eqref{def:knn_est_q} as follows:
\begin{align}
    \label{KL_div_est}
    \widehat{D}_{\mathrm{KL}}(p\|q) &= \frac{1}{n} \sum_{i=1}^{n} \log \frac{\hat{p}_{k}(X_{i})}{\hat{q}_{k}(X_{i})}
    = \frac{d}{n} \sum_{i=1}^{n} \log \frac{\rho_{k}(i)}{\nu_{k}(i)} + \log \frac{m}{n-1}.
\end{align}
This estimator enjoys asymptotical properties such as asymptotical unbiasedness, $L_{2}$-consistency and almost sure convergence (\cite{PrezCruz08,Wang09}).
If we use $1$-NN density estimation, the above estimator (\ref{KL_div_est}) can be expressed as
\begin{align}
    \label{KL_div_est_special}
    \widehat{D}_{\mathrm{KL}}(p\|q) = \frac{d}{n} \sum_{i=1}^{n} \log \frac{\min_{j} \| X_{i} - Y_{j} \|_{2}}{\min_{j \neq i}^{n}\| X_{i} - X_{j} \|_{2}} + \log \frac{m}{n-1},
\end{align}
where $\| \cdot \|_{2}$ means $l_{2}$-norm.

\citet{Jiang18} proposed to use this estimator (\ref{KL_div_est_special}) as the data discrepancy in the ABC framework.
As ABC involves $2n$ operations of nearest neighbor search,
\citet{Jiang18} also proposed to use $KD$ trees~\cite{Bentley75,Maneewongvatana01}.
The time cost thus is $\mathcal{O}((n \lor m) \log (n \lor m))$ on average, where we denote $\max\{ a, b\}$ as $a \lor b$.

According to Theorem 1 in \cite{Jiang18}, the asymptotic ABC posterior is a restriction of the prior $\pi$ on the region $\{ \theta \in \Theta : D(g_{\theta^{*}}\| g_{\theta}) < \epsilon \}$.
\begin{thm}[Theorem 1 in \cite{Jiang18}]
\label{thm:asymptotic_quasi_posterior}
Let the data discrepancy measure $D(X^{n},Y^{m})$ in Algorithm \ref{alg:reject_abc} converges to some real number $D(p_{\theta^{*}},q_{\theta})$ almost surely as the data size $n \rightarrow \infty$, $m/n \rightarrow \alpha > 0$.
Then, the ABC posterior distribution $\pi(\theta| X^{n};D,\epsilon)$ defined by \eqref{eq:quasi_posterior} converges to $\pi(\theta | D(p_{\theta^{*}},q_{\theta}) < \epsilon)$ for any $\theta$.
That is,
\begin{align*}
    \lim_{n \rightarrow \infty} \pi(\theta| X^{n};D,\epsilon) &= \pi(\theta | D(p_{\theta^{*}},q_{\theta}) < \epsilon)
    \propto \pi(\theta) \Ind{D(p_{\theta^{*}},q_{\theta}) < \epsilon}.
\end{align*}
\end{thm}
\citet{Jiang18} also showed the behavior of the ABC posterior based on KL-divergence estimator.
\begin{cor}[Corollary 1 in \cite{Jiang18}]
Let the data size $n \rightarrow \infty$, $m/n \rightarrow \alpha > 0$.
Let us define $\pi(\theta | D_{\mathrm{KL}}(p_{\theta^{*}},q_{\theta}) < \epsilon)$ as the posterior under $D_{\mathrm{KL}}(p_{\theta^{*}},q_{\theta}) < \epsilon)$.
If Algorithm \ref{alg:reject_abc} uses $\widehat{D}_{\mathrm{KL}}$ defined by Eq.~\eqref{KL_div_est_special} as the data discrepancy measure, then the ABC posterior distribution $\pi(\theta| X^{n};\widehat{D}_{\mathrm{KL}},\epsilon)$ defined by Eq.~\eqref{eq:quasi_posterior} converges to $\pi(\theta | D_{\mathrm{KL}}(p_{\theta^{*}},q_{\theta}) < \epsilon)$ for any $\theta$.
That is,
\begin{align*}
    \lim_{n \rightarrow \infty} \pi(\theta| X^{n};D_{\mathrm{KL}},\epsilon) &= \pi(\theta | D_{\mathrm{KL}}(p_{\theta^{*}},q_{\theta}) < \epsilon)
    \propto \pi(\theta) \Ind{D_{\mathrm{KL}}(p_{\theta^{*}},q_{\theta}) < \epsilon}.
\end{align*}
\end{cor}
It is known that the maximum likelihood estimator minimizes the KL-divergence between the empirical distribution of $p_{\theta^{*}}$ and $q_{\theta}$. ABC with $D_{\mathrm{KL}}$ shares the same idea to find $\theta$ with small KL-divergence.

\section{Details of Experimental Settings}
\label{app:detail_models_exp}
In this section, we summarize the details of the model settings we used in experiments.

\subsection{Gaussian Mixture Model (GM)}
The univariate Gaussian mixture model is the most fundamental benchmark model in ABC literature~\cite{Sisson07,Wilkinson13,Jiang18}.
We adopted a bivariate Gaussian mixture model with the true parameters $p^{*} = 0.3$, $\mu_{0}^{*} = (0.7, 0.7)$ and $\mu_{1}^{*} = (-0.7, -0.7)$, where $p^{*}$ means the mixture ratio and $\mu_{0}^{*}, \mu_{1}^{*}$ are sub-population means of Gaussian distribution.
Therefore, the set of the true parameter is $\theta^{*} = (p^{*},\mu_{0}^{*},\mu_{1}^{*})$
The generative process of data is as follows:
\begin{align*}
    &Z \sim \mathrm{Bernoulli}(p), \\
    &[X|Z=0] \sim \mathcal{N}(\mu_{0},[0.5,-0.3; -0.3,0.5]), \\
    &[X|Z=1] \sim \mathcal{N}(\mu_{1},[0.25,0; 0,0.25]).
\end{align*}

We set the $n=500$ observed data and the prior on the unknown parameter $\theta = (p, \mu_{0},\mu_{1})$ as $p \sim \mathrm{Uniform}[0,1]$ and $\mu_{0},\mu_{1} \sim \mathrm{Uniform}[-1,1]^{2}$.

\subsection{\textit{M}/\textit{G}/1-queueing Model (MG1)}
Queuing models are usually easy to simulate from; however,
it is difficult to conduct inference because these have no intractable likelihoods.
The $M$/$G$/$1$-queuing model well has been studied in ABC context~\cite{Blum10,fearnhead12,Jiang18}.
The $M$, $G$ and $1$ means \emph{Memoryless} which follows some arrival process, \emph{General  holding time distribution} and \emph{single server}, respectively.
In this model, the service times follows $\mathrm{Uniform}[\theta_{1},\theta_{2}]$ and the inter arrival times are exponentially distributed with rate $\theta_{3}$. Each datum is a $5$-dimensional vector consisting of the first five inter departure times $x = (x_{1}, x_{2}, x_{3}, x_{4}, x_{5})$ after the queue starts from empty~\cite{Jiang18}.

We adopted this model with the true parameters $\theta^{*} = (1,5,0.2)$.
We set the $n=500$ observed data and the prior on the unknown parameter $\theta = (\theta_{1}, \theta_{2},\theta_{3})$ as $\theta_{1} \sim \mathrm{Uniform}[0,10]$, $\theta_{2}-\theta_{1} \sim \mathrm{Uniform}[0,10]$ and $\theta_{3} \sim \mathrm{Uniform}[0,0.5]$.

\subsection{Bivariate Beta Model (BB)}
The bivariate beta model was proposed as a model with $8$ parameters $\theta = (\theta_{1},\ldots,\theta_{8})$ by \citet{Arnold11}.
The generative process of data is as follows:
\begin{align*}
    &U_{i} \sim \mathrm{Gamma}(\theta_{i},1) \ (i = 1,\ldots,8), \\
    &V_{1} = \frac{U_{1}+U_{5}+U_{7}}{U_{3}+U_{6}+U_{8}}, \\
    &V_{2} = \frac{U_{2}+U_{5}+U_{8}}{U_{4}+U_{6}+U_{7}}, \\
    &Z_{1} = \frac{V_{1}}{1+V_{1}}, \\
    &Z_{2} = \frac{V_{2}}{1+V_{2}}.
\end{align*}
Then, $Z = (Z_{1},Z_{2})$ follows a bivariate beta distribution.
\citet{Crackel17} reconsidered as a $5$-parameter sub-model by restricting $\theta_{3},\theta_{4},\theta_{5} = 0$. \citet{Jiang18} used the $5$-parameter models for ABC experiments and therefore we also adopted this with the true parameter $\theta^{*} = (3,2.5,2,1.5,1)$ as a benchmark model.

We set the $n=500$ observed data and the prior on the unknown parameter $\theta = (\theta_{1}, \theta_{2},\theta_{6},\theta_{7},\theta_{8})$ as $\theta_{1}, \theta_{2},\theta_{6},\theta_{7},\theta_{8} \sim \mathrm{Uniform}[0,5]^{5}$.

\subsection{Moving-average Model of Order 2 (MA2)}
\citet{Marin12} used the moving-average model of order 2 as a benchmark model.
We adopted this model with $10$-length time series and unobserved noise error term $Z_{j}$, which follows Student's t-distribution with $5$ degrees of freedom.
Therefore, the generative process of data is
\begin{align*}
    Y_{j} = Y_{j} + \theta_{1}Y_{j-1} + \theta_{2}Y_{j-2} \ \ (j=1,2,\ldots, 10).
\end{align*}
We also assumed this model has the true parameter $\theta^{*} = (0.6,0.2)$.
We then set the $n=200$ observed data and the prior on the unknown parameter $\theta = (\theta_{1},\theta_{2})$ as $\theta_{1},\theta_{2} \sim \mathrm{Uniform}[-2,2] \times \mathrm{Uniform}[-1,1]$.

\subsection{Multivariate \texorpdfstring{$g$}{Lg}-and-\texorpdfstring{$k$}{Lg} Distribution (GK)}
The univariate $g$-and-$k$ distribution is defined by its inverse distribution function as
\begin{align*}
    F^{-1}(x) = A + B \bigg[ 1 + c \frac{1 - \exp(-gz_{x})}{1 + \exp(-gz_{x})} \bigg](1+z_{x}^{2})^{k}z_{x},
\end{align*}
where $z_{x}$ is the $x$-th quantile of the standard normal distribution, and the parameters $A,B,g,k$ are related to location, scale, skewness and kurtosis, respectively.
The hyper-parameter $c$ is conventionally chose as $c=0.8$~\cite{fearnhead12}.
As the inversion transform method can conveniently sample from this distribution by drawing $Z \sim N(0,1)$ \IID and then transforming them to be $g$-and-$k$ distributed random variables.
\citet{rayner02} mentioned that the univariate $g$-and-$k$ distribution had no analytical form of the density function, and the numerical evaluation of the likelihood function is costly.
Therefore, ABC is often used on it~\cite{Peters06,fearnhead12,Allingham09}.
Furthermore, \citet{Drovandi11} and \cite{Li15} has also considered the multivariate $g$-and-$k$ distribution.

In our experiments, we set a $5$-dimensional $g$-and-$k$ distribution.
The generative steps are as follows:
\begin{align*}
    &\textrm{Draw:} \ \ Z=(Z_{1},\ldots,Z_{5}) \sim \mathcal{N}(0,\Sigma), \\
    &\textrm{Transform:} \ Z,
\end{align*}
where $\Sigma$ is sparse matrix which has $\Sigma_{ii} = 1$ and $\Sigma_{ii} = \rho$ if $|i-j| =1$ or $0$ otherwise.
We used the transformation for $Z$ that changes marginally as the univariate $g$-and-$k$ distribution does.
We also adopted this model with the true parameters $\theta^{*} = (A^{*},B^{*},g^{*},k^{*},\rho^{*})$, where $A^{*} = 3$, $B^{*} = 1$, $g^{*} = 2$, $k^{*} = 0.5$ and $\rho^{*} = -0.3$.
We set the $n=500$ observed data and the prior on the unknown parameter $\theta = (A, B, g, k, \rho)$ as $A,B,g,k \sim \mathrm{Uniform}[0,4]$ and $\rho$ is sampled from $\mathrm{Uniform}[0,1]$ and is transformed by $2\sqrt{3}(\rho-0.5) / 3$.


\clearpage
\section{Additional Results for Experiments in Section \ref{sec:experiments}}
\label{app:full_MSE_scores}
We summarize the additional mean-squared-error (MSE) results for the experiments in Section \ref{sec:experiments}.
Furthermore, we report the simulation error results based on the energy distance.

\subsection{MSEs for All Parameters}

The following table shows the experimental results of MSEs for all parameters in the experiments of Section \ref{sec:experiments}.
From these results, our method almost outperforms the other baseline methods, especially when the observed data have heavy contamination.

\begin{table}[th]
\label{table:benchmark}
\centering
\caption{Experimental results of $8$ baseline methods for $5$ benchmark models on MSE and standard error of all parameters. We performed ABC over $10$ trials on $10$ different datasets. Lower values are better. The scores of $\gamma$-divergence estimator are picked up from the all of experimental results in Figure \ref{fig:gamma_gaussian}-\ref{fig:gamma_gk}. Bold-faces indicate the best score per contamination rate.
}
\scalebox{0.88}{
\begin{tabular}{ccccccc}
\hline
Discrepancy measure                                                  & Outlier                     & GM & MG1 & BB & MA2 & GK \\ \hline \hline
\multicolumn{1}{c|}{\multirow{3}{*}{AL (Indirect)}}                       & \multicolumn{1}{c|}{$0$\%}  &$0.350 \ (0.419)$    &$0.940 \ (0.851)$     &$0.946 \ (0.412)$    & $0.006 \ (0.004)$    &$0.155 \ (0.144)$    \\
\multicolumn{1}{c|}{}                                                & \multicolumn{1}{c|}{$10$\%} &$0.805 \ (0.669)$    &$0.556 \ (0.448)$     &$1.538 \ (0.251)$    &$1.094 \ (0.033)$     &$0.870 \ (0.275)$    \\
\multicolumn{1}{c|}{}                                                & \multicolumn{1}{c|}{$20$\%} &$0.734 \ (0.882)$    &$2.888 \ (1.222)$     &$1.557 \ (0.229)$    &$1.125 \ (0.022)$     &$1.374 \ (0.439)$    \\ \hline
\multicolumn{1}{c|}{\multirow{3}{*}{AL with Huber (Robust Indirect)}}                & \multicolumn{1}{c|}{$0$\%}  &$0.097 \ (0.261)$    &$0.734 \ (1.369)$     &$1.092 \ (0.456)$    &$0.029 \ (0.030)$     &$0.199 \ (0.116)$    \\
\multicolumn{1}{c|}{}                                                & \multicolumn{1}{c|}{$10$\%} &$0.920 \ (0.033)$    &$0.370 \ (0.369)$     &$1.948 \ (0.140)$    &$1.017 \ (0.154)$     &$1.066 \ (0.180)$    \\
\multicolumn{1}{c|}{}                                                & \multicolumn{1}{c|}{$20$\%} &$1.000 \ (0.025)$    &$0.836 \ (0.567)$     &$2.441 \ (0.700)$    &$2.275 \ (0.998)$     &$0.872 \ (0.300)$    \\ \hline
\multicolumn{1}{c|}{\multirow{3}{*}{Classification ($L_{1}$ + Logistic)}} & \multicolumn{1}{c|}{$0$\%}  &$1.324 \ (0.088)$    &$4.018 \ (0.664)$     &$1.076 \ (0.430)$    &$0.459 \ (0.410)$     &$1.076 \ (0.384)$    \\
\multicolumn{1}{c|}{}                                                & \multicolumn{1}{c|}{$10$\%} &$0.270 \ (0.242)$    &$6.422 \ (0.554)$     &$0.680 \ (0.213)$    &$0.757 \ (0.138)$     &$1.240 \ (0.290)$    \\
\multicolumn{1}{c|}{}                                                & \multicolumn{1}{c|}{$20$\%} &$0.212 \ (0.250)$    &$8.394 \ (0.051)$     &$0.709 \ (0.276)$    &$0.810 \ (0.112)$     &$1.477 \ (0.145)$    \\ \hline
\multicolumn{1}{c|}{\multirow{3}{*}{Classification (Boosting)}}      & \multicolumn{1}{c|}{$0$\%}  &$1.564 \ (0.075)$    &$0.022 \ (0.033)$     &$\mathbf{0.204 \ (0.123)}$    &$\mathbf{0.004 \ (0.002)}$     &$\mathbf{0.074 \ (0.076)}$    \\
\multicolumn{1}{c|}{}                                                & \multicolumn{1}{c|}{$10$\%} &$1.495 \ (0.218)$    &$0.005 \ (0.006)$     & $\mathbf{0.315 \ (0.334)}$   &$\mathbf{0.005 \ (0.005)}$     &$0.187 \ (0.121)$    \\
\multicolumn{1}{c|}{}                                                & \multicolumn{1}{c|}{$20$\%} &$0.639 \ (0.686)$    &$0.017 \ (0.017)$     &$0.346 \ (0.136)$    &$0.008 \ (0.007)$     &$0.179 \ (0.090)$    \\ \hline
\multicolumn{1}{c|}{\multirow{3}{*}{MMD}}                            & \multicolumn{1}{c|}{$0$\%}  &$0.054 \ (0.105)$    &$0.617 \ (0.413)$     &$0.326 \ (0.179)$    &$\mathbf{0.004 \ (0.003)}$     &$0.240 \ (0.141)$    \\
\multicolumn{1}{c|}{}                                                & \multicolumn{1}{c|}{$10$\%} &$0.760 \ (0.500)$    &$0.333 \ (0.229)$     &$0.366 \ (0.253)$    &$0.079 \ (0.024)$     &$\mathbf{0.165 \ (0.094)}$    \\
\multicolumn{1}{c|}{}                                                & \multicolumn{1}{c|}{$20$\%} &$1.342 \ (0.339)$    &$1.237 \ (0.764)$     &$0.823 \ (0.175)$    &$0.382 \ (0.054)$     &$0.559 \ (0.281)$    \\ \hline
\multicolumn{1}{c|}{\multirow{3}{*}{MONK-BCD Fast}}                  & \multicolumn{1}{c|}{$0$\%}  &$0.647 \ (0.203)$    &$0.113 \ (0.115)$     &$0.424 \ (0.205)$    &$0.049 \ (0.040)$     &$0.362 \ (0.348)$    \\
\multicolumn{1}{c|}{}                                                & \multicolumn{1}{c|}{$10$\%} &$0.719 \ (0.164)$    &$0.114 \ (0.145)$     &$0.524 \ (0.243)$    &$0.054 \ (0.060)$     &$0.326 \ (0.110)$    \\
\multicolumn{1}{c|}{}                                                & \multicolumn{1}{c|}{$20$\%} &$0.714 \ (0.211)$    &$0.160 \ (0.204)$     &$0.753 \ (0.403)$    &$0.102 \ (0.077)$     &$0.282 \ (0.139)$    \\ \hline
\multicolumn{1}{c|}{\multirow{3}{*}{$q$-Wasserstein}}                & \multicolumn{1}{c|}{$0$\%}  &$0.009 \ (0.011)$    &$0.419 \ (0.235)$     &$0.317 \ (0.210)$    &$0.009 \ (0.006)$     &$0.189 \ (0.153)$    \\
\multicolumn{1}{c|}{}                                                & \multicolumn{1}{c|}{$10$\%} &$1.349 \ (0.311)$    &$0.188 \ (0.110)$     &$1.880 \ (0.165)$    &$0.255 \ (0.051)$     &$0.305 \ (0.129)$    \\
\multicolumn{1}{c|}{}                                                & \multicolumn{1}{c|}{$20$\%} &$1.371 \ (0.296)$    &$3.384 \ (1.116)$     &$1.967 \ (0.257)$    &$0.432 \ (0.104)$     &$0.585 \ (0.252)$    \\ \hline
\multicolumn{1}{c|}{\multirow{3}{*}{KL-divergence}}                  & \multicolumn{1}{c|}{$0$\%}  &$0.005 \ (0.003)$    &$0.089 \ (0.058)$     &$0.406 \ (0.129)$    &$\mathbf{0.004 \ (0.004)}$     &$0.240 \ (0.152)$    \\
\multicolumn{1}{c|}{}                                                & \multicolumn{1}{c|}{$10$\%} &$0.007 \ (0.004)$    &$0.102 \ (0.064)$     &$0.346 \ (0.123)$    &$0.012 \ (0.006)$     &$0.377 \ (0.159)$    \\
\multicolumn{1}{c|}{}                                                & \multicolumn{1}{c|}{$20$\%} &$\mathbf{0.004 \ (0.003)}$    &$0.113 \ (0.069)$     &$\mathbf{0.270 \ (0.132)}$    &$0.051 \ (0.027)$     &$0.578 \ (0.290)$    \\ \hline
\multicolumn{1}{c|}{\multirow{3}{*}{$\gamma$-divergence (proposed)}}            & \multicolumn{1}{c|}{$0$\%}  &$\mathbf{0.002 \ (0.006)}$    &$\mathbf{0.003 \ (0.025)}$     &$0.405 \ (0.194)$    &$0.005 \ (0.008)$     &$0.260 \ (0.140)$    \\
\multicolumn{1}{c|}{}                                                & \multicolumn{1}{c|}{$10$\%} &$\mathbf{0.004 \ (0.002)}$    &$\mathbf{0.001 \ (0.025)}$     &$0.418 \ (0.150)$    &$\mathbf{0.005 \ (0.080)}$     &$0.228 \ (0.140)$    \\
\multicolumn{1}{c|}{}                                                & \multicolumn{1}{c|}{$20$\%} &$\mathbf{0.004 \ (0.002)}$    &$\mathbf{0.003 \ (0.017)}$     &$0.314 \ (0.296)$    &$\mathbf{0.004 \ (0.010)}$     &$\mathbf{0.170 \ (0.146)}$
\end{tabular}
}
\end{table}

\begin{figure}
    \centering
    \includegraphics[scale=0.25]{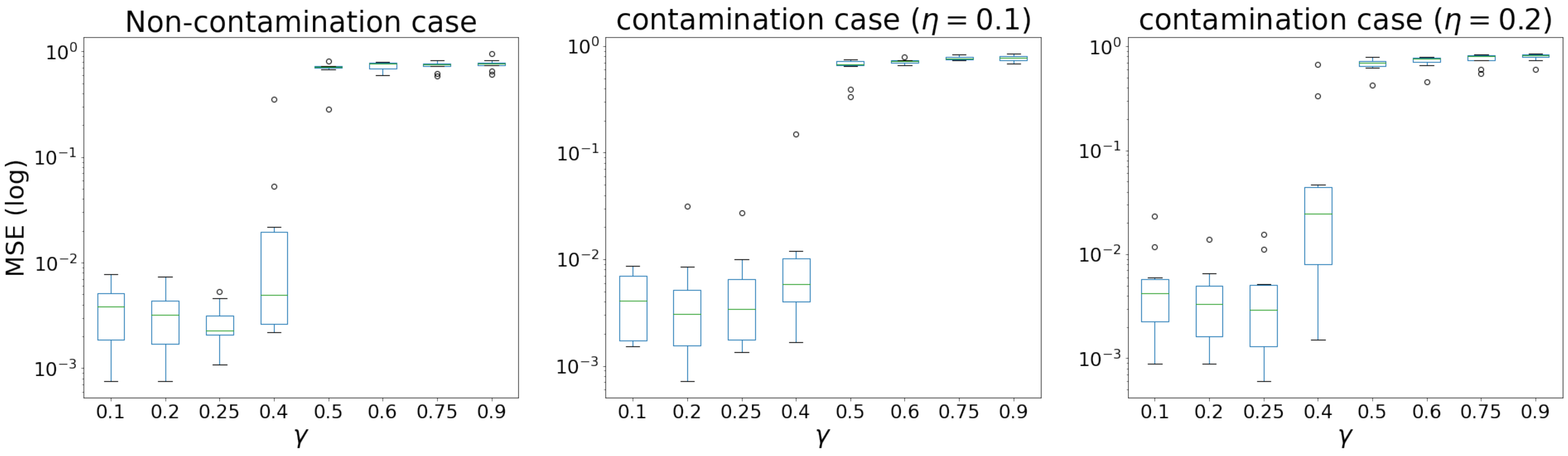}
    \caption{All of the experimental results of our method for the GM model based on MSE.}
    \label{fig:gamma_gaussian}
\end{figure}

\begin{figure}
    \centering
    \includegraphics[scale=0.25]{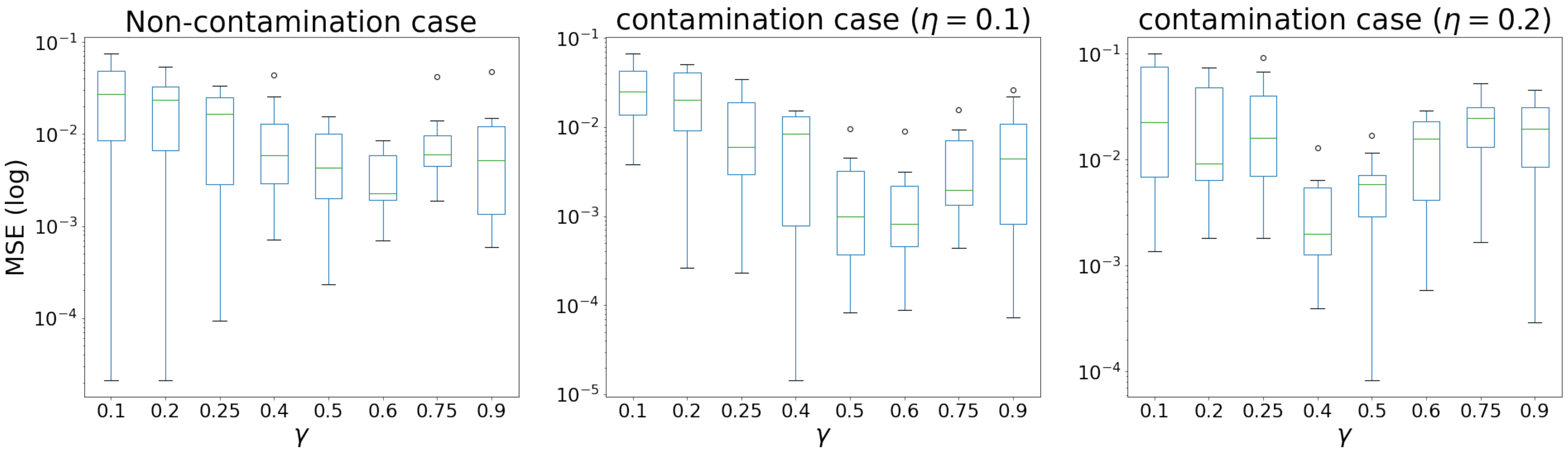}
    \caption{All of the experimental results of our method for the MG1 model based on MSE.}
    \label{fig:gamma_mg1}
\end{figure}

\begin{figure}
    \centering
    \includegraphics[scale=0.25]{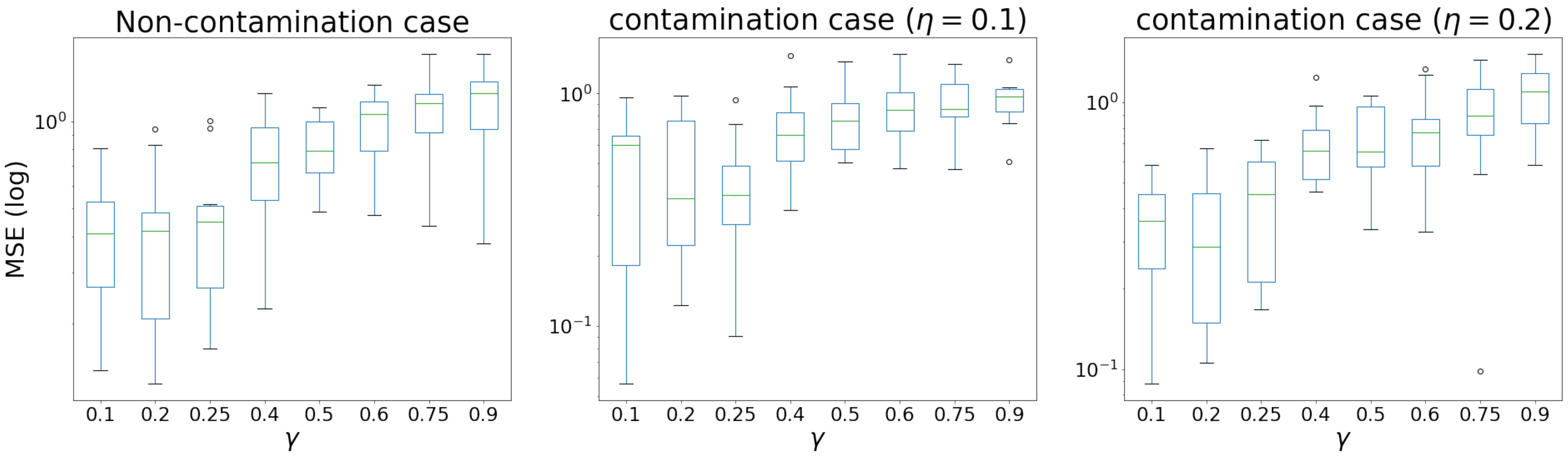}
    \caption{All of the experimental results of our method for the BB model based on MSE.}
    \label{fig:gamma_bb}
\end{figure}

\begin{figure}
    \centering
    \includegraphics[scale=0.25]{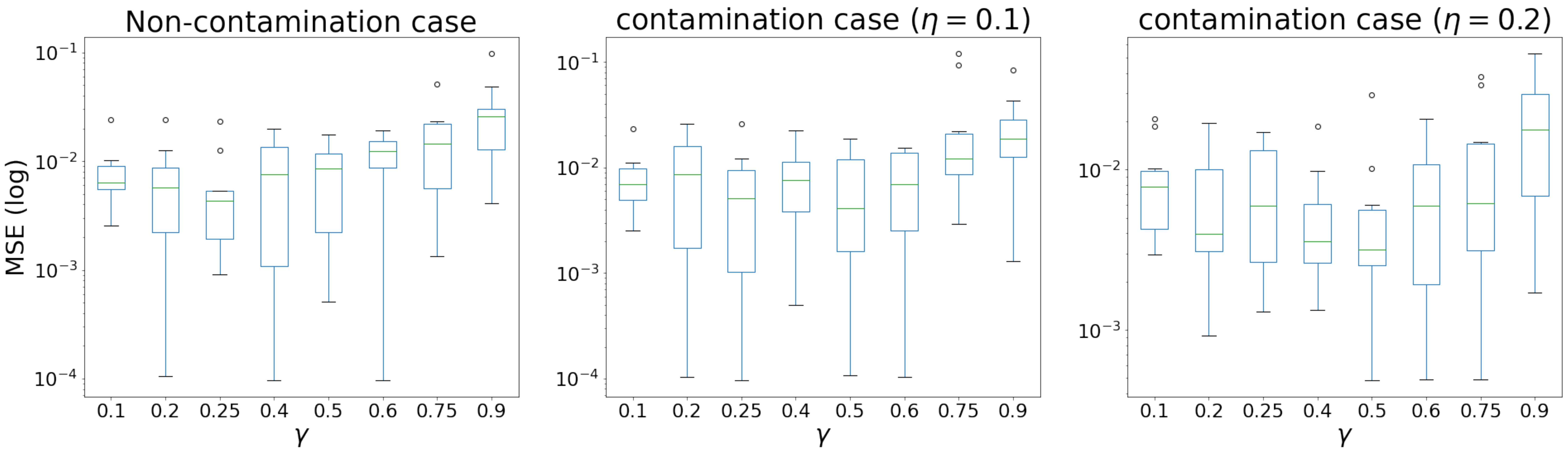}
    \caption{All of the experimental results of our method for the MA2 model based on MSE.}
    \label{fig:gamma_ma2}
\end{figure}

\begin{figure}
    \centering
    \includegraphics[scale=0.25]{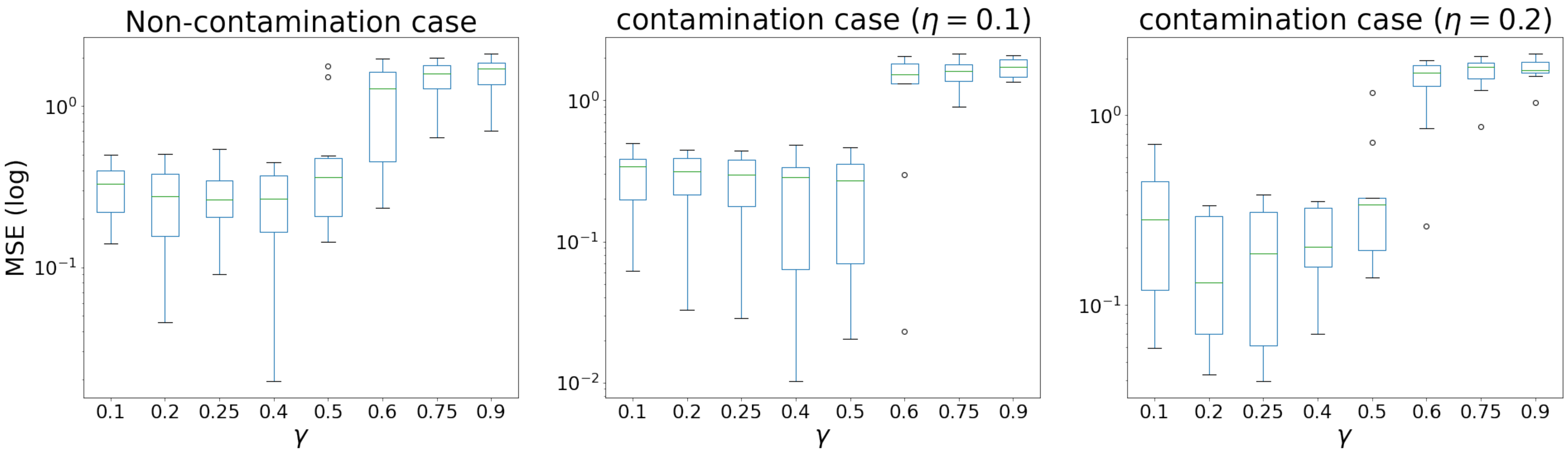}
    \caption{All of the experimental results of our method for the GK model based on MSE.}
    \label{fig:gamma_gk}
\end{figure}

\clearpage
\begin{figure}
    \centering
    \includegraphics[scale=0.29]{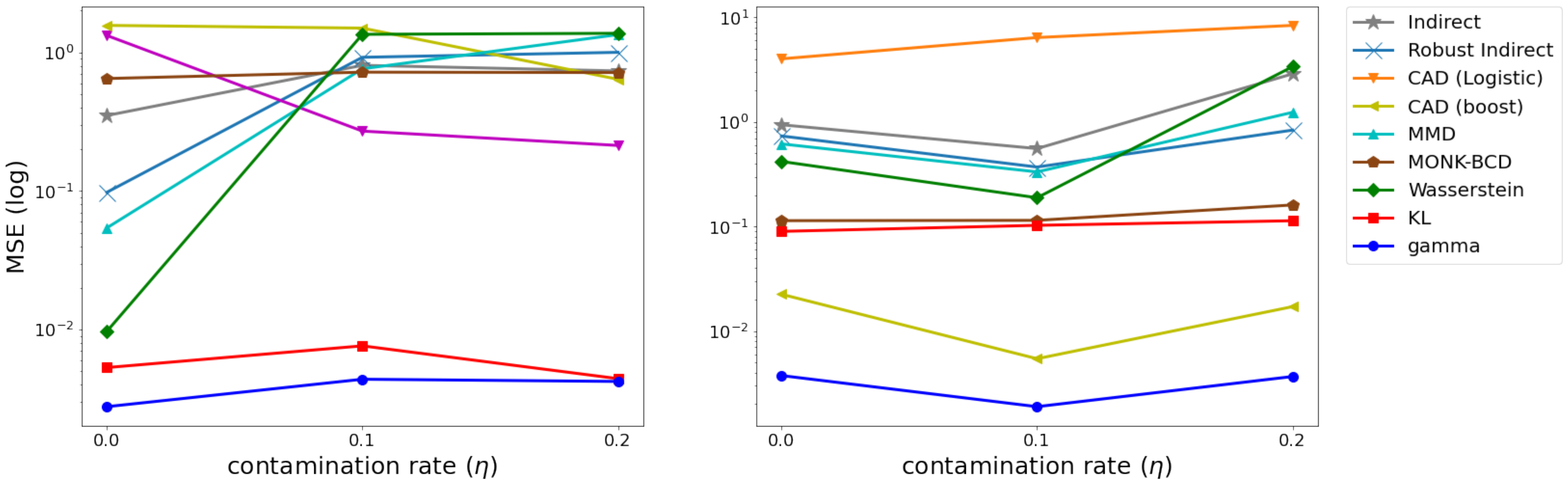}
    \caption{Experimental results for the GM and the MG1 model based on MSE.}
    \label{fig:mse_gm_mg1}
\end{figure}

\begin{figure}
    \centering
    \includegraphics[scale=0.32]{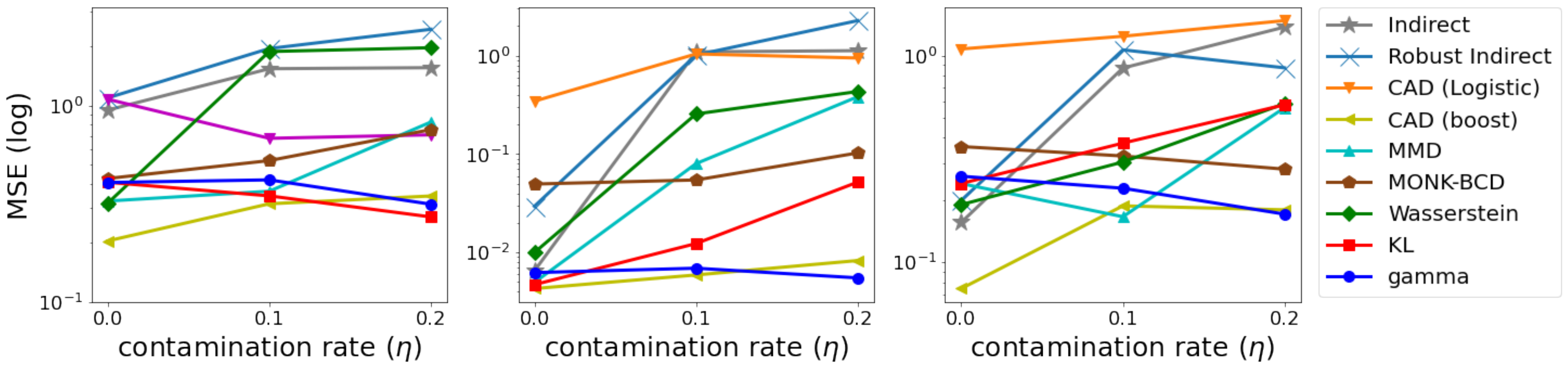}
    \caption{Experimental results for the BB, the MA2, and the GK model based on MSE.}
    \label{fig:mse_bmg}
\end{figure}

\clearpage
\subsection{MSEs for Individual Parameters and Simulation Error}
\label{app:indivi_MSE}
Here, we report the MSE results for each parameter and simulation error in all experiments in Section~\ref{sec:experiments}.

\subsubsection{Gaussian Mixture Model (GM)}
The following table shows the experimental results of MSEs for each parameter in Gaussian mixture experiments.
From these results, our method achieves almost a better performance than that of the other baseline methods, especially when the observed data have heavy contamination.

\begin{table}[th]
\label{table:benchmark_gm}
\centering
\caption{Experimental results of $8$ baseline methods for Gaussian mixture model on MSE and standard error of each parameter. We performed ABC over $10$ trials on $10$ different datasets. Lower values are better. The scores for $\gamma$-divergence estimator are picked up from the all of experimental results in Figure \ref{fig:gamma_gaussian}-\ref{fig:gamma_gk}. Bold-faces indicate the best score per contamination rate.}
\scalebox{0.88}{
\begin{tabular}{ccccccc}
\hline
Discrepancy measure                                            & Outlier                     & $p$      & $\mu_{0\{0\}}$     & $\mu_{0\{1\}}$      & $\mu_{1\{0\}}$       & $\mu_{1\{1\}}$      \\ \hline \hline
\multicolumn{1}{c|}{\multirow{3}{*}{AL (Indirect)}}            & \multicolumn{1}{c|}{$0\%$}  & $0.024 \ (0.028)$ & $0.868 \ (1.047)$ & $0.851 \ (1.027)$ & $0.003 \ (0.004)$ & $\mathbf{0.001 \ (0.001)}$ \\
\multicolumn{1}{c|}{}                                          & \multicolumn{1}{c|}{$10\%$} & $0.060 \ (0.030)$ & $0.898 \ (1.093)$ & $0.867 \ (1.046)$ & $0.912 \ (1.107)$   & $1.290 \ (1.047)$ \\
\multicolumn{1}{c|}{}                                          & \multicolumn{1}{c|}{$20\%$} & $0.044 \ (0.021)$ & $0.879 \ (1.067)$ & $0.924 \ (1.129)$ & $0.907 \ (1.100)$   & $0.915 \ (1.122)$ \\ \hline
\multicolumn{1}{c|}{\multirow{3}{*}{AL with Huber (Robust Indirect)}} & \multicolumn{1}{c|}{$0\%$}  & $0.008 \ (0.008)$ & $0.252 \ (0.697)$ & $0.223 \ (0.608)$ & $0.002 \ (0.002)$   & $\mathbf{0.001 \ (0.002)}$ \\
\multicolumn{1}{c|}{}                                          & \multicolumn{1}{c|}{$10\%$} & $0.112 \ (0.141)$ & $0.007 \ (0.006)$ & $0.006 \ (0.004)$ & $2.225 \ (0.159)$   & $2.252 \ (0.055)$ \\
\multicolumn{1}{c|}{}                                          & \multicolumn{1}{c|}{$20\%$} & $0.169 \ (0.151)$ & $0.025 \ (0.003)$ & $0.020 \ (0.006)$ & $2.399 \ (0.058)$   & $2.387 \ (0.078)$ \\ \hline
\multicolumn{1}{c|}{\multirow{3}{*}{Classification ($L_{1}$ + Logistic)}}           & \multicolumn{1}{c|}{$0\%$}  & $0.060 \ (0.023)$ & $2.249 \ (0.152)$ & $2.224 \ (0.139)$ & $1.015 \ (0.333)$   & $1.070 \ (0.351)$ \\
\multicolumn{1}{c|}{}                                          & \multicolumn{1}{c|}{$10\%$} & $0.047 \ (0.024)$ & $0.187 \ (0.545)$ & $\mathbf{0.005 \ (0.006)}$ & $0.951 \ (0.945)$   & $0.158 \ (0.468)$ \\
\multicolumn{1}{c|}{}                                          & \multicolumn{1}{c|}{$20\%$} & $0.079 \ (0.046)$ & $0.219 \ (0.609)$ & $0.014 \ (0.007)$ & $0.545 \ (0.829)$   & $0.206 \ (0.613)$ \\ \hline
\multicolumn{1}{c|}{\multirow{3}{*}{Classification (Boosting)}}           & \multicolumn{1}{c|}{$0\%$}  & $0.179 \ (0.021)$ & $2.010 \ (0.169)$ & $2.026 \ (0.128)$ & $1.802 \ (0.273)$   & $1.804 \ (0.216)$ \\
\multicolumn{1}{c|}{}                                          & \multicolumn{1}{c|}{$10\%$} & $0.162 \ (0.034)$ & $2.031 \ (0.107)$ & $1.952 \ (0.103)$ & $1.668 \ (0.581)$   & $1.663 \ (0.570)$ \\
\multicolumn{1}{c|}{}                                          & \multicolumn{1}{c|}{$20\%$} & $0.067 \ (0.039)$ & $0.955 \ (0.963)$ & $0.959 \ (0.962)$ & $0.603 \ (0.924)$   & $0.610 \ (0.934)$ \\ \hline
\multicolumn{1}{c|}{\multirow{3}{*}{MMD}}                      & \multicolumn{1}{c|}{$0\%$}  & $0.005 \ (0.004)$ & $0.247 \ (0.529)$ & $0.013 \ (0.008)$ & $\mathbf{0.001 \ (0.001)}$  & $0.002 \ (0.002)$ \\
\multicolumn{1}{c|}{}                                          & \multicolumn{1}{c|}{$10\%$} & $0.082 \ (0.061)$ & $1.657 \ (0.831)$ & $1.402 \ (0.920)$ & $0.320 \ (0.637)$   & $0.340 \ (0.675)$ \\
\multicolumn{1}{c|}{}                                          & \multicolumn{1}{c|}{$20\%$} & $0.114 \ (0.041)$ & $2.141 \ (0.116)$ & $2.130 \ (0.138)$ & $1.241 \ (0.835)$   & $1.084 \ (0.897)$ \\ \hline
\multicolumn{1}{c|}{\multirow{3}{*}{MONK-Fast}}                & \multicolumn{1}{c|}{$0\%$}  & $0.009 \ (0.005)$ & $1.592 \ (0.577)$ & $1.620 \ (0.494)$ & $0.007 \ (0.011)$   & $0.005 \ (0.008)$ \\
\multicolumn{1}{c|}{}                                          & \multicolumn{1}{c|}{$10\%$} & $0.013 \ (0.011)$ & $1.699 \ (0.227)$ & $1.689 \ (0.383)$ & $\mathbf{0.002 \ (0.002)}$   & $0.192 \ (0.562)$ \\
\multicolumn{1}{c|}{}                                          & \multicolumn{1}{c|}{$20\%$} & $0.032 \ (0.042)$ & $1.792 \ (0.268)$ & $1.547 \ (0.556)$ & $0.193 \ (0.561)$   & $0.007 \ (0.006)$ \\ \hline
\multicolumn{1}{c|}{\multirow{3}{*}{$q$-Wasserstein}}          & \multicolumn{1}{c|}{$0\%$}  & $\mathbf{0.001 \ (0.001)}$ & $0.023 \ (0.032)$ & $0.018 \ (0.029)$ & $\mathbf{0.001 \ (0.001)}$   & $0.003 \ (0.004)$ \\
\multicolumn{1}{c|}{}                                          & \multicolumn{1}{c|}{$10\%$} & $0.044 \ (0.026)$ & $0.978 \ (0.777)$ & $0.880 \ (0.859)$ & $2.411 \ (0.058)$   & $2.430 \ (0.049)$ \\
\multicolumn{1}{c|}{}                                          & \multicolumn{1}{c|}{$20\%$} & $0.018 \ (0.018)$ & $1.004 \ (0.804)$ & $0.767 \ (0.654)$ & $2.550 \ (0.050)$   & $2.518 \ (0.068)$ \\ \hline
\multicolumn{1}{c|}{\multirow{3}{*}{KL-divergence}}            & \multicolumn{1}{c|}{$0\%$}  & $0.003 \ (0.002)$ & $0.010 \ (0.018)$ & $0.003 \ (0.003)$ & $0.002 \ (0.003)$   & $0.007 \ (0.005)$ \\
\multicolumn{1}{c|}{}                                          & \multicolumn{1}{c|}{$10\%$} & $0.007 \ (0.018)$ & $0.010 \ (0.013)$ & $0.011 \ (0.010)$ & $0.004 \ (0.004)$   & $0.004 \ (0.005)$ \\
\multicolumn{1}{c|}{}                                          & \multicolumn{1}{c|}{$20\%$} & $\mathbf{0.004 \ (0.003)}$ & $0.004 \ (0.004)$ & $\mathbf{0.006 \ (0.013)}$ & $0.003 \ (0.006)$   & $\mathbf{0.002 \ (0.004)}$ \\ \hline
\multicolumn{1}{c|}{\multirow{3}{*}{$\gamma$-divergence (proposed)}}            & \multicolumn{1}{c|}{$0\%$}  & $0.003 \ (0.002)$ & $\mathbf{0.006 \ (0.007)}$ & $\mathbf{0.001 \ (0.001)}$ & $\mathbf{0.001 \ (0.001)}$   & $\mathbf{0.001 \ (0.001)}$ \\
\multicolumn{1}{c|}{}                                          & \multicolumn{1}{c|}{$10\%$} & $\mathbf{0.002 \ (0.004)}$ & $\mathbf{0.006 \ (0.007)}$ & $0.007 \ (0.008)$ & $\mathbf{0.002 \ (0.002)}$   & $\mathbf{0.003 \ (0.004)}$ \\
\multicolumn{1}{c|}{}                                          & \multicolumn{1}{c|}{$20\%$} & $\mathbf{0.001 \ (0.001)}$ & $0.005 \ (0.006)$ & $0.009 \ (0.018)$ & $\mathbf{0.002 \ (0.002)}$   & $\mathbf{0.002 \ (0.003)}$
\end{tabular}
}
\end{table}

\clearpage
\subsubsection{\textit{M}/\textit{G}/1-queueing Model (MG1)}
The following table shows the experimental results of MSEs for each parameter in \textit{M}/\textit{G}/1-queueing Model experiments.
From these results, our method achieves almost a better performance than that of the other baseline methods, especially when the observed data have heavy contamination.

\begin{table}[th]
\label{table:benchmark_MG1}
\centering
\caption{Experimental results of $8$ baseline methods for $M/G/1$-queueing model on MSE and standard error of each parameter. We performed ABC over $10$ trials on $10$ different datasets. Lower values are better. The scores for $\gamma$-divergence estimator are picked up from the all of experimental results in Figure \ref{fig:gamma_gaussian}-\ref{fig:gamma_gk}. Bold-faces indicate the best score per contamination rate.}
\begin{tabular}{ccccc}
\hline
Discrepancy measure                                                       & Outlier                     & $\theta_{1}$ & $\theta_{2}$ & $\theta_{3}$ \\ \hline \hline
\multicolumn{1}{c|}{\multirow{3}{*}{AL (Indirect)}}                       & \multicolumn{1}{c|}{$0$\%}  &$0.083 \ (0.069)$              &$2.737 \ (2.547)$              &$\mathbf{0.0001 \ (0.0002)}$              \\
\multicolumn{1}{c|}{}                                                     & \multicolumn{1}{c|}{$10$\%} &$1.008 \ (0.749)$              &$0.660 \ (0.845)$              &$0.0009 \ (0.0008)$              \\
\multicolumn{1}{c|}{}                                                     & \multicolumn{1}{c|}{$20$\%} &$4.804 \ (3.593)$              &$3.859 \ (2.911)$              &$0.003 \ (0.001)$              \\ \hline
\multicolumn{1}{c|}{\multirow{3}{*}{AL with Huber (Robust Indirect)}}     & \multicolumn{1}{c|}{$0$\%}  &$0.202 \ (0.242)$              &$2.001 \ (4.142)$              &$\mathbf{0.0001 \ (0.0002)}$              \\
\multicolumn{1}{c|}{}                                                     & \multicolumn{1}{c|}{$10$\%} &$0.998 \ (1.802)$              &$0.113 \ (0.118)$              &$0.0007 \ (0.0003)$              \\
\multicolumn{1}{c|}{}                                                     & \multicolumn{1}{c|}{$20$\%} &$1.339 \ (1.221)$              &$1.167 \ (1.119)$              &$0.001 \ (0.0007)$              \\ \hline
\multicolumn{1}{c|}{\multirow{3}{*}{Classification ($L_{1}$ + Logistic)}} & \multicolumn{1}{c|}{$0$\%}  &$0.078 \ (0.082)$              &$11.961 \ (1.990)$              &$0.016 \ (0.015)$              \\
\multicolumn{1}{c|}{}                                                     & \multicolumn{1}{c|}{$10$\%} &$0.078 \ (0.077)$              &$19.180 \ (1.692)$              &$0.009 \ (0.010)$              \\
\multicolumn{1}{c|}{}                                                     & \multicolumn{1}{c|}{$20$\%} &$0.308 \ (0.429)$              &$24.861 \ (0.414)$              &$0.013 \ (0.015)$              \\ \hline
\multicolumn{1}{c|}{\multirow{3}{*}{Classification (Boosting)}}           & \multicolumn{1}{c|}{$0$\%}  &$0.015 \ (0.022)$              &$0.051 \ (0.081)$              &$0.0002 \ (0.0002)$              \\
\multicolumn{1}{c|}{}                                                     & \multicolumn{1}{c|}{$10$\%} &$0.013 \ (0.018)$              &$0.002 \ (0.006)$              &$0.0008 \ (0.0006)$              \\
\multicolumn{1}{c|}{}                                                     & \multicolumn{1}{c|}{$20$\%} &$0.022 \ (0.038)$              &$0.027 \ (0.049)$              &$0.0009 \ (0.0004)$              \\ \hline
\multicolumn{1}{c|}{\multirow{3}{*}{MMD}}                                 & \multicolumn{1}{c|}{$0$\%}  &$0.528 \ (0.567)$              &$1.323 \ (0.873)$              &$\mathbf{0.0001 \ (>1e-6)}$              \\
\multicolumn{1}{c|}{}                                                     & \multicolumn{1}{c|}{$10$\%} &$0.630 \ (0.667)$              &$0.368 \ (0.346)$              &$0.0004 \ (0.0003)$              \\
\multicolumn{1}{c|}{}                                                     & \multicolumn{1}{c|}{$20$\%} &$0.655 \ (0.732)$              &$3.053 \ (2.477)$              &$0.002 \ (0.0004)$              \\ \hline
\multicolumn{1}{c|}{\multirow{3}{*}{MONK-BCD Fast}}                       & \multicolumn{1}{c|}{$0$\%}  &$0.019 \ (0.024)$              &$0.318 \ (0.335)$              &$0.003 \ (0.004)$              \\
\multicolumn{1}{c|}{}                                                     & \multicolumn{1}{c|}{$10$\%} &$0.043 \ (0.031)$              &$0.298 \ (0.447)$              &$0.001 \ (0.002)$              \\
\multicolumn{1}{c|}{}                                                     & \multicolumn{1}{c|}{$20$\%} &$0.182 \ (0.284)$              &$0.295 \ (0.579)$              &$0.004 \ (0.005)$              \\ \hline
\multicolumn{1}{c|}{\multirow{3}{*}{$q$-Wasserstein}}                     & \multicolumn{1}{c|}{$0$\%}  &$0.174 \ (0.221)$              &$1.082 \ (0.772)$              &$\mathbf{0.0001 \ (>1e-6)}$              \\
\multicolumn{1}{c|}{}                                                     & \multicolumn{1}{c|}{$10$\%} &$0.175 \ (0.201)$              &$0.389 \ (0.321)$              &$\mathbf{0.00009 \ (>1e-6)}$              \\
\multicolumn{1}{c|}{}                                                     & \multicolumn{1}{c|}{$20$\%} &$0.393 \ (0.547)$              &$9.758 \ (3.177)$              &$0.0008 \ (>1e-6)$              \\ \hline
\multicolumn{1}{c|}{\multirow{3}{*}{KL-divergence}}                       & \multicolumn{1}{c|}{$0$\%}  &$0.124 \ (0.186)$              &$0.145 \ (0.139)$              &$\mathbf{0.0001 \ (0.0002)}$              \\
\multicolumn{1}{c|}{}                                                     & \multicolumn{1}{c|}{$10$\%} &$0.160 \ (0.132)$              &$0.147 \ (0.142)$              &$0.0002 \ (0.0003)$              \\
\multicolumn{1}{c|}{}                                                     & \multicolumn{1}{c|}{$20$\%} &$0.249 \ (0.185)$              &$0.090 \ (0.060)$              &$0.001 \ (0.0008)$              \\ \hline
\multicolumn{1}{c|}{\multirow{3}{*}{$\gamma$-divergence}}                 & \multicolumn{1}{c|}{$0$\%}  &$\mathbf{0.009 \ (0.007)}$              &$\mathbf{>1e-5 \ (0.0001)}$              &$0.001 \ (0.002)$              \\
\multicolumn{1}{c|}{}                                                     & \multicolumn{1}{c|}{$10$\%} &$\mathbf{0.005 \ (0.007)}$              &$\mathbf{>1e-5 \ (0.0002)}$              &$0.0003 \ (0.0002)$              \\
\multicolumn{1}{c|}{}                                                     & \multicolumn{1}{c|}{$20$\%} &$\mathbf{0.008 \ (0.010)}$              &$\mathbf{0.002 \ (0.003)}$              &$\mathbf{0.0002 \ (0.0003)}$             
\end{tabular}
\end{table}

\clearpage
\subsubsection{Bivariate Beta Model (BB)}
The following table shows the experimental results of MSEs for each parameter in bivariate-beta model experiments.
From these results, our method fails to reduce the effects of outliers.
Furthermore, the KL-divergence method works well, even if the observed data are heavily contaminated.
We will investigate the reason why this phenomenon occurs as future work.
We believe this may be due to the way the contamination of the data occurs.

\begin{table}[th]
\label{table:benchmark_bb}
\centering
\caption{Experimental results of $8$ baseline methods for the Bivariate-Beta model on MSE and standard error of each parameter. We performed ABC over $10$ trials on $10$ different datasets. Lower values are better. The scores for $\gamma$-divergence estimator are picked up from the all of experimental results in Figure \ref{fig:gamma_gaussian}-\ref{fig:gamma_gk}. Bold-faces indicate the best score per contamination rate.}
\scalebox{0.88}{
\begin{tabular}{ccccccc}
\hline
Discrepancy measure                                                       & Outlier                     & $\theta_{1}$ & $\theta_{2}$ & $\theta_{6}$ & $\theta_{7}$ & $\theta_{8}$ \\ \hline \hline
\multicolumn{1}{c|}{\multirow{3}{*}{AL (Indirect)}}                       & \multicolumn{1}{c|}{$0$\%}  &$1.065 \ (0.538)$              &$1.304 \ (0.927)$              &$1.365 \ (1.228)$              &$0.823 \ (0.617)$              &$0.175 \ (0.092)$              \\
\multicolumn{1}{c|}{}                                                     & \multicolumn{1}{c|}{$10$\%} &$0.852 \ (0.645)$              &$1.713 \ (1.110)$              &$3.066 \ (0.245)$              &$1.621 \ (0.086)$              &$0.438 \ (0.156)$              \\
\multicolumn{1}{c|}{}                                                     & \multicolumn{1}{c|}{$20$\%} &$0.768 \ (0.419)$              &$2.044 \ (1.026)$              &$2.908 \ (0.159)$              &$1.618 \ (0.185)$              &$0.446 \ (0.142)$              \\ \hline
\multicolumn{1}{c|}{\multirow{3}{*}{AL with Huber (Robust Indirect)}}     & \multicolumn{1}{c|}{$0$\%}  &$0.788 \ (0.466)$              &$1.763 \ (0.770)$              &$2.038 \ (1.576)$              &$0.800 \ (0.799)$              &$0.071 \ (0.082)$              \\
\multicolumn{1}{c|}{}                                                     & \multicolumn{1}{c|}{$10$\%} &$1.917 \ (0.546)$              &$3.883 \ (0.503)$              &$2.279 \ (0.462)$              &$0.992 \ (0.266)$              &$0.668 \ (0.092)$              \\
\multicolumn{1}{c|}{}                                                     & \multicolumn{1}{c|}{$20$\%} &$2.125 \ (2.138)$              &$1.504 \ (1.028)$              &$2.028 \ (1.549)$              &$2.892 \ (3.149)$              &$3.656 \ (2.863)$              \\ \hline
\multicolumn{1}{c|}{\multirow{3}{*}{Classification ($L_{1}$ + Logistic)}} & \multicolumn{1}{c|}{$0$\%}  &$1.135 \ (0.464)$              &$1.757 \ (1.118)$              &$1.918 \ (1.291)$              &$0.412 \ (0.397)$              &$0.158 \ (0.235)$              \\
\multicolumn{1}{c|}{}                                                     & \multicolumn{1}{c|}{$10$\%} &$0.833 \ (0.668)$              &$0.848 \ (0.692)$              &$0.589 \ (0.669)$              &$0.687 \ (0.358)$              &$0.443 \ (0.143)$              \\
\multicolumn{1}{c|}{}                                                     & \multicolumn{1}{c|}{$20$\%} &$0.715 \ (0.451)$              &$1.994 \ (1.213)$              &$\mathbf{0.141 \ (0.149)}$              &$0.312 \ (0.266)$              &$0.381 \ (0.148)$              \\ \hline
\multicolumn{1}{c|}{\multirow{3}{*}{Classification (Boosting)}} & \multicolumn{1}{c|}{$0$\%}  &$\mathbf{0.309 \ (0.396)}$              &$\mathbf{0.482 \ (0.460)}$              &$0.113 \ (0.138)$              &$\mathbf{0.036 \ (0.034)}$              &$0.080 \ (0.049)$              \\
\multicolumn{1}{c|}{}                                                     & \multicolumn{1}{c|}{$10$\%} &$0.622 \ (1.000)$              &$\mathbf{0.328 \ (0.530)}$              &$0.268 \ (0.287)$              &$0.315 \ (0.251)$              &$\mathbf{0.044 \ (0.071)}$              \\
\multicolumn{1}{c|}{}                                                     & \multicolumn{1}{c|}{$20$\%} &$0.571 \ (0.461)$              &$\mathbf{0.307 \ (0.337)}$              &$0.210 \ (0.145)$              &$0.546 \ (0.470)$              &$0.095 \ (0.129)$              \\ \hline
\multicolumn{1}{c|}{\multirow{3}{*}{MMD}}                                 & \multicolumn{1}{c|}{$0$\%}  &$0.756 \ (0.593)$              &$0.668 \ (0.370)$              &$\mathbf{0.085 \ (0.094)}$              &$0.059 \ (0.073)$              &$0.061 \ (0.064)$              \\
\multicolumn{1}{c|}{}                                                     & \multicolumn{1}{c|}{$10$\%} &$0.653 \ (0.984)$              &$0.458 \ (0.527)$              &$0.245 \ (0.267)$              &$0.391 \ (0.247)$              &$0.081 \ (0.100)$              \\
\multicolumn{1}{c|}{}                                                     & \multicolumn{1}{c|}{$20$\%} &$0.774 \ (0.581)$              &$0.980 \ (0.784)$              &$1.320 \ (0.517)$              &$0.796 \ (0.431)$              &$0.246 \ (0.163)$              \\ \hline
\multicolumn{1}{c|}{\multirow{3}{*}{MONK-BCD Fast}}                       & \multicolumn{1}{c|}{$0$\%}  &$0.729 \ (0.365)$              &$0.564 \ (0.611)$              &$0.538 \ (0.896)$              &$0.220 \ (0.128)$              &$0.071 \ (0.109)$              \\
\multicolumn{1}{c|}{}                                                     & \multicolumn{1}{c|}{$10$\%} &$0.792 \ (0.638)$              &$0.931 \ (0.916)$              &$0.678 \ (0.898)$              &$\mathbf{0.138 \ (0.146)}$              & $0.079 \ (0.084)$             \\
\multicolumn{1}{c|}{}                                                     & \multicolumn{1}{c|}{$20$\%} &$0.851 \ (0.709)$              &$1.270 \ (0.950)$              &$1.189 \ (1.080)$              &$0.359 \ (0.802)$              &$0.096 \ (0.073)$              \\ \hline
\multicolumn{1}{c|}{\multirow{3}{*}{$q$-Wasserstein}}                     & \multicolumn{1}{c|}{$0$\%}  &$0.373 \ (0.414)$              &$0.635 \ (0.622)$              &$0.379 \ (0.331)$              &$0.128 \ (0.116)$              &$0.070 \ (0.106)$              \\
\multicolumn{1}{c|}{}                                                     & \multicolumn{1}{c|}{$10$\%} &$1.663 \ (0.553)$              &$3.042 \ (0.736)$              &$2.774 \ (0.320)$              &$1.364 \ (0.206)$              &$0.559 \ (0.139)$              \\
\multicolumn{1}{c|}{}                                                     & \multicolumn{1}{c|}{$20$\%} &$1.871 \ (0.322)$              &$3.255 \ (1.137)$              &$2.688 \ (0.322)$              &$1.392 \ (0.127)$              &$0.629 \ (0.079)$              \\ \hline
\multicolumn{1}{c|}{\multirow{3}{*}{KL-divergence}}                       & \multicolumn{1}{c|}{$0$\%}  &$0.794 \ (0.503)$              &$0.871 \ (0.408)$              &$0.214 \ (0.207)$              &$0.065 \ (0.064)$              &$0.086 \ (0.090)$              \\
\multicolumn{1}{c|}{}                                                     & \multicolumn{1}{c|}{$10$\%} &$\mathbf{0.323 \ (0.341)}$              &$0.911 \ (0.734)$              &$\mathbf{0.238 \ (0.323)}$              &$0.205 \ (0.206)$              &$0.055 \ (0.090)$              \\
\multicolumn{1}{c|}{}                                                     & \multicolumn{1}{c|}{$20$\%} &$0.568 \ (0.344)$              &$0.439 \ (0.383)$              &$0.222 \ (0.257)$              &$\mathbf{0.049 \ (0.050)}$              &$\mathbf{0.074 \ (0.085)}$              \\ \hline
\multicolumn{1}{c|}{\multirow{3}{*}{$\gamma$-divergence}}                 & \multicolumn{1}{c|}{$0$\%}  &$0.639 \ (0.599)$              &$1.114 \ (0.632)$              &$0.169 \ (0.255)$              &$0.051 \ (0.050)$              &$\mathbf{0.052 \ (0.101)}$              \\
\multicolumn{1}{c|}{}                                                     & \multicolumn{1}{c|}{$10$\%} &$0.897 \ (0.500)$              &$0.551 \ (0.581)$              &$0.377 \ (0.514)$              &$0.162 \ (0.205)$              &$0.102 \ (0.133)$              \\
\multicolumn{1}{c|}{}                                                     & \multicolumn{1}{c|}{$20$\%} &$\mathbf{0.350 \ (0.356)}$              &$0.689 \ (0.552)$              &$0.359 \ (0.314)$              &$0.096 \ (0.114)$              &$\mathbf{0.074 \ (0.082)}$             
\end{tabular}
}
\end{table}

\clearpage
\subsubsection{Moving-average Model of Order 2 (MA2)}
The following table shows the experimental results of MSEs for each parameter in the Moving-average Model of Order 2 experiments.
From these results, our method achieves almost a better performance than that of the other baseline methods, especially when the observed data have heavy contamination.

\begin{table}[th]
\label{tab:ma2_mse_indivi}
\centering
\caption{Experimental results of $8$ baseline methods for the Moving-average model of order $2$ on MSE and standard error of each parameter. We performed ABC over $10$ trials in $10$ different datasets. Lower values are better. The scores for $\gamma$-divergence estimator are picked up the best score from all of the experimental results in Figure \ref{fig:gamma_gaussian}-\ref{fig:gamma_gk}. Bold-faces indicate the best score per contamination rate.}
\begin{tabular}{cccc}
\hline
Discrepancy measure                                                  & Outlier                     & $\theta_{1}$ & $\theta_{2}$ \\ \hline \hline
\multicolumn{1}{c|}{\multirow{3}{*}{Indirect}}                       & \multicolumn{1}{c|}{$0$\%}  &$0.008 \ (0.008)$              &$0.004 \ (0.002)$              \\
\multicolumn{1}{c|}{}                                                & \multicolumn{1}{c|}{$10$\%} & $1.679 \ (0.060)$             &$0.508 \ (0.029)$              \\
\multicolumn{1}{c|}{}                                                & \multicolumn{1}{c|}{$20$\%} &$1.737 \ (0.047)$              &$0.514 \ (0.018)$              \\ \hline
\multicolumn{1}{c|}{\multirow{3}{*}{Robust Indirect}}                & \multicolumn{1}{c|}{$0$\%}  &$0.035 \ (0.032)$              &$0.023 \ (0.030)$              \\
\multicolumn{1}{c|}{}                                                & \multicolumn{1}{c|}{$10$\%} &$1.563 \ (0.100)$              & $0.470 \ (0.270)$             \\
\multicolumn{1}{c|}{}                                                & \multicolumn{1}{c|}{$20$\%} &$4.251 \ (1.966)$              & $0.299 \ (0.189)$             \\ \hline
\multicolumn{1}{c|}{\multirow{3}{*}{Classification (L1 + Logistic)}} & \multicolumn{1}{c|}{$0$\%}  & $0.775 \ (0.772)$             &$0.143 \ (0.119)$              \\
\multicolumn{1}{c|}{}                                                & \multicolumn{1}{c|}{$10$\%} &$1.023 \ (0.127)$              &$0.491 \ (0.271)$              \\
\multicolumn{1}{c|}{}                                                & \multicolumn{1}{c|}{$20$\%} &$1.395 \ (0.134)$              &$0.226 \ (0.151)$              \\ \hline
\multicolumn{1}{c|}{\multirow{3}{*}{Classification (Boosting)}}      & \multicolumn{1}{c|}{$0$\%}  & $0.004 \ (0.002)$             & $0.004 \ (0.003)$             \\
\multicolumn{1}{c|}{}                                                & \multicolumn{1}{c|}{$10$\%} & $\mathbf{0.004 \ (0.005)}$             &$0.007 \ (0.008)$              \\
\multicolumn{1}{c|}{}                                                & \multicolumn{1}{c|}{$20$\%} &$\mathbf{0.006 \ (0.006)}$              &$0.009 \ (0.015)$              \\ \hline
\multicolumn{1}{c|}{\multirow{3}{*}{MMD}}                            & \multicolumn{1}{c|}{$0$\%}  &$0.006 \ (0.006)$              &$\mathbf{0.002 \ (0.002)}$              \\
\multicolumn{1}{c|}{}                                                & \multicolumn{1}{c|}{$10$\%} &$0.121 \ (0.025)$              &$0.038 \ (0.036)$              \\
\multicolumn{1}{c|}{}                                                & \multicolumn{1}{c|}{$20$\%} &$0.547 \ (0.089)$              &$0.218 \ (0.063)$             \\ \hline
\multicolumn{1}{c|}{\multirow{3}{*}{MONK-BCD Fast}}                  & \multicolumn{1}{c|}{$0$\%}  &$0.063 \ (0.064)$              &$0.035 \ (0.047)$              \\
\multicolumn{1}{c|}{}                                                & \multicolumn{1}{c|}{$10$\%} &$0.086 \ (0.110)$              &$0.022 \ (0.032)$              \\
\multicolumn{1}{c|}{}                                                & \multicolumn{1}{c|}{$20$\%} &$0.170 \ (0.151)$              & $0.034 \ (0.028)$             \\ \hline
\multicolumn{1}{c|}{\multirow{3}{*}{$q$-Wasserstein}}                & \multicolumn{1}{c|}{$0$\%}  &$0.017 \ (0.013)$              & $\mathbf{0.002 \ (0.004)}$             \\
\multicolumn{1}{c|}{}                                                & \multicolumn{1}{c|}{$10$\%} &$0.153 \ (0.050)$              &$0.357 \ (0.088)$             \\
\multicolumn{1}{c|}{}                                                & \multicolumn{1}{c|}{$20$\%} &$0.423 \ (0.134)$              &$0.442 \ (0.102)$              \\ \hline
\multicolumn{1}{c|}{\multirow{3}{*}{KL-divergence}}                  & \multicolumn{1}{c|}{$0$\%}  &$0.004 \ (0.005)$              &$0.004 \ (0.004)$              \\
\multicolumn{1}{c|}{}                                                & \multicolumn{1}{c|}{$10$\%} &$0.007 \ (0.008)$              &$0.016 \ (0.007)$              \\
\multicolumn{1}{c|}{}                                                & \multicolumn{1}{c|}{$20$\%} & $0.045 \ (0.025)$             &$0.058 \ (0.034)$              \\ \hline
\multicolumn{1}{c|}{\multirow{3}{*}{$\gamma$-divergence (proposed)}}            & \multicolumn{1}{c|}{$0$\%}  &$\mathbf{0.003 \ (0.005)}$              &$0.008 \ (0.009)$              \\
\multicolumn{1}{c|}{}                                                & \multicolumn{1}{c|}{$10$\%} &$0.008 \ (0.006)$              & $\mathbf{0.002 \ (0.002)}$             \\
\multicolumn{1}{c|}{}                                                & \multicolumn{1}{c|}{$20$\%} &$\mathbf{0.006 \ (0.005)}$              & $\mathbf{0.003 \ (0.003)}$            
\end{tabular}
\end{table}

\clearpage
\subsubsection{Multivariate \texorpdfstring{$g$}{Lg}-and-\texorpdfstring{$k$}{Lg} Distribution (GK)}
The following table shows the experimental results of MSEs for each parameter in Multivariate \texorpdfstring{$g$}{Lg}-and-\texorpdfstring{$k$}{Lg} Distribution model experiments.
From these results, our method achieves almost a better performance than that of the other baseline methods, especially when the observed data have heavy contamination.

\begin{table}[th]
\label{tab:GK_mse_indivi}
\centering
\caption{Experimental results of $8$ baseline methods for the Multivariate $g$-and-$k$ distribution model on MSE and standard error of each parameter. We performed ABC over $10$ trials on $10$ different datasets. Lower values are better. The scores for $\gamma$-divergence estimator are picked up the best score from all of the experimental results in Figure \ref{fig:gamma_gaussian}-\ref{fig:gamma_gk}. Bold-faces indicate the best score per contamination rate.}
\scalebox{0.88}{
\begin{tabular}{ccccccc}
\hline
Discrepancy measure                                                       & Outlier                     & $A$ & $B$ & $g$ & $k$ & $\rho$ \\ \hline \hline
\multicolumn{1}{c|}{\multirow{3}{*}{AL (Indirect)}}                       & \multicolumn{1}{c|}{$0$\%}  &$0.080 \ (0.116)$     &$0.119 \ (0.116)$     &$0.505 \ (0.697)$     &$0.063 \ (0.030)$     &$0.009 \ (0.013)$        \\
\multicolumn{1}{c|}{}                                                     & \multicolumn{1}{c|}{$10$\%} &$0.294 \ (0.537)$     &$3.135 \ (1.220)$     &$0.796 \ (0.584)$     &$0.088 \ (0.026)$     &$0.039 \ (0.007)$        \\
\multicolumn{1}{c|}{}                                                     & \multicolumn{1}{c|}{$20$\%} &$1.209 \ (1.668)$     &$4.985 \ (1.142)$     &$0.600 \ (0.569)$     &$0.039 \ (0.036)$     &$0.039 \ (0.005)$        \\ \hline
\multicolumn{1}{c|}{\multirow{3}{*}{AL with Huber (Robust Indirect)}}     & \multicolumn{1}{c|}{$0$\%}  &$0.052 \ (0.053)$     &$0.151 \ (0.172)$     &$0.763 \ (0.534)$     &$\mathbf{0.020 \ (0.016)}$     &$0.008 \ (0.010)$        \\
\multicolumn{1}{c|}{}                                                     & \multicolumn{1}{c|}{$10$\%} &$0.150 \ (0.099)$     &$4.531 \ (0.899)$     &$0.606 \ (0.556)$     &$\mathbf{0.003 \ (0.006)}$     &$0.039 \ (0.003)$        \\
\multicolumn{1}{c|}{}                                                     & \multicolumn{1}{c|}{$20$\%} &$0.248 \ (0.121)$     &$3.439 \ (1.686)$     &$0.546 \ (0.399)$     &$0.110 \ (0.074)$     &$0.017 \ (0.004)$        \\ \hline
\multicolumn{1}{c|}{\multirow{3}{*}{Classification ($L_{1}$ + Logistic)}} & \multicolumn{1}{c|}{$0$\%}  &$0.109 \ (0.045)$     &$0.340 \ (0.083)$     &$1.732 \ (0.739)$     &$2.910 \ (1.808)$     &$0.290 \ (0.232)$        \\
\multicolumn{1}{c|}{}                                                     & \multicolumn{1}{c|}{$10$\%} &$0.397 \ (0.131)$     &$3.217 \ (1.354)$     &$2.362 \ (0.215)$     &$0.016 \ (0.012)$     &$0.209 \ (0.144)$        \\
\multicolumn{1}{c|}{}                                                     & \multicolumn{1}{c|}{$20$\%} &$0.201 \ (0.113)$     &$5.401 \ (0.802)$     &$1.632 \ (0.591)$     &$0.019 \ (0.023)$     &$0.130 \ (0.089)$        \\ \hline
\multicolumn{1}{c|}{\multirow{3}{*}{Classification (Boosting)}}           & \multicolumn{1}{c|}{$0$\%}  &$0.009 \ (0.009)$     &$\mathbf{0.016 \ (0.014)}$     &$\mathbf{0.317 \ (0.382)}$     &$\mathbf{0.020 \ (0.022)}$     &$0.008 \ (0.004)$        \\
\multicolumn{1}{c|}{}                                                     & \multicolumn{1}{c|}{$10$\%} &$0.024 \ (0.028)$     &$0.285 \ (0.279)$     &$0.588 \ (0.511)$     &$0.020 \ (0.017)$     &$0.017 \ (0.008)$        \\
\multicolumn{1}{c|}{}                                                     & \multicolumn{1}{c|}{$20$\%} &$0.035 \ (0.036)$     &$0.377 \ (0.312)$     &$\mathbf{0.447 \ (0.448)}$     &$0.014 \ (0.018)$     &$0.020 \ (0.005)$        \\ \hline
\multicolumn{1}{c|}{\multirow{3}{*}{MMD}}                                 & \multicolumn{1}{c|}{$0$\%}  &$0.021 \ (0.019)$     &$0.130 \ (0.128)$     &$0.958 \ (0.711)$     &$0.063 \ (0.122)$     &$0.026 \ (0.053)$        \\
\multicolumn{1}{c|}{}                                                     & \multicolumn{1}{c|}{$10$\%} &$0.054 \ (0.028)$     &$0.190 \ (0.196)$     &$\mathbf{0.526 \ (0.441)}$     &$0.040 \ (0.041)$     &$0.018 \ (0.005)$        \\
\multicolumn{1}{c|}{}                                                     & \multicolumn{1}{c|}{$20$\%} &$0.299 \ (0.166)$     &$1.729 \ (1.117)$     &$0.714 \ (0.483)$     &$0.021 \ (0.024)$     &$0.033 \ (0.006)$        \\ \hline
\multicolumn{1}{c|}{\multirow{3}{*}{MONK-BCD Fast}}                       & \multicolumn{1}{c|}{$0$\%}  &$0.009 \ (0.011)$     &$0.071 \ (0.146)$     &$0.593 \ (0.457)$     &$1.063 \ (1.919)$     &$0.076 \ (0.143)$        \\
\multicolumn{1}{c|}{}                                                     & \multicolumn{1}{c|}{$10$\%} &$\mathbf{0.009 \ (0.008)}$     &$0.114 \ (0.160)$     &$1.175 \ (0.402)$     &$0.316 \ (0.253)$     &$0.018 \ (0.025)$        \\
\multicolumn{1}{c|}{}                                                     & \multicolumn{1}{c|}{$20$\%} &$0.016 \ (0.013)$     &$0.195 \ (0.494)$     &$0.842 \ (0.526)$     &$0.222 \ (0.237)$     &$0.133 \ (0.162)$        \\ \hline
\multicolumn{1}{c|}{\multirow{3}{*}{$q$-Wasserstein}}                     & \multicolumn{1}{c|}{$0$\%}  &$0.028 \ (0.037)$     &$0.025 \ (0.022)$     &$0.859 \ (0.769)$     &$0.028 \ (0.030)$     &$\mathbf{0.006 \ (0.010)}$        \\
\multicolumn{1}{c|}{}                                                     & \multicolumn{1}{c|}{$10$\%} &$0.190 \ (0.156)$     &$0.502 \ (0.414)$     &$0.722 \ (0.675)$     &$0.087 \ (0.035)$     &$0.023 \ (0.010)$        \\
\multicolumn{1}{c|}{}                                                     & \multicolumn{1}{c|}{$20$\%} &$0.530 \ (0.133)$     &$1.474 \ (0.769)$     &$0.790 \ (0.827)$     &$0.109 \ (0.038)$     &$0.022 \ (0.007)$        \\ \hline
\multicolumn{1}{c|}{\multirow{3}{*}{KL-divergence}}                       & \multicolumn{1}{c|}{$0$\%}  &$\mathbf{0.007 \ (0.006)}$     &$0.042 \ (0.040)$     &$1.103 \ (0.752)$     &$0.040 \ (0.032)$     &$\mathbf{0.006 \ (0.006)}$        \\
\multicolumn{1}{c|}{}                                                     & \multicolumn{1}{c|}{$10$\%} &$0.015 \ (0.022)$     &$0.149 \ (0.348)$     &$1.663 \ (0.545)$     &$0.038 \ (0.028)$     &$0.018 \ (0.010)$        \\
\multicolumn{1}{c|}{}                                                     & \multicolumn{1}{c|}{$20$\%} &$0.066 \ (0.069)$     &$0.993 \ (1.128)$     &$1.766 \ (0.732)$     &$0.030 \ (0.024)$     &$0.033 \ (0.004)$        \\ \hline
\multicolumn{1}{c|}{\multirow{3}{*}{$\gamma$-divergence}}                 & \multicolumn{1}{c|}{$0$\%}  &$0.046 \ (0.016)$     &$0.065 \ (0.038)$     &$1.105 \ (0.591)$     &$0.080 \ (0.135)$     &$\mathbf{0.006 \ (0.005)}$        \\
\multicolumn{1}{c|}{}                                                     & \multicolumn{1}{c|}{$10$\%} &$0.033 \ (0.014)$     &$\mathbf{0.041 \ (0.039)}$     &$1.028 \ (0.757)$     &$0.030 \ (0.029)$     &$\mathbf{0.007 \ (0.006)}$        \\
\multicolumn{1}{c|}{}                                                     & \multicolumn{1}{c|}{$20$\%} &$\mathbf{0.008 \ (0.008)}$     &$\mathbf{0.020 \ (0.016)}$     &$0.809 \ (0.575)$     &$\mathbf{0.007 \ (0.008)}$     &$\mathbf{0.009 \ (0.004)}$       
\end{tabular}
}
\end{table}

\clearpage
\subsubsection{All of Simulation Error}
\label{app:simulation_error}
The following table shows the experimental results of simulation errors (energy distance) in the experiments of Section \ref{sec:experiments}.
From these results, our method also outperforms the other baseline methods, especially when the observed data have heavy contamination.

\begin{table}[th]
\label{table:benchmark_simulation}
\centering
\caption{Experimental results of $8$ baseline methods for $5$ benchmark models on simulation error (energy distance) and its standard error. We performed ABC over $10$ trials on $10$ different datasets. Lower values are better. The scores of $\gamma$-divergence estimator are picked up from the all of experimental results in Figure \ref{fig:gamma_gm_energy}-\ref{fig:gamma_gk_energy}. Bold-faces indicate the best score per contamination rate.
}
\scalebox{0.88}{
\begin{tabular}{ccccccc}
\hline
Discrepancy measure                                                  & Outlier                     & GM & MG1 & BB & MA2 & GK \\ \hline \hline
\multicolumn{1}{c|}{\multirow{3}{*}{AL (Indirect)}}                       & \multicolumn{1}{c|}{$0$\%}  &$0.199 \ (0.169)$    &$0.270 \ (0.114)$     &$0.070 \ (0.030)$    & $0.056 \ (0.017)$    &$0.260 \ (0.093)$    \\
\multicolumn{1}{c|}{}                                                & \multicolumn{1}{c|}{$10$\%} &$0.349 \ (0.244)$    &$0.408 \ (0.182)$     &$0.254 \ (0.034)$    &$0.475 \ (0.014)$     &$0.460 \ (0.211)$    \\
\multicolumn{1}{c|}{}                                                & \multicolumn{1}{c|}{$20$\%} &$0.263 \ (0.082)$    &$0.906 \ (0.225)$     &$0.251 \ (0.026)$    &$0.501 \ (0.022)$     &$0.724 \ (0.351)$    \\ \hline
\multicolumn{1}{c|}{\multirow{3}{*}{AL with Huber (Robust Indirect)}}                & \multicolumn{1}{c|}{$0$\%}  &$0.200 \ (0.221)$    &$0.223 \ (0.090)$     &$0.063 \ (0.019)$    &$0.064 \ (0.015)$     &$0.300 \ (0.144)$    \\
\multicolumn{1}{c|}{}                                                & \multicolumn{1}{c|}{$10$\%} &$0.966 \ (0.040)$    &$0.345 \ (0.101)$     &$0.300 \ (0.018)$    &$0.466 \ (0.033)$     &$0.470 \ (0.038)$    \\
\multicolumn{1}{c|}{}                                                & \multicolumn{1}{c|}{$20$\%} &$1.005 \ (0.043)$    &$0.509 \ (0.168)$     &$0.232 \ (0.061)$    &$0.403 \ (0.028)$     &$0.689 \ (0.124)$    \\ \hline
\multicolumn{1}{c|}{\multirow{3}{*}{Classification ($L_{1}$ + Logistic)}} & \multicolumn{1}{c|}{$0$\%}  &$0.157 \ (0.027)$    &$0.453 \ (0.019)$     &$0.066 \ (0.021)$    &$0.180 \ (0.065)$     &$0.422 \ (0.086)$    \\
\multicolumn{1}{c|}{}                                                & \multicolumn{1}{c|}{$10$\%} &$0.434 \ (0.148)$    &$0.605 \ (0.049)$     &$0.132 \ (0.039)$    &$0.354 \ (0.028)$     &$0.678 \ (0.124)$    \\
\multicolumn{1}{c|}{}                                                & \multicolumn{1}{c|}{$20$\%} &$0.443 \ (0.125)$    &$0.779 \ (0.079)$     &$0.142 \ (0.028)$    &$0.413 \ (0.040)$     &$0.873 \ (0.058)$    \\ \hline
\multicolumn{1}{c|}{\multirow{3}{*}{Classification (Boosting)}}      & \multicolumn{1}{c|}{$0$\%}  &$0.112 \ (0.006)$    &$0.169 \ (0.104)$     &$0.042 \ (0.030)$    &$\mathbf{0.048 \ (0.018)}$     &$\mathbf{0.138 \ (0.048)}$    \\
\multicolumn{1}{c|}{}                                                & \multicolumn{1}{c|}{$10$\%} &$0.150 \ (0.062)$    &$0.359 \ (0.130)$     & $0.049 \ (0.023)$   &$\mathbf{0.052 \ (0.015)}$     &$0.193 \ (0.068)$    \\
\multicolumn{1}{c|}{}                                                & \multicolumn{1}{c|}{$20$\%} &$0.273 \ (0.100)$    &$0.293 \ (0.106)$     &$0.052 \ (0.024)$    &$0.062 \ (0.024)$     &$0.181 \ (0.069)$    \\ \hline
\multicolumn{1}{c|}{\multirow{3}{*}{MMD}}                            & \multicolumn{1}{c|}{$0$\%}  &$0.059 \ (0.042)$    &$0.233 \ (0.075)$     &$0.055 \ (0.026)$    &$0.055 \ (0.024)$     &$0.231 \ (0.076)$    \\
\multicolumn{1}{c|}{}                                                & \multicolumn{1}{c|}{$10$\%} &$0.249 \ (0.120)$    &$0.275 \ (0.098)$     & $0.048 \ (0.023)$   &$0.121 \ (0.026)$     &$0.317 \ (0.093)$    \\
\multicolumn{1}{c|}{}                                                & \multicolumn{1}{c|}{$20$\%} &$0.229 \ (0.119)$    &$0.593 \ (0.084)$     &$0.070 \ (0.025)$    &$0.262 \ (0.031)$     &$0.419 \ (0.116)$    \\ \hline
\multicolumn{1}{c|}{\multirow{3}{*}{MONK-BCD Fast}}                  & \multicolumn{1}{c|}{$0$\%}  &$0.283 \ (0.069)$    &$0.312 \ (0.130)$     &$0.049 \ (0.030)$    &$0.066 \ (0.017)$     &$0.393 \ (0.402)$    \\
\multicolumn{1}{c|}{}                                                & \multicolumn{1}{c|}{$10$\%} &$0.279 \ (0.086)$    &$0.266 \ (0.189)$     &$0.069 \ (0.041)$    &$0.073 \ (0.032)$     &$0.295 \ (0.156)$    \\
\multicolumn{1}{c|}{}                                                & \multicolumn{1}{c|}{$20$\%} &$0.262 \ (0.091)$    &$0.362 \ (0.251)$     &$0.075 \ (0.040)$    &$0.094 \ (0.031)$     &$0.245 \ (0.174)$    \\ \hline
\multicolumn{1}{c|}{\multirow{3}{*}{$q$-Wasserstein}}                & \multicolumn{1}{c|}{$0$\%}  &$\mathbf{0.051 \ (0.021)}$    &$0.200 \ (0.060)$     &$\mathbf{0.037 \ (0.018)}$    &$0.066 \ (0.027)$     &$0.215 \ (0.082)$    \\
\multicolumn{1}{c|}{}                                                & \multicolumn{1}{c|}{$10$\%} &$0.671 \ (0.289)$    &$0.175 \ (0.041)$     &$0.298 \ (0.020)$    &$0.238 \ (0.035)$     &$0.344 \ (0.088)$    \\
\multicolumn{1}{c|}{}                                                & \multicolumn{1}{c|}{$20$\%} &$0.755 \ (0.213)$    &$0.599 \ (0.098)$     &$0.307 \ (0.018)$    &$0.320 \ (0.051)$     &$0.587 \ (0.067)$    \\ \hline
\multicolumn{1}{c|}{\multirow{3}{*}{KL-divergence}}                  & \multicolumn{1}{c|}{$0$\%}  &$0.066 \ (0.024)$    &$0.125 \ (0.050)$     &$0.055 \ (0.024)$    &$0.064 \ (0.016)$     &$0.198 \ (0.074)$    \\
\multicolumn{1}{c|}{}                                                & \multicolumn{1}{c|}{$10$\%} &$0.098 \ (0.079)$    &$0.178 \ (0.094)$     &$\mathbf{0.041 \ (0.015)}$    &$0.073 \ (0.027)$     &$0.155 \ (0.081)$    \\
\multicolumn{1}{c|}{}                                                & \multicolumn{1}{c|}{$20$\%} &$0.085 \ (0.044)$    &$0.322 \ (0.123)$     &$\mathbf{0.038 \ (0.025)}$    &$0.117 \ (0.022)$     &$0.271 \ (0.100)$    \\ \hline
\multicolumn{1}{c|}{\multirow{3}{*}{$\gamma$-divergence (proposed)}}            & \multicolumn{1}{c|}{$0$\%}  &$0.060 \ (0.028)$    &$\mathbf{0.096 \ (0.042)}$     &$0.066 \ (0.026)$    &$0.049 \ (0.018)$     &$0.195 \ (0.030)$    \\
\multicolumn{1}{c|}{}                                                & \multicolumn{1}{c|}{$10$\%} &$\mathbf{0.076 \ (0.044)}$    &$\mathbf{0.099 \ (0.041)}$     &$0.048 \ (0.022)$    &$0.055 \ (0.018)$     &$\mathbf{0.138 \ (0.047)}$    \\
\multicolumn{1}{c|}{}                                                & \multicolumn{1}{c|}{$20$\%} &$\mathbf{0.060 \ (0.018)}$    &$\mathbf{0.121 \ (0.085)}$     &$0.043 \ (0.019)$    &$\mathbf{0.060 \ (0.017)}$     &$\mathbf{0.140 \ (0.439)}$
\end{tabular}
}
\end{table}

\begin{figure}
    \centering
    \includegraphics[scale=0.25]{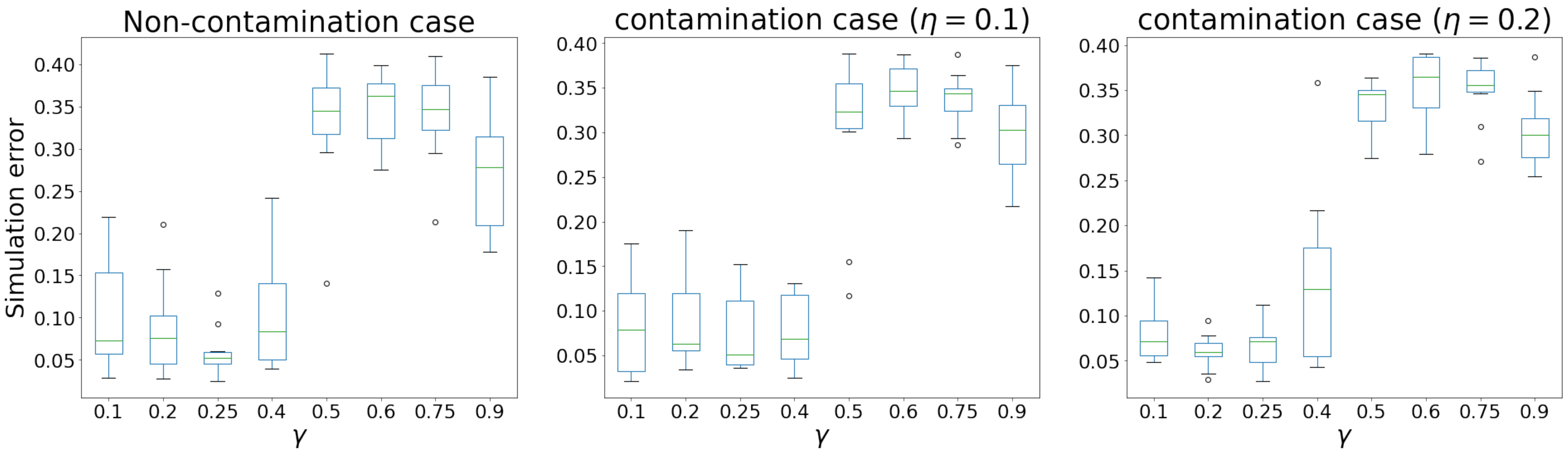}
    \caption{All of the experimental results of our method for the GM model based on simulation error.}
    \label{fig:gamma_gm_energy}
\end{figure}

\begin{figure}
    \centering
    \includegraphics[scale=0.25]{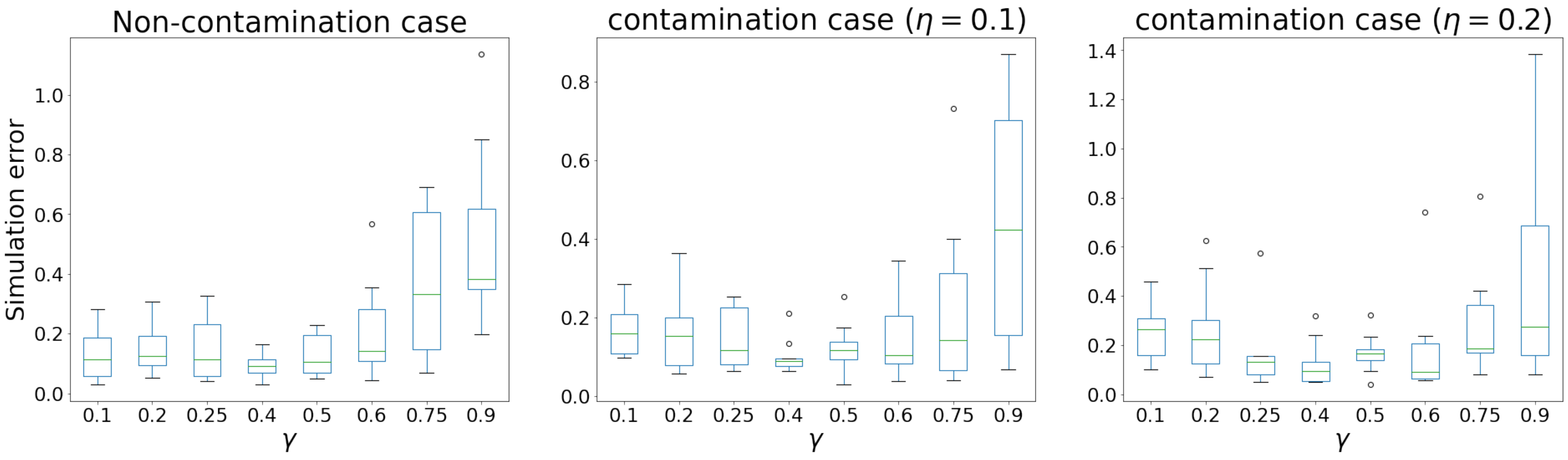}
    \caption{All of the experimental results of our method for the MG1 model based on simulation error.}
    \label{fig:gamma_mg1_energy}
\end{figure}

\begin{figure}
    \centering
    \includegraphics[scale=0.25]{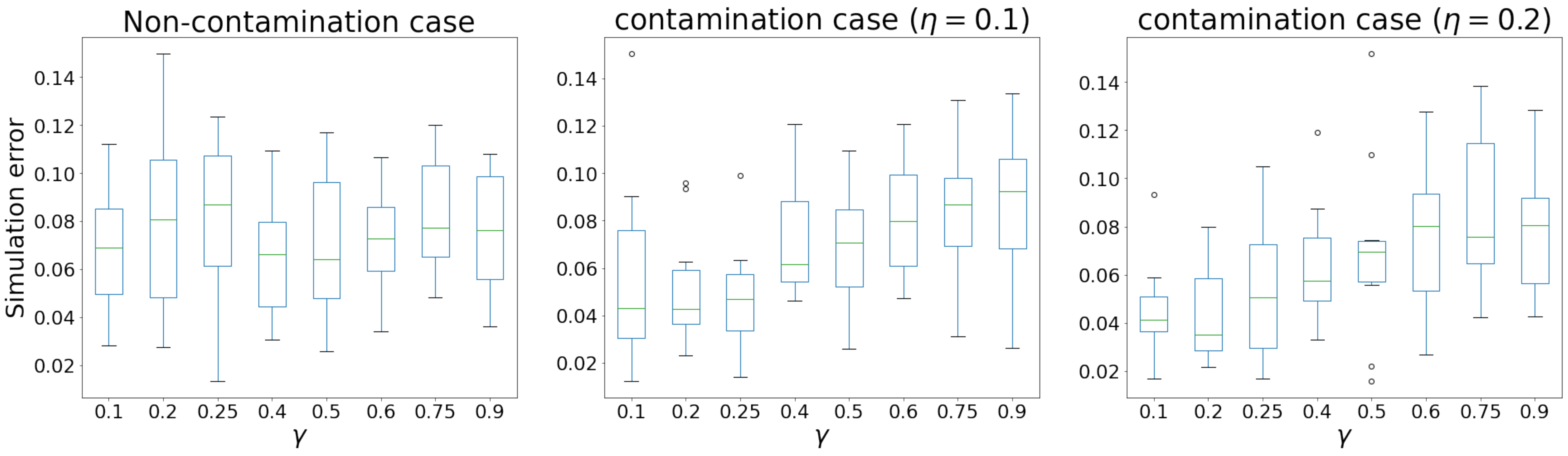}
    \caption{All of the experimental results of our method for the BB model based on simulation error.}
    \label{fig:gamma_bb_energy}
\end{figure}

\begin{figure}[htbp]
    \centering
    \includegraphics[scale=0.25]{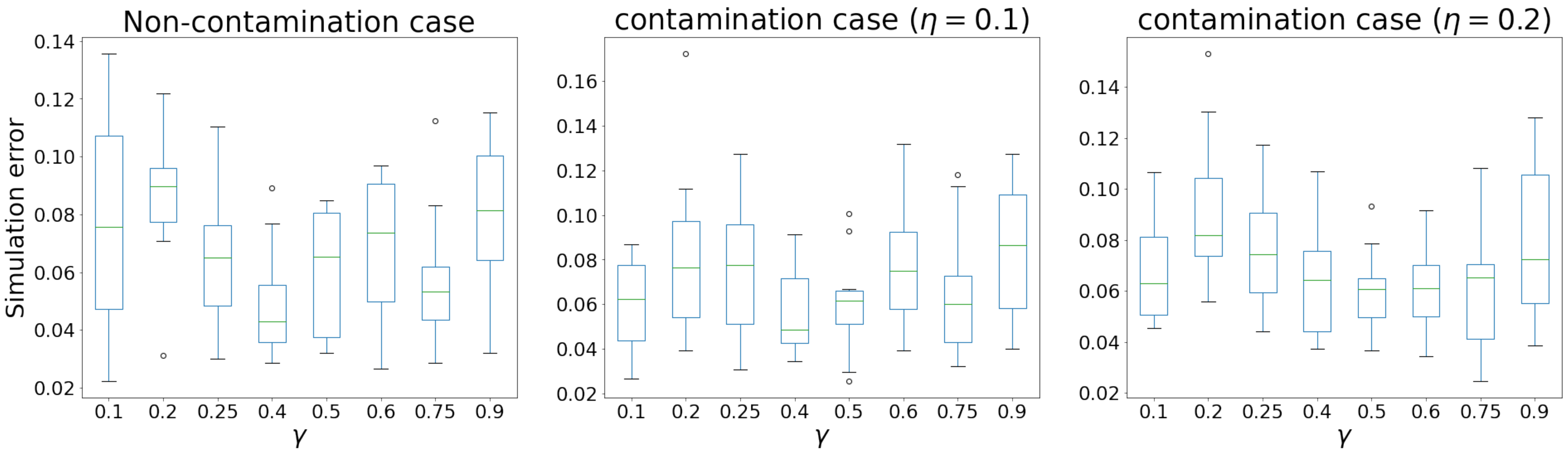}
    \caption{All of the experimental results of our method for the MA2 model based on simulation error.}
    \label{fig:gamma_ma2_energy}
\end{figure}
\begin{figure}[htbp]
    \centering
    \includegraphics[scale=0.25]{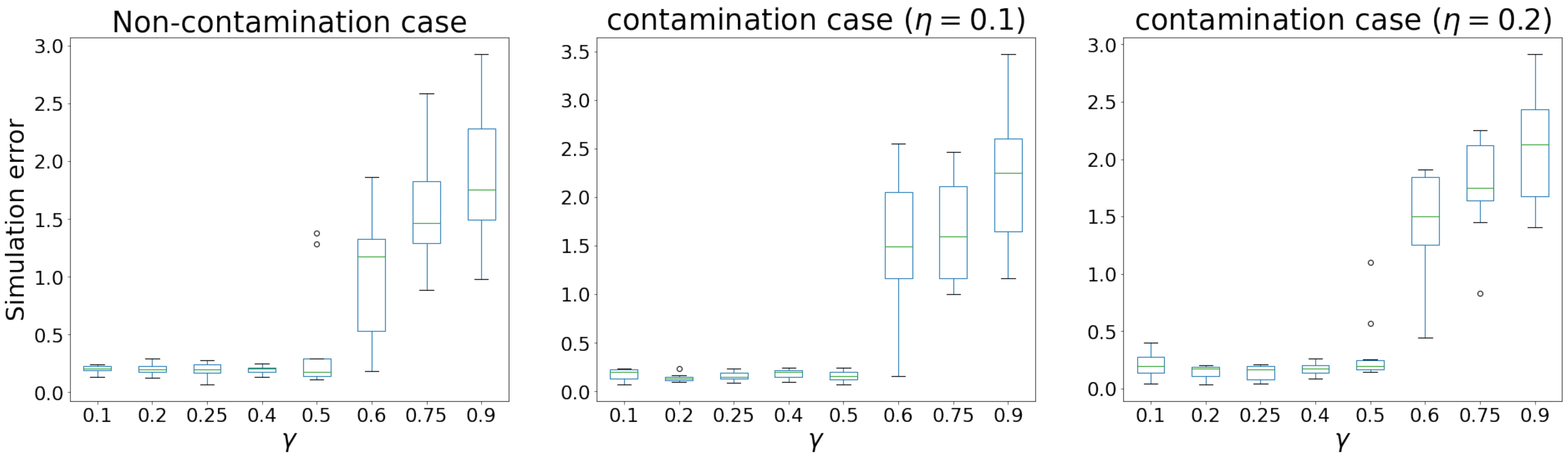}
    \caption{All of the experimental results of our method for the GK model based on simulation error.}
    \label{fig:gamma_gk_energy}
\end{figure}


\clearpage
\subsection{ABC posterior via our method and the second-best method}
\label{app:abc_post}
\diff{In this section, we report the ABC posterior distributions of our method for all experiments in Section~\ref{sec:experiments} when $\eta = 0.2$, and compare them with those of the second-best method.}

\subsubsection{Gaussian Mixture Model (GM)}

\begin{figure}[th]
    \centering
    \includegraphics[scale=0.23]{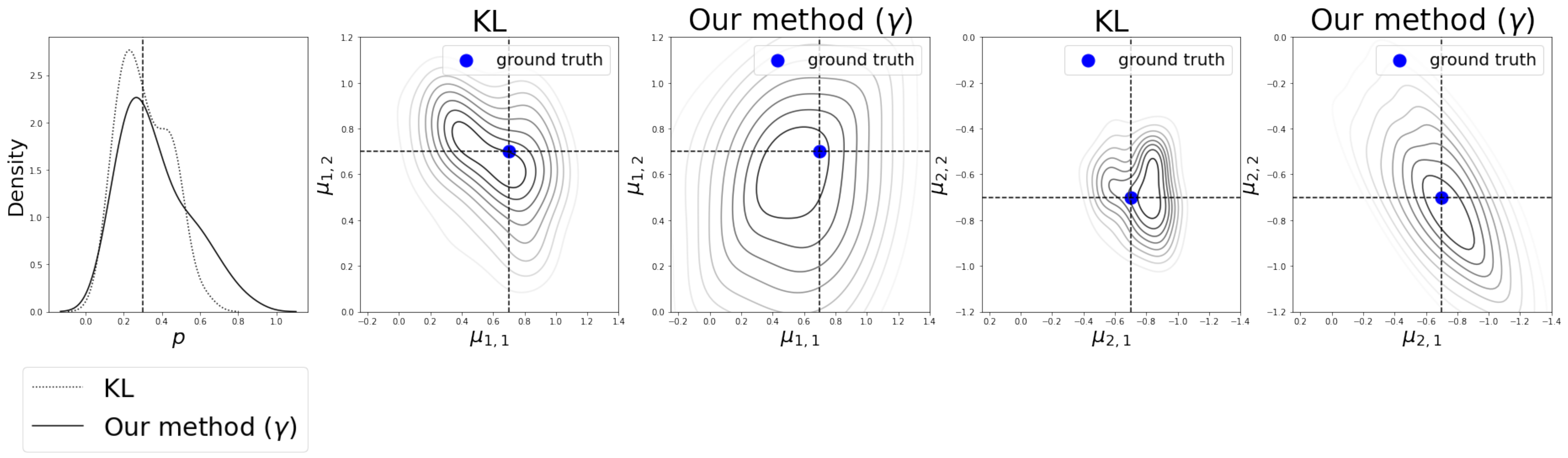}
    \caption{ABC posterior via our method and KL method.}
    \label{fig:pos_gm}
\end{figure}

\subsubsection{\textit{M}/\textit{G}/1-queueing Model (MG1)}
\begin{figure}[th]
    \centering
    \includegraphics[scale=0.23]{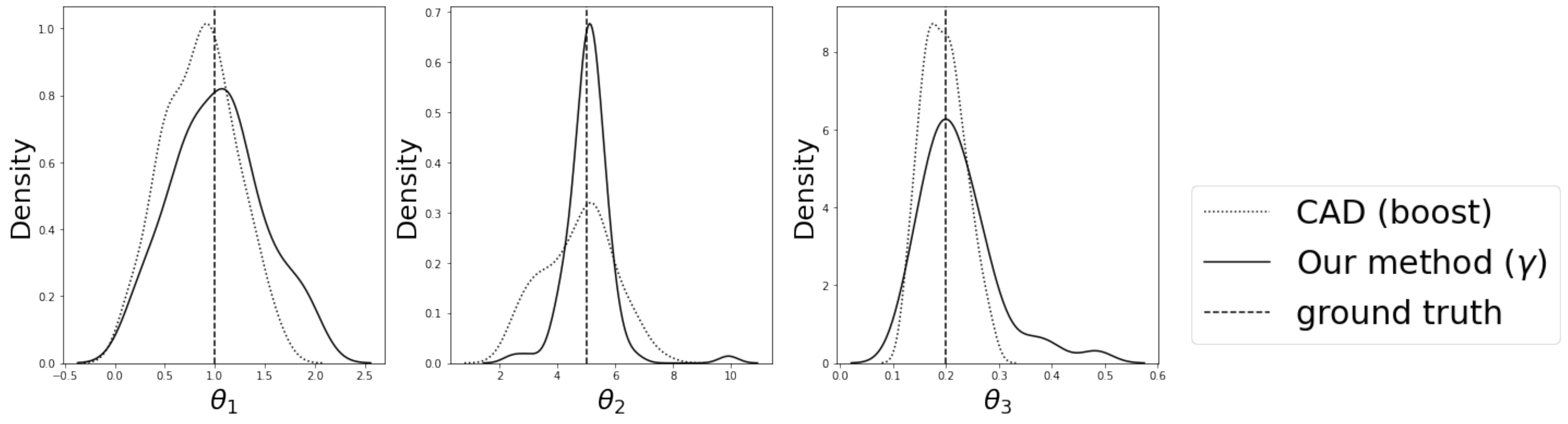}
    \caption{ABC posterior via our method and classification method with boosting.}
    \label{fig:pos_mg}
\end{figure}

\subsubsection{Bivariate Beta Model (BB)}
\begin{figure}[th]
    \centering
    \includegraphics[scale=0.23]{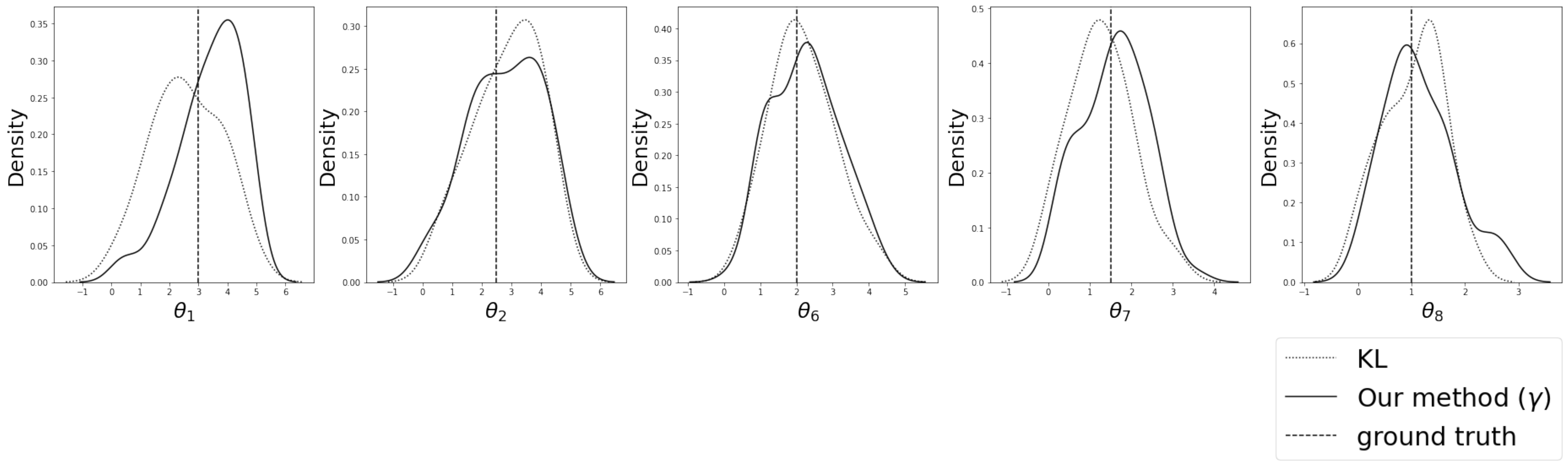}
    \caption{ABC posterior via our method and KL method.}
    \label{fig:pos_bb}
\end{figure}

\clearpage
\subsubsection{Moving-average Model of Order 2 (MA2)}
\begin{figure}[th]
    \centering
    \includegraphics[scale=0.23]{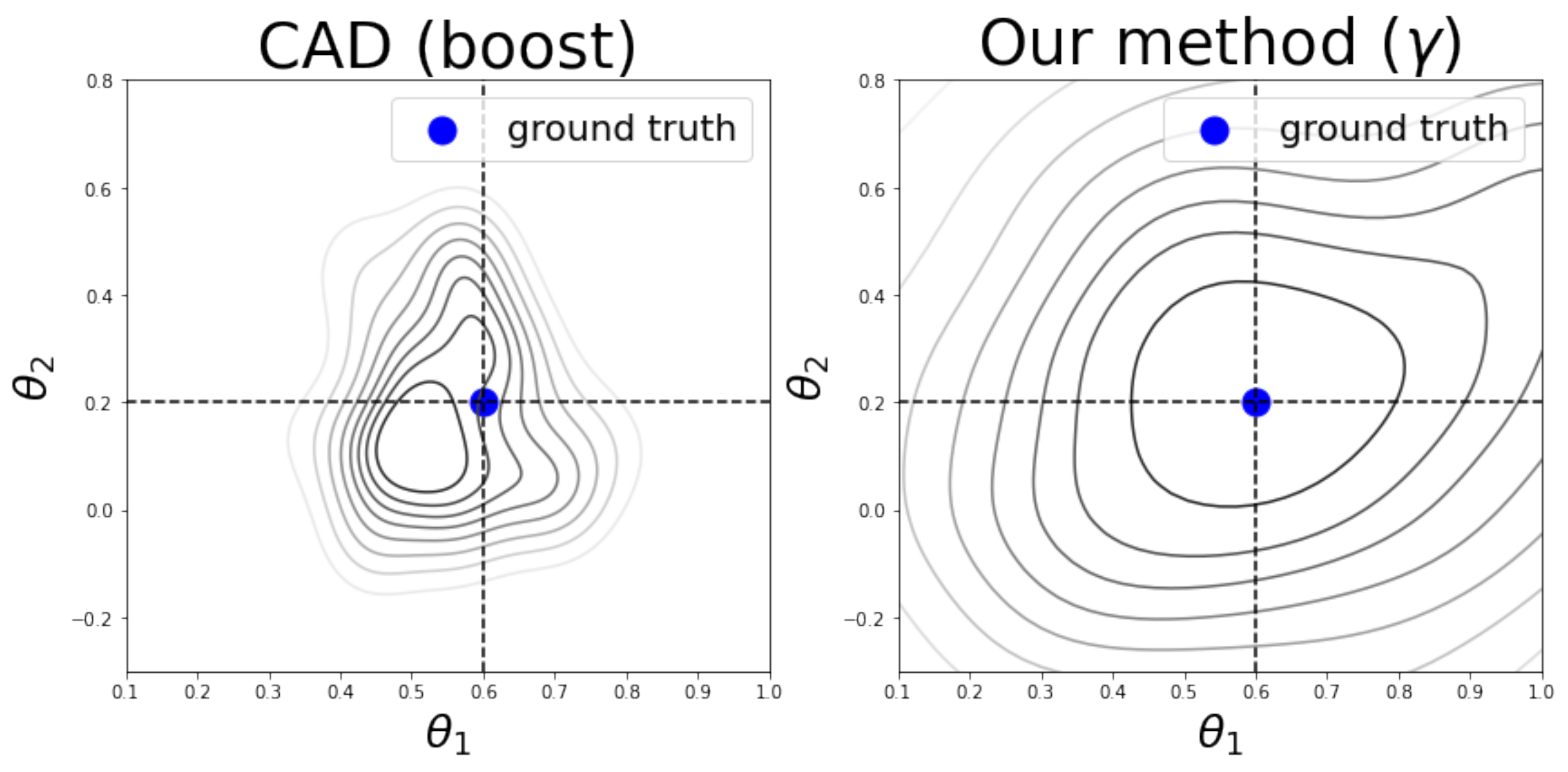}
    \caption{ABC posterior via our method and classification method with boosting.}
    \label{fig:pos_ma2}
\end{figure}

\subsubsection{Multivariate \texorpdfstring{$g$}{Lg}-and-\texorpdfstring{$k$}{Lg} Distribution (GK)}

\begin{figure}[th]
    \centering
    \includegraphics[scale=0.23]{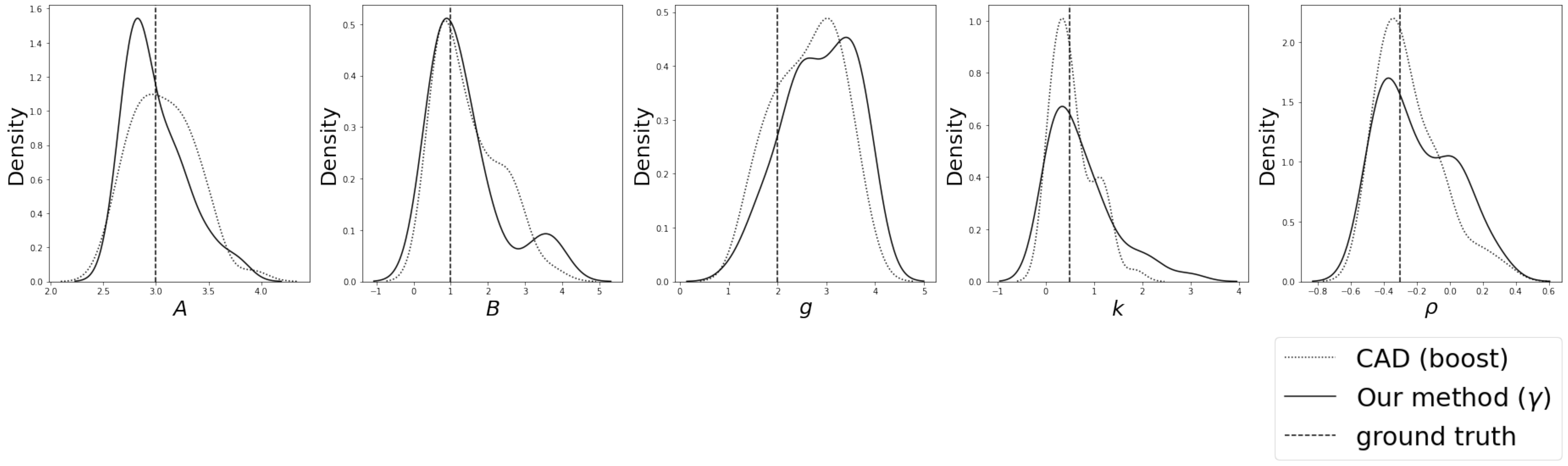}
    \caption{ABC posterior via our method and classification method with boosting.}
    \label{fig:pos_gk}
\end{figure}

\end{document}